\documentclass[draftclsnofoot,onecolumn]{IEEEtran}
\usepackage{algpseudocode,algorithm,amsmath,amsfonts,amssymb,amsthm,cite,graphicx,epstopdf,threeparttable,multirow,pifont,stmaryrd}
\usepackage[tight,footnotesize]{subfigure}
\usepackage[hidelinks]{hyperref}
\usepackage[usenames]{color}
\usepackage[normalem]{ulem}
\usepackage{enumitem}
\IEEEoverridecommandlockouts

\newcommand{\ba}        {\mathbf{a}}
\newcommand{\bA}        {\mathbf{A}}
\newcommand{\cA}        {\mathcal{A}}

\newcommand{\cD}        {\mathcal{D}}

\newcommand{\E}         {\mathbb{E}}

\newcommand{\cE}        {\mathcal{E}}

\newcommand{\cF}        {\mathcal{F}}

\newcommand{\bg}        {\mathbf{g}}

\newcommand{\bh}        {\mathbf{h}}

\newcommand{\bI}        {\mathbf{I}}

\newcommand{\cN}        {\mathcal{N}}

\newcommand{\bbP}       {\mathbb{P}}

\newcommand{\R}         {\mathbb{R}}

\newcommand{\bu}        {\mathbf{u}}

\newcommand{\bv}        {\mathbf{v}}

\newcommand{\cV}        {\mathcal{V}}

\newcommand{\bw}        {\mathbf{w}}
\newcommand{\bW}        {\mathbf{W}}
\newcommand{\cW}        {\mathcal{W}}

\newcommand{\wtbw}      {\widetilde{\bw}}

\newcommand{\bx}        {\mathbf{x}}

\newcommand{\wtbx}      {\widetilde{\bx}}

\newcommand{\Z}         {\mathbb{Z}}

\newcommand{\bz}        {\mathbf{z}}

\newcommand{\cZ}        {\mathcal{Z}}

\newcommand{\bxi}       {\boldsymbol{\xi}}

\newcommand{\bzeta}    {\boldsymbol{\zeta}}
\newcommand{\bSigma}    {\boldsymbol{\Sigma}}

\newcommand{\bzero}     {\mathbf{0}}

\newcommand{\nN}[1]             {\left[\!\left[{#1}\right]\!\right]}

\DeclareMathOperator*{\argmin}  {arg\,min}
\DeclareMathOperator*{\argmax}  {arg\,max}
\newcommand{\tT}        {\mathrm{T}}
\newcommand{\DMK}{\mbox{DM-Krasulina}}
\newcommand{\gap}{{\tt gap}}
\theoremstyle{plain}

\newtheorem{theorem}{Theorem}
\newtheorem{corollary}{Corollary}

\theoremstyle{definition}
\newtheorem{definition}{Definition}
\theoremstyle{remark}
\newtheorem{remark}{Remark}

\newcommand{\revise}[1]{#1}
\graphicspath{{./}}
\begin{document}

\title{Scaling-up Distributed Processing of Data Streams for Machine Learning}

% authors
\author{Matthew~Nokleby, Haroon~Raja, and Waheed~U.~Bajwa,~\IEEEmembership{Senior~Member,~IEEE}% <-this % stops a space
% thanks
\thanks{M.~Nokleby ({\tt matthew.nokleby@target.com}) is with the Target Corporation, Minneapolis, MN. H.~Raja ({\tt hraja@umich.edu}) is with the Department of Electrical Engineering and Computer Science at the University of Michigan, Ann Arbor, MI. W.U.~Bajwa ({\tt waheed.bajwa@rutgers.edu}) is with the Department of Electrical and Computer Engineering and Department of Statistics at Rutgers University--New Brunswick, NJ.}
\thanks{The work of WUB has been supported in part by the National Science Foundation under awards CCF-1453073, CCF-1907658, and OAC-1940074, by the Army Research Office under award W911NF-17-1-0546, and by the DARPA Lagrange Program under ONR/NIWC contract N660011824020. The work of HR has been supported in part by the National Science Foundation under awards CCF-1845076 and IIS-1838179, and by the Army Research Office under award W911NF-19-1-0027.}}

% running headings
%\markboth{}{}

\maketitle

%***************
\vspace{-2\baselineskip}
\begin{abstract}
Emerging applications of machine learning in numerous areas---including online social networks, remote sensing, internet-of-things systems, smart grids, and more---involve continuous gathering of and learning from streams of data samples. Real-time incorporation of streaming data into the learned machine learning models is essential for improved inference in these applications. Further, these applications often involve data that are either inherently gathered at geographically distributed entities due to physical reasons---e.g., internet-of-things systems and smart grids---or that are intentionally distributed across multiple computing machines for memory, storage, computational, and/or privacy reasons. Training of machine learning models in this distributed, streaming setting requires solving stochastic optimization problems in a collaborative manner over communication links between the physical entities. When the streaming data rate is high compared to the processing capabilities of individual computing entities and/or the rate of the communications links, this poses a challenging question: how can one best leverage the incoming data for distributed training of machine learning models under constraints on computing capabilities and/or communications rate? A large body of research in distributed online optimization has emerged in recent decades to tackle this and related problems. This paper reviews recently developed methods that focus on large-scale distributed stochastic optimization in the compute- and bandwidth-limited regime, with an emphasis on convergence analysis that explicitly accounts for the mismatch between computation, communication and streaming rates, and that provides sufficient conditions for order-optimum convergence. In particular, it focuses on methods that solve: ($i$) distributed stochastic convex problems, and ($ii$) distributed principal component analysis, which is a nonconvex problem with geometric structure that permits global convergence. For such methods, the paper discusses recent advances in terms of distributed algorithmic designs when faced with high-rate streaming data. Further, it reviews theoretical guarantees underlying these methods, which show there exist regimes in which systems can learn from distributed processing of streaming data at order-optimal rates---nearly as fast as if all the data were processed at a single super-powerful machine.
\end{abstract}


%***************
\section{Introduction}\label{sec:intro}

\subsection{Motivation and Background}\label{ssec:intro.motivation}
Over the past decade---and especially the past few years---there has been a rapid increase in research and development of \emph{artificial intelligence} (AI) systems across the public and private sectors. A significant fraction of this increase is attributable to remarkable recent advances in a subfield of AI that is termed \emph{machine learning}. Briefly, a machine learning system uses a number of data samples---referred to as \emph{training data}---in order to \emph{learn} a mathematical \emph{model} of some aspect of the physical world that can then be used for automated decision making; see Fig.~\ref{fig:intro.quad.chart}(a) for an example of this in the context of automated tagging of images of cats and dogs. Training a machine learning model involves mathematical optimization of a \emph{data-driven} function with respect to the model variable. The decision making capabilities of a machine learning system, in particular, tend to be directly tied to one's ability to solve the resulting optimization problem up to a prescribed level of accuracy.

While solution accuracy remains one of the defining aspects of machine learning, the advent of \emph{big data}---in terms of data dimensionality and/or number of training samples---and the adoption of large-scale models with millions of parameters in machine learning methods such as \emph{deep learning}~\cite{LeCunBengioEtAl.N15} has catapulted the computing time for training (i.e., the \emph{training time}) to another one of the defining parameters of modern %machine learning
systems. It is against this backdrop that stochastic optimization methods such as \emph{stochastic gradient descent} (SGD) and its variants~\cite{bottou2010large,lan2012optimal,johnson2013accelerating,li2014efficient}, in which training data are processed one sample or a small batch of samples---referred to as a \emph{mini batch}---per %algorithmic
iteration, as opposed to deterministic optimization methods such as gradient descent~\cite{Nocedal.Wright.Book2006}, in which the entire batch of training data is used in each %algorithmic
iteration, have become the de facto standard for faster training of %machine learning
models.

\begin{figure*}[t]
    \centering
    \includegraphics[width=\textwidth]{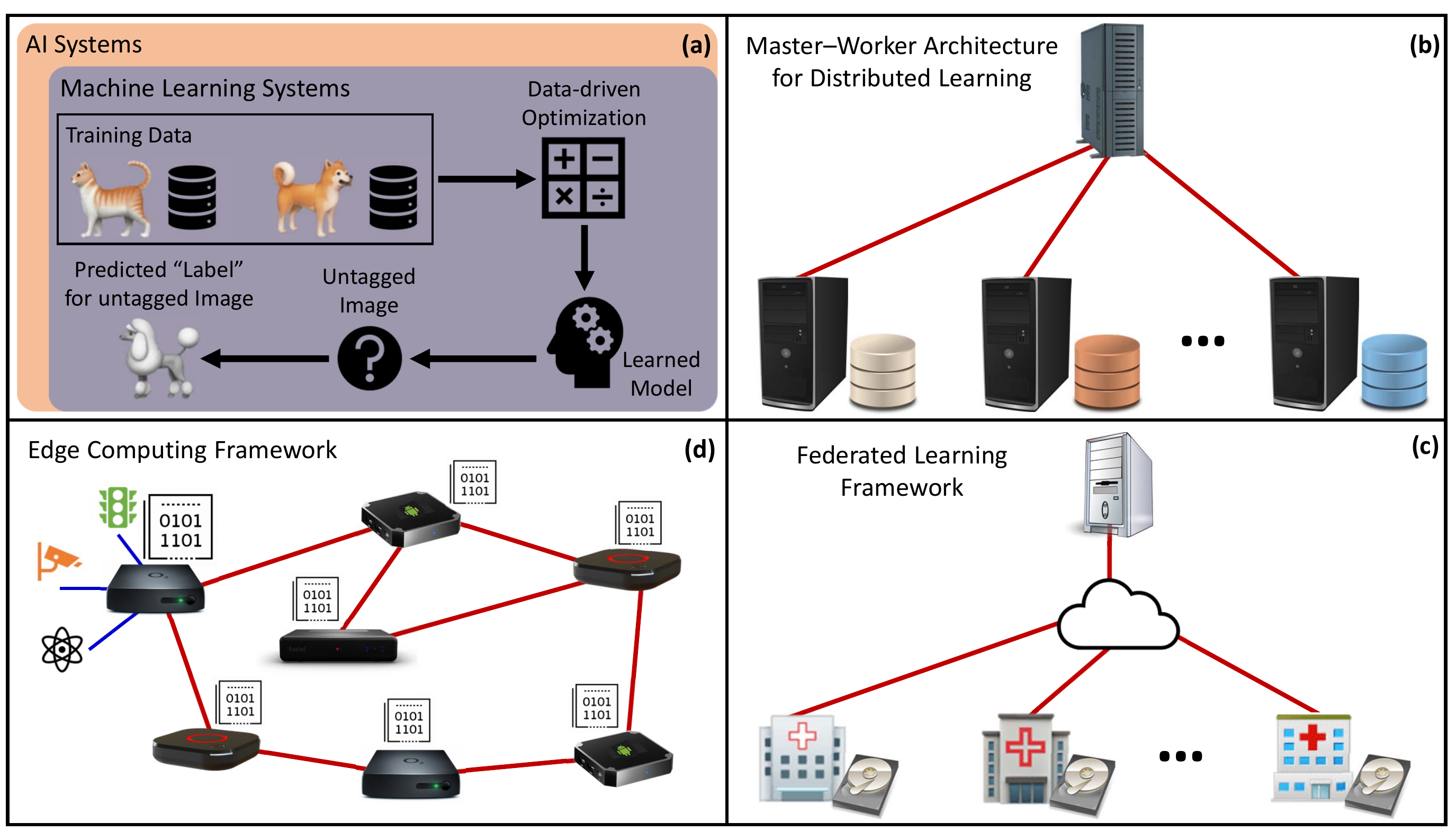}
    \caption{A schematic quad chart illustrating four different concepts in machine learning. (a) A simplified representation of a machine learning system within the context of an image tagging application. (b) Master--worker architecture that can be used for distributed training of a machine learning model within a compute cluster. (c) A federated machine learning system representing the federation of a group of autonomous hospitals. (d) A representative edge computing framework that can be used for decentralized training of a machine learning model.}
    \label{fig:intro.quad.chart}
\end{figure*}

Another major shift in machine learning practice concerns the use of \emph{distributed} and \emph{decentralized} computing platforms, as opposed to a single computing unit, for training of models. There are myriad reasons for this paradigm shift, which range from the focus on further decreasing the training times %of machine learning models
and preserving privacy of data to adoption of machine learning for decision making in inherently decentralized %engineering
systems. In particular, distributed and decentralized training of machine learning models can be epitomized by the following three prototypical frameworks.
\begin{itemize}
    \item \emph{Distributed computing framework:} A distributed computing framework, %sometimes
    also referred to as a \emph{compute cluster}, brings together a set of computing units such as CPUs and GPUs %in order
    to accelerate training of large-scale machine learning models from big data in a more cost-effective manner than a single computer with comparable storage capacity, memory, and computing power. Computing units/machines in a compute cluster typically communicate among themselves using either Ethernet or InfiniBand interconnects, with the intra-cluster communication infrastructure often abstracted in the form of a graph in which vertices/nodes correspond to computing units. A typical graph structure that is commonly utilized for distributed training %of machine learning models
    within compute clusters is \emph{star graph}, which corresponds to the so-called \emph{master--worker} architecture; see, e.g., Fig.~\ref{fig:intro.quad.chart}(b). Training data within this setup is split among the worker nodes, which perform bulk of the computations, while the master node coordinates the distributed training of machine learning model among the worker nodes.

    \item \emph{Federated learning framework:} The term ``federated learning,'' coined in \cite{KonecnyMcMahanEtAl.arxiv16}, refers to any machine learning setup in which a collection of autonomous entities (e.g., smartphones and hospitals), each maintaining its own private training data, collaborate under the coordination of a central server to learn a ``global'' machine learning model that best describes the collective non-collocated/distributed training data. A typical federated learning system, in which entities collaborate only through communication with the central server and are prohibited from sharing raw training samples with the server, can also be abstracted as a star graph; see, e.g., Fig.~\ref{fig:intro.quad.chart}(c). However, unlike a master--worker distributed machine learning system---in which the primary objective is reduction of the wall-clock time for training of machine learning models, the first and foremost objective of a federated learning system is to preserve privacy of the data of collaborating entities.

    \item \emph{Edge computing framework:} The term ``edge computing'' refers to any decentralized computing system comprising geographically distributed and compact computing devices that collaboratively complete a computational task through local computations and device-to-device communications. Some of the defining characteristics of an edge computing system, which set it apart from a compute cluster, include lack of a coordinating central server, (relatively) slower-speed device-to-device communications (e.g., wireless communications and power-line communications), and abstraction of inter-device communication infrastructure in terms of arbitrary graph topologies (as opposed to star topology). Many emerging edge computing systems, such as the \emph{internet-of-things} (IoT) systems and smart grids, have each computing device connected to a number of data-gathering sensors that generate large volumes of data. Since exchange of these large-scale ``local'' data among the computing devices becomes prohibitive due to communication constraints, machine learning in such systems necessitates decentralized collaborative training that involves each device learning (approximately) the same ``local'' model through inter-device communications that best fits the collective system data; see, e.g., Fig.~\ref{fig:intro.quad.chart}(d).
\end{itemize}

The purpose of this paper is to provide an overview of some important aspects of distributed/decentralized machine learning that have implications for all three of the aforementioned %computing
frameworks. We slightly abuse terminology in the following for ease of exposition and refer to training of machine learning models under any one of these frameworks as distributed machine learning. When one considers distributed training of (large-scale) %machine learning
models from (big) distributed datasets, it raises a number of important questions; these include: ($i$) What are the fundamental limits on solution accuracy of distributed machine learning? ($ii$) What kind of %(deterministic or stochastic)
optimization frameworks and communication strategies (which exclude exchange of raw data among subcomponents of the system) result in near-optimal distributed learning? ($iii$) How do the topology of the graph underlying the distributed computing setup and the speed of communication links in the setup impact the learning performance of any optimization framework? A vast body of literature in the last decade has addressed these (and related) questions for distributed machine learning by expanding on foundational works in distributed consensus~\cite{Tsitsiklis.Athans.ITAC1984,xiao2004fast}, distributed diffusion~\cite{ChenSayed.ITSP12}, distributed optimization~\cite{Tsitsiklis.etal.ITAC1986,nedic2009distributed}, and distributed computing~\cite{DeanGhemawat.ConfOSDI04}. Several of the key findings of such works have also been elucidated through excellent survey articles and overview papers in recent years~\cite{OlfatiSaber.etal.PI2007,olshevsky2011convergence,Boyd.etal.Book2011,Sayed.BookNOW14,Nedic.BookCh18,NedicOlshevskyEtAl.PI18,LiSahuEtAl.arxiv19,KairouzMcMahanEtAl.arxiv19,YangGangEtAl.arxiv19,XinKarEtAl.ISPM20}. Nonetheless, there remains a need to better understand the interplay between solution accuracy, communication capabilities, and computational resources in distributed %machine learning
systems that carry out training using ``streaming'' data. Indeed, distributed training using streaming data necessitates utilization of \emph{single-pass} stochastic optimization within the distributed framework, which gives rise to several important operational changes that are not widely known%in the literature
. It is in this regard that this overview paper summarizes some of the key research findings, and their implications, in relation to distributed machine learning from streaming data.

\subsection{Streaming Data and Distributed Machine Learning}\label{ssec:intro.streaming.data}
Continuous gathering of data is a hallmark of the digital revolution; in countless applications, this translates into streams of data entering into the respective machine learning systems. Within the context of distributed machine learning, the continuous data gathering has the effect of training data associated with each ``node'' in the distributed system being given in the form of a data stream (cf.~Fig.~\ref{fig:DistributedNetwork} in Section~\ref{sec:ProblemFormulation}). Since ``(full) batch processing'' is practically infeasible in the face of continuous data arrival, distributed training of %machine learning
models from streaming data requires (single-pass) stochastic optimization methods. Accordingly, we provide in this paper an overview of some of the state-of-the-art concerning stochastic optimization-based distributed training from streaming data.

Unlike much of the literature on centralized machine learning from streaming data, (relative) streaming rate of data---defined as the (average) number of new data samples arriving per second---fundamentally shapes the discussion of streaming-based distributed machine learning. In this regard, our objective is to elucidate the performance challenges and fundamental limits when the streaming rate of data is \emph{fast} compared to the processing speed of computing units and/or the communications speed of inter-node links in the system. In particular, this involves addressing of the following question: \emph{what happens to the solution accuracy of distributed machine learning when it is impossible to have high-performance computing machines for computing nodes and/or (multi-)gigabit connections for inter-node communication links?} Note that this question cannot be addressed by simply ``slowing down'' the data stream(s) through regular discarding of incoming samples. Within a distributed computing framework, for instance, letting some of the incoming samples pass without updating the model would be antithetical to its overarching objective of accelerated training. Similarly, downsampling of time-series data streams in an edge computing system would cause the system to lose out on critical high-frequency modes of data. In short, processing all data samples arriving into the distributed system and incorporating them into the learned model is both paramount and non-trivial.

There are many ways to frame and analyze the problem of distributed machine learning from fast streaming data, leading to far more relevant works than we can discuss in this overview paper. Instead, we provide a very brief discussion of the different framings, and motivate our prioritization of the following system choices under the general umbrella of distributed machine learning: \textbf{decentralized-parameter} systems, \textbf{synchronous-communications} distributed computing, and \textbf{statistical risk minimization} for training of machine learning models. We dive into the relevant distinctions for these system choices in the following.

\subsection{General Framing of the Overview}\label{ssec:intro.framing}
The %general
area of distributed machine learning is far too rich and broad to be covered in a %single overview
paper. Instead, we cover %in this paper
only some aspects of the area that are the most relevant to the topic of distributed machine learning from streaming data. %In order to
To %
put the rest of our discussion %in the paper
in context, we give a very coarse description of these aspects in the following, drawing out some of the crucial distinctions and pointing out which aspects %of the area
remain uncovered in the paper.

\textbf{System models for distributed learning.} We abstract away the dependence on any particular computing architecture by modeling the architecture as an interconnected network of (computing) nodes having a certain topology (e.g., star topology for the master--worker architecture). Accordingly, our discussion is applicable to any of the computing frameworks discussed in Section~\ref{ssec:intro.motivation} that adhere to the data and system assumptions described later in Section~\ref{sec:ProblemFormulation}. In the interest of generality, we also move away from the so-called \emph{parameter-server} system model that is used in some distributed %machine learning
environments~\cite{LiAndersenEtAl.ConfOSDI14,AbadiBarhamEtAl.ConfOSDI16}. In the simplest version of this %system
model, a single node---termed parameter server---maintains and updates parameters of the machine learning model, whereas the remaining nodes in the network compute gradients of their local data that are then transmitted to the parameter server and used to make updates to the shared set of parameters. We instead center our discussion around the {\em decentralized-parameter} system model, where each node maintains and updates its own copy of the parameters. This system model is more general, since any result that holds for a decentralized-parameter network also holds for a parameter-server network, it prevents a single point of failure in the system, and it allows us to present a unified discussion that transcends multiple system models.

\textbf{Models for message passing and communications.} Algorithmic-level synchronization (or lack thereof) among different computing nodes is one of the most important design choices in distributed %machine learning
implementations. On one hand, \textit{synchronous} implementations (which often make use of ``blocking'' message passing protocols for synchronization~\cite{DongarraOttoEtAl.TechRep95,GabrielFaggEtAl.ConfOpenMPI04}) can slow down training times due to either message passing (i.e., communications) delays or ``straggler'' nodes taking longer than the rest of the network to complete their subtasks. On the other hand, \textit{asynchronous} implementations have the potential to drastically impact the solution accuracy. Such tradeoffs between synchronous and asynchronous implementations, as well as approaches that hybridize the two%implementations
, have been investigated %by many researchers
in recent years%; see, e.g.,
~\cite{Recht.etal.AiNIPS22011,TsianosLawlorEtAl.ConfAllerton12,zhang2016hogwild++,ChenMongaEtAl.ConfICLRWT16,TandonLeiEtAl.ConfICML17,WangTantiaEtAl.arxiv19}. In this paper, we focus exclusively on synchronous implementations %(and blocking message passing primitives)
for the sake of concreteness. In addition, we abstract lower-level communications within the synchronous system as happening in discrete, pre-defined epochs (time intervals, slots, etc.). While such an abstraction models only a restrictive set of communications protocols, it greatly simplifies the exposition without sacrificing too much of the generality.

\textbf{Optimization framework for distributed machine learning.} \revise{Machine learning problems involve the optimization of a ``loss'' function with respect to the machine learning model.} And this optimization side of machine learning can be framed in two major \revise{interrelated} ways. The first (and perhaps most well-known) framing is referred to as \emph{empirical risk minimization} (ERM). The objective in this case is to minimize the \emph{empirical} risk \revise{$\hat{f}(\bw)$}, defined as the empirical average of the so-called \emph{(regularized) loss function} \revise{$\ell(\bw,\cdot)$} evaluated on the training samples\revise{, with respect to the model variable $\bw$}. Under mild assumptions on the loss function, data distribution, and training data, the ERM solution \revise{$\bw^*_{\textsc{erm}} \in \argmin_{\bw} \hat{f}(\bw)$} is known to converge (with high probability) to the minimizer \revise{$\bw^*$} of the ``true'' risk \revise{$f(\bw)$}, i.e., the \emph{expected} loss \revise{$f(\bw) := \E[\ell(\bw,\cdot)]$}~\cite{vapnik1999statistical}, with study of the rates of this convergence being a long-standing and active research area~\cite{ShalevShwartzShamirEtAl.ConfCOLT09}. Distributed learning literature within the ERM framework typically supposes a fixed and finite number \revise{$T$} of training samples distributed across computing nodes, and primarily focuses on understanding convergence of \revise{the output $\widehat{\bw}^*_{\textsc{erm}}$ of} different \emph{distributed optimization} schemes to the ERM solution \revise{$\bw^*_{\textsc{erm}}$}~\cite{NedicOlshevskyEtAl.PI18}. The accuracy of the final solution, termed \emph{excess risk} \revise{and defined as $f(\widehat{\bw}^*_{\textsc{erm}}) - f(\bw^*)$}, is then provided either implicitly or explicitly in the works as the sum of two gaps: \revise{($i$)} gap \revise{between the risk of the} distributed optimization solution \revise{and that of} the ERM solution\revise{, i.e., $f(\widehat{\bw}^*_{\textsc{erm}}) - f(\bw^*_{\textsc{erm}})$,} and \revise{($ii$)} gap \revise{between the risk of} the ERM solution \revise{and that of} the optimal solution, termed \emph{Bayes' risk}\revise{, i.e., $f(\bw^*_{\textsc{erm}}) - f(\bw^*)$}. In contrast to the ERM framework, the second optimization-based framing of machine learning---termed \emph{statistical risk minimization} (SRM)---facilitates a direct bound on the excess risk; see, e.g., Fig.~\ref{fig:erm.vs.srm}. This is because the objective in SRM framework is to minimize expected loss (risk) over the true data distribution, as opposed to empirical loss over the training data in ERM framework. The SRM framework falls squarely within the confines of stochastic optimization, with a large body of existing work---covering both centralized and distributed machine learning---that characterizes the excess risk of the resulting solution \revise{$\bw_{\textsc{srm},T}$} under the assumption that either the number of training samples \revise{$T$} is sufficiently large or it grows asymptotically. Since we are concerned with %the case of
streaming data, in which a virtually unbounded number of %data
samples may arrive at the system, we focus on the SRM-based framework and single-pass stochastic optimization for distributed machine learning. We discuss further the distinction between the convergence results derived under the frameworks of ERM and SRM in the sequel.

\begin{figure*}[t]
    \centering
    \includegraphics[width=0.55\textwidth]{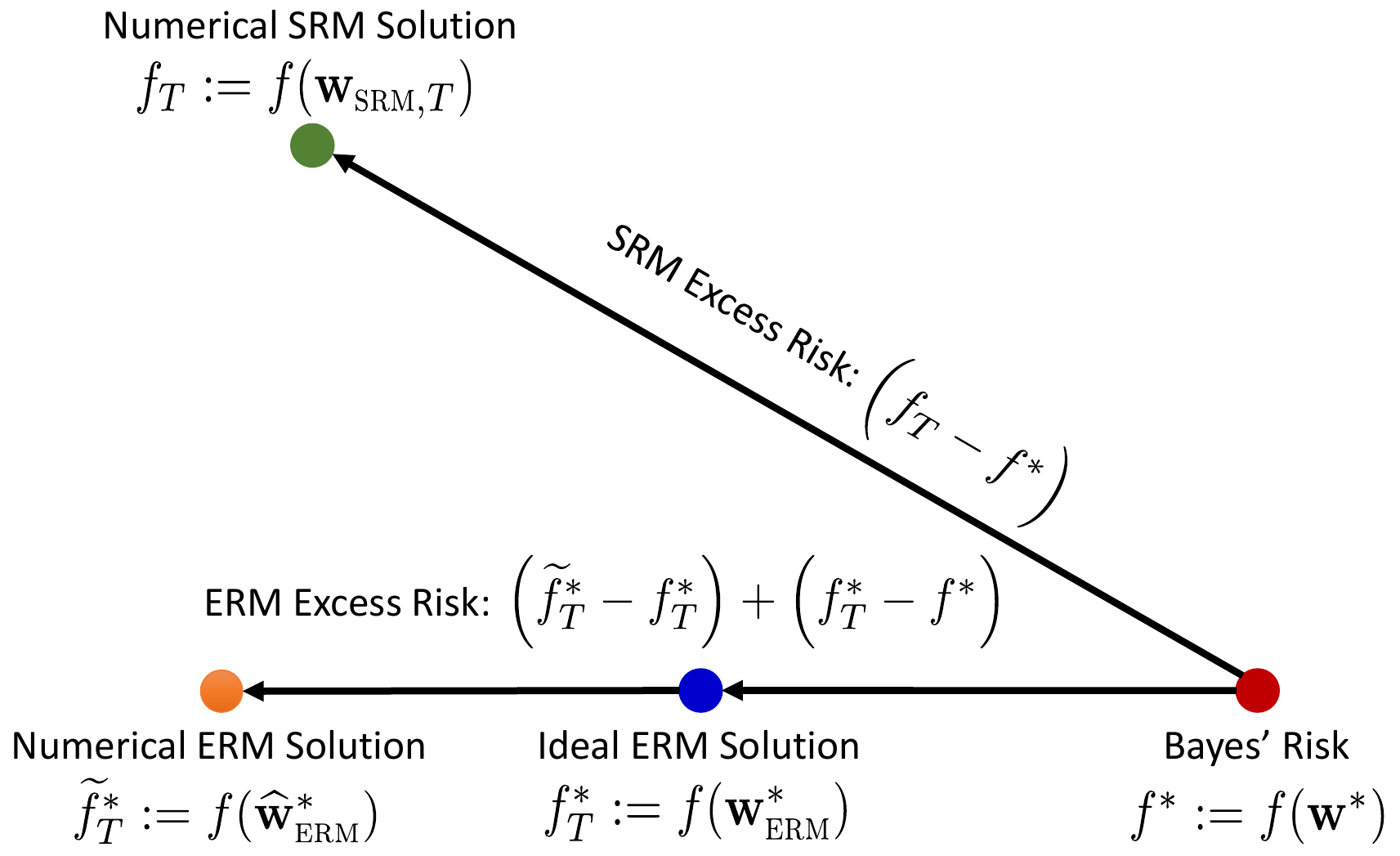}
    \caption{A geometrical view of ``excess risk,'' defined as the gap between the expected loss \revise{$f(\bw)$} (i.e., risk) of the solution $\bw$ of a practical machine learning algorithm and that of an ideal solution $\bw^*$ (i.e., Bayes' risk \revise{$f^* := f(\bw^*)$}), under the ERM and SRM optimization frameworks after receiving a total of $T$ training data samples. Under the ERM framework, the excess risk of the solution is bounded as the sum of bounds on two terms, namely, the gap between the risk of an ideal ERM solution and Bayes' risk, $(f^*_T - f^*)$, and the gap \revise{in risk} due to the error associated with numerical optimization, $(\widetilde{f}^*_T - f^*_T)$. In contrast, excess risk under the SRM framework is directly captured in terms of the gap between the risk of the numerical SRM solution $\bw_{\textsc{srm},T}$ and Bayes' risk, i.e., $(f_T - f^*)$.}
    \label{fig:erm.vs.srm}
\end{figure*}

\textbf{Structure of the optimization objective function.} The vast majority of works at the intersection of (stochastic) optimization and (distributed) learning suppose that the loss function is \emph{convex} with respect to the model parameters. But some of the most exciting recent results in machine learning have come about in the context of deep learning, where the objective function tends to be highly nonconvex and most practical methods do not even concern themselves with \emph{global optimality} of the solution~\cite{LeNgiamEtAl.ConfICML11,NeyshaburSalakhutdinovEtAl.ConfNeurIPS15}. Nevertheless, for the purpose of being able to carry out analysis, we focus in this overview on either convex problems or \emph{structured} nonconvex problems, such as \emph{principal component analysis} (PCA), where the structure can be exploited by local search methods to find a global solution. %We do, however, sketch out how one might obtain similar flavored results for more general nonconvex problems, as well as describe the challenges nonconvexity poses to distributed machine learning from streaming data.

\subsection{An Outline of the Overview Paper}\label{subsec:paper_Organization}
We now provide an outline of the remainder of this paper. Section~II gives a formal description of the learning and system models considered in this paper, including the loss function, the streaming data model, the communications model, and the way compute nodes exchange messages with each other during distributed learning. In Section~III, we discuss relevant results in (distributed) machine learning that prefigure the state-of-the-art being reviewed in the paper. Section~IV and Section~V of the paper are devoted to coverage of the state-of-the-art in terms of distributed machine learning from fast streaming data. The main distinction between the two sections is the nature of the communications infrastructure underlying the distributed computing framework. %Specifically,
Section~IV focuses on the case of (relatively) high-speed communications infrastructure that enables completion of message-passing primitives such as {\tt AllReduce}~\cite{GabrielFaggEtAl.ConfOpenMPI04} in a ``reasonable'' amount of time, whereas Section~V discusses distributed machine learning from streaming data in systems with (relatively) lower-speed communications infrastructure. In both cases, we discuss scenarios and distributed algorithms that can lead to near-optimal excess risk for the final solution as a function of the number of data samples arriving at the system; in addition, we present results of numerical experiments to corroborate some of the stated results. One of the key insights delivered by these two sections is that a judicious use of (implicit or explicit) mini-batching of data samples in distributed systems is fundamental in dealing with fast streaming data in compute- and/or communications% bandwidth
-limited scenarios. To this end, we provide theoretical results for the optimum choice of the size of (network-wide and local) mini-batches as well as conditions on when mini-batching is sufficient to achieve near-optimal convergence. We conclude the paper in Section~VI with a brief recap of the implications of presented results for the practitioners as well as a discussion of possible next steps for researchers working on distributed machine learning.

\subsection{Notational Convention}\label{ssec:intro.notation}
We use regular-faced (e.g., $a$ and $B$), bold-faced lower-case (e.g., $\ba$), and bold-faced upper-case (e.g., $\bA$) letters for scalars, vectors, and matrices, respectively. We use calligraphic letters (e.g., $\cA$) to represent sets, while $\nN{N} := \{1,\dots,N\}$ denotes the set of first $N$ natural numbers, and $\R_{\geq 0}$ and $\Z_+$ denote the sets of non-negative real numbers and positive integers, respectively. %Given a scalar $a \in \R$, $[a]^+ := \max\{a,0\}$ denotes the positive part of $a$.
Given a vector $\ba$ and a matrix $\bA$, $\|\ba\|_2 := \sqrt{\sum_{i}|a_i|^2}$, $\|\bA\|_2 := \argmax_{\bv : \bv \not= \bzero}\frac{\|\bA\bv\|_2}{\|\bv\|_2}$, and $\|\bA\|_F := \sqrt{\sum_{i,j}|a_{ij}|^2}$ denote the $\ell_2$-norm of $\ba$, the %operator (
spectral%)
~norm of $\bA$, and the Frobenius norm of $\bA$, respectively. Given a symmetric matrix $\bA \in \R^{d \times d}$, $\lambda_i(\bA)$ denotes its $i$-th largest-by-magnitude eigenvalue, i.e., $|\lambda_1(\bA)| \geq \dots \geq |\lambda_d(\bA)| \geq 0$. Given a function $f : \R^d \times \cZ \rightarrow \R$ that is partially differentiable in the first argument, $\nabla f$ denotes the gradient of $f(\cdot,\cdot)$ with respect to its first argument. Given functions $f(x)$ and $g(x)$, we use Landau's \emph{Big-O} notation (e.g., $f(x) = O(g(x))$ or $f(x) = o(g(x))$) to describe the scaling relationship between them. Finally, $\E\{\cdot\}$ denotes the expectation operator, where the underlying probability space $(\Omega,\cF,\bbP)$ is either implicit from the context or is explicitly noted.

\section{Problem Formulation}\label{sec:ProblemFormulation}
In this section, we %formally
discuss the problem of distributed processing of fast streaming data for machine learning %; this discussion is carried out
in three parts. First, we describe the general statistical optimization problem underlying machine learning. Second, we describe the system model that formalizes distributed processing of streaming data %for the purposes of machine learning
. Finally, we formalize the notion of \emph{fast} streaming data in terms of, among other things, data streaming rate, processing rate of compute nodes, and communications rate of inter-node links within the distributed environment.

\subsection{Statistical Optimization for Machine Learning}\label{subsec:StatisticalLearning}
% The concept of a \emph{loss function}, which assigns cost to machine learning models in terms of their inability to describe data drawn from an unknown probability distribution, is central to the discussion of machine learning.
Most machine learning problems can be posed as data-driven optimization problems, with the objective termed a \emph{loss function} that quantifies the error (classification, regression or clustering error, mismatch between the learned and true data distributions, etc.) in a candidate solution. We denote this loss function by $\ell: \cW \times \cZ \rightarrow \R_{\geq 0}$, where $\cW$ denotes the space of candidate machine learning models and $\cZ$ denotes the space of data samples. Given a model $\bw \in \cW$, %the non-negative number
$\ell(\bw,\bz)$ measures the modeling loss associated with $\bw$ in relation to the %particular
data sample $\bz \in \cZ$.\revise{\footnote{\revise{The model space $\cW$ in most formulations is taken to be one that is completely described by a set of parameters. For example, if $\cW$ denotes the space of all polynomials of degree $(d-1)$, a model $\bw \in \cW$ is uniquely characterized by the $d$ coefficients of the respective polynomial. In this paper, we slightly abuse the notation and use $\bw$ to denote both the model and, when the model is parameterizable, its respective parameters.}}}

Several examples of loss functions \revise{and their respective model space(s)} for \emph{supervised learning} (e.g., regression and classification) and \emph{unsupervised learning} (e.g., feature learning and clustering) problems are listed below.

\textbf{Loss functions \revise{and models} for supervised machine learning.} Data samples in supervised %machine
learning can be expressed as tuples $\bz := (\bx,y) \in \R^d \times \R =: \cZ$, with $\bx$ %simply
referred to as data and $y$ referred to as its \emph{label}. In particular, focusing on the linear classification problem with label $y \in \{-1,1\}$, \revise{augmented data $\wtbx := [\bx^\tT \ 1]^\tT \in \R^{d+1}$, and model space $\cW := \R^{d+1}$, the model $\bw$ describes an affine hyperplane in $\R^d$ and} two common choices of loss functions are:
%\begin{itemize}
    %\item
        ($i$) \emph{Hinge loss:} $\ell(\bw, \bz) := \max\left(0, 1-y \cdot \bw^\tT \revise{\wtbx}\right)$%.
            , and %
    %\item
        ($ii$) \emph{Logistic loss:} $\ell(\bw, \bz) := \ln\left(1+\exp{(-y\cdot\bw^\tT \revise{\wtbx})}\right)$.
%\end{itemize}

\textbf{Loss functions \revise{and models} for unsupervised machine learning.} Data samples in unsupervised machine learning do not have labels, with the \emph{unlabeled} data sample $\bz \in \R^d =: \cZ$ in this case. We now describe the loss functions \revise{and models/model parameters} associated with two popular unsupervised machine learning problems.
\begin{itemize}
    \item \emph{Principal component analysis (PCA):} The $k$-PCA problem is a feature learning problem in which the model space is $\cW := \left\{\bA\in\R^{d\times k}:\bA^{\tT}\bA=\bI\right\}$\revise{, the matrix-valued model $\bW \in \cW$ describes a $k$-dimensional subspace of $\R^d$,} and the loss function is $\ell(\bW,\bz) := \|\bz - \bW\bW^\tT\bz\|_2^2$.
    \item \emph{Center-based clustering:} The $k$-means clustering problem has the model space $\cW := \underbrace{\R^d \times \dots \times \R^d}_{k \text{ times}}$, \revise{the model $(\bw_1,\dots,\bw_k) \in \cW$ is a $k$-tuple of $d$-dimensional \emph{cluster centroids}, and} a common choice of the loss function is $\ell\left((\bw_1,\dots,\bw_k),\bz\right) := \min_{1\leq i \leq k}\|\bz - \bw_i\|_2^2$.
\end{itemize}

In this paper, our discussion of machine learning revolves around the statistical learning viewpoint~\cite{vapnik1999statistical}. To this end, we suppose each data sample $\bz$ is drawn from some unknown probability distribution $\cD$ that is supported on $\cZ$. The overarching goal in (statistical) machine learning then is to obtain a model $\bw \in \cW$ that has the smallest loss \emph{averaged over all} $\bz \in \cZ$. Specifically, let
\begin{align}\label{eqn:ExpectedLoss}
    f(\bw) := \E_{\bz \sim \cD}\big\{\ell(\bw, \bz)\big\}
\end{align}
denote the expected loss, also referred to as \emph{(statistical) risk}, associated with model $\bw$ for the entire data space $\cZ$. Then, the objective of machine learning from the statistical learning perspective is to approach the \emph{Bayes optimal} solution $\bw^*$ that minimizes the statistical risk, i.e.,
\begin{align}\label{eqn:StatisticalLearningCentralized}
    \bw^* \in \arg\min_{\bw\in\cW}f(\bw).
\end{align}
The risk incurred by $\bw^*$ (i.e., $f(\bw^*)$) is termed \emph{Bayes' risk}. The main challenge in machine learning is that the distribution $\cD$ is unknown and therefore \eqref{eqn:StatisticalLearningCentralized} cannot be directly solved. Instead, one uses training data samples $\{\bz_{t'}\}_{t' \in \Z_+}$ that are independently drawn from $\cD$ to obtain a model $\widehat{\bw}$ whose risk comes close to Bayes' risk as a function of the number of training samples. In particular, the performance of a machine learning algorithm is measured in terms of either the \emph{excess risk} of its solution, defined as $f(\widehat{\bw}) - f(\bw^*)$, or the \emph{parameter estimation error} calculated in terms of some distance between the solution $\widehat{\bw}$ and the set of minimizers $\arg\min_{\bw\in\cW}f(\bw)$.

%It can be seen from the preceding discussion that (data-driven) optimization is at the heart of machine learning.
%Because (statistical)
Since %
optimization is central to machine learning, the geometrical structure and properties of the loss function determine whether and how easily %an optimization
a %
method finds a solution $\widehat{\bw}$ that has (nearly) minimal excess risk / estimation error. We describe this structure and properties of $\ell(\bw, \bz)$ %in the paper
in terms of its \emph{gradients}, \emph{convexity}, and \emph{smoothness}.
\begin{definition}[Existence of Gradients]
    A loss function $\ell(\bw, \bz)$ is said to have its gradients exist everywhere if $\nabla_\bw \ell(\bw, \bz)$ exists for all $(\bw, \bz) \in \mathcal{W} \times \mathcal{Z}$.
\end{definition}
\begin{definition}[Convexity and Strong Convexity]
    A loss function $\ell(\bw, \bz)$ is {\em convex} in $\bw$ if for all $\bw_1, \bw_2 \in \mathcal{W}$, all $\bz \in \mathcal{Z}$, and all $\alpha \in [0,1]$, we have
    \begin{equation*}
        \ell(\alpha \bw_1 + (1-\alpha)\bw_2, \bz) \leq \alpha\ell(\bw_1, \bz) + (1-\alpha)\ell(\bw_2, \bz).
    \end{equation*}
    \revise{In words, the function $\ell(\cdot, \bz)$ for all $\bz \in \mathcal{Z}$} must lie below any chord for the loss function to be convex in $\bw$. Further, a loss function whose gradients exist everywhere is said to be {\em strongly convex} with modulus \revise{$m > 0$} if for all $\bw_1, \bw_2 \in \mathcal{W}$ and all $\bz \in \mathcal{Z}$, we have
    \begin{equation*}
        (\nabla_\bw \ell(\bw_1, \bz) - \nabla_\bw \ell(\bw_2, \bz))^\tT(\bw_1 - \bw_2) \geq m \| \bw_1 - \bw_2 \|_2^2.
    \end{equation*}
\end{definition}
\begin{definition}[Smoothness]
    We say that a loss function whose gradients exist everywhere is {\em smooth} if its gradients are Lipschitz continuous with some constant $L>0$, i.e., \revise{for all $\bw_1, \bw_2 \in \mathcal{W}$ and all $\bz \in \mathcal{Z}$, we have}
    \begin{align*}
        \|\nabla_{\bw}\ell(\bw_1, \bz) - \nabla_{\bw}\ell(\bw_2, \bz)\|_2\leq L\|\bw_1 -\bw_2\|_2.
    \end{align*}
\end{definition}
Going forward, we drop the subscript $\bw$ in $\nabla_\bw \ell(\bw,\bz)$ for notational compactness. Note that in the case of a smooth, convex (loss) function, gradient-based local search methods are guaranteed to converge to a global minimizer of the function. In addition, the global minimizer is unique for {\em strongly} convex functions and convergence of gradient-based methods to the minimizer of these functions is provably fast.

Our discussion %in the paper
revolves around both convex and (certain structured) nonconvex loss functions. Some of %this discussion
it %
in relation to convex %loss
losses %
%functions
requires an assumption on the variance of the gradients with respect to the data distribution.
\begin{definition}[Gradient Noise]\label{def:gradient.noise.variance}
    We say the gradients of $\ell(\cdot, \bz)$ have {\em bounded variance} if for every $\bw \in \mathcal{W}$, we have
    \begin{equation*}
        \E_{\bz \sim \mathcal{D}}\left\{\|\nabla \ell(\bw, \bz) - \nabla f(\bw)\|_2^2\right\} \leq \sigma^2.
    \end{equation*}
\end{definition}
In the following, we term $\sigma^2$ as the \emph{gradient noise variance}. In addition, we use the notion of \emph{single-sample covariance noise variance} in lieu of gradient noise variance in relation to our discussion of the nonconvex loss function associated with the $1$-PCA problem.
\begin{definition}[Sample-covariance Noise]\label{def:covariance.noise.variance}
    We say the single-sample covariance matrix $\bz\bz^\tT$ associated with data sample $\bz$ drawn from distribution $\cD$ has \emph{bounded variance} if we have
    \begin{equation*}
        \E_{\bz \sim \mathcal{D}}\left\{\left\|\bz\bz^\tT - \E_{\bz \sim \mathcal{D}}\left\{\bz\bz^\tT\right\}\right\|_F^2\right\} \leq \sigma^2.
    \end{equation*}
\end{definition}
The gradient (resp., sample covariance) noise variance controls the error associated with evaluating the gradient (resp., sample covariance) at individual sample points $\bz$ instead of evaluating it at the statistical mean of the unknown distribution $\mathcal{D}$. Smaller gradient (resp., sample covariance) noise variance results in faster convergence, and the main message of this paper is that leveraging distributed streams to average out gradient (resp., sample covariance) noise is often an optimum way to speed up convergence in compute- and/or communications-limited regimes.

The last definition we need is that of a \emph{bounded model space}, which plays a role in the analysis of optimization methods for convex loss functions.
\begin{definition}[Bounded model space]\label{def:bounded.model.space} Let $D_\mathcal{W} := \sqrt{\max_{\bu,\bv\in\cW}\|\bu - \bv\|_2^2/2}$ denote the \emph{expanse} of the model space~$\cW$. We say that an optimization problem has bounded model space if $D_\mathcal{W} < \infty$.
% We say that a optimization problem has {\em bounded model space}, with {\em diameter}
% \begin{equation*}
%     D_\mathcal{W} := \sqrt{\max_{\bu,\bv\in\cW}\|\bu - \bv\|_2^2/2}
% \end{equation*}
% if the diameter $D_\mathcal{W}$ is finite. We will suppose this condition and need the diameter for the analysis of machine learning problems with convex losses.
\end{definition}

%Because the focus of this paper is machine learning
Since our focus is training %
from fast streaming data that necessitates distributed processing%(cf.~Section~\ref{sec:intro}),
, %
we next formalize the distributed processing / communications framework underlying the %distributed machine learning
algorithms being discussed in the paper.

\subsection{Distributed Training of Machine Learning Models from Streaming Data}\label{subsec:DistributedLearning}
In addition to optimality, in the face of large volumes and high dimensionality of data in modern %machine learning
applications, the solution needs to be efficient in terms of resource utilization as well (e.g., computational, communication, storage, energy, etc.). In Section~\ref{sec:intro}, we discussed three mainstream distributed frameworks for resource-efficient machine learning, where each of the frameworks is primarily designed to adhere to specific practical constraints posed due to characteristics of the training data. One such characteristic is the physical locality of data, which results in following two common scenarios involving streaming data: ($i$) for the \emph{master--worker} learning framework, the data stream arrives at a single master node and, in order to ease the computational load and accelerate training time, the data stream is then divided among a total of $N$ worker nodes (Fig.~\ref{fig:intro.quad.chart}(b) and Fig.~\ref{fig:DistributedNetwork}(a)), or ($ii$) for the \emph{federated learning} and \emph{edge computing} frameworks, there is a collection of $N$ geographically distributed nodes---each of which receives its own independent stream of data---and the goal is to learn a machine learning model using information from all these nodes (Fig.~\ref{fig:intro.quad.chart}(c), Fig.~\ref{fig:intro.quad.chart}(d), and Fig.~\ref{fig:DistributedNetwork}(b)). Despite the apparent physical differences between these two scenarios, we can study them under a unified abstraction that assumes the data are arriving at a \emph{hypothetical} ``data splitter'' that then evenly distributes the data across an interconnected network of $N$ nodes for distributed processing (Fig.~\ref{fig:DistributedNetwork}(c)).

\begin{figure*}[t]
    \centering
    \includegraphics[width=0.75\textwidth]{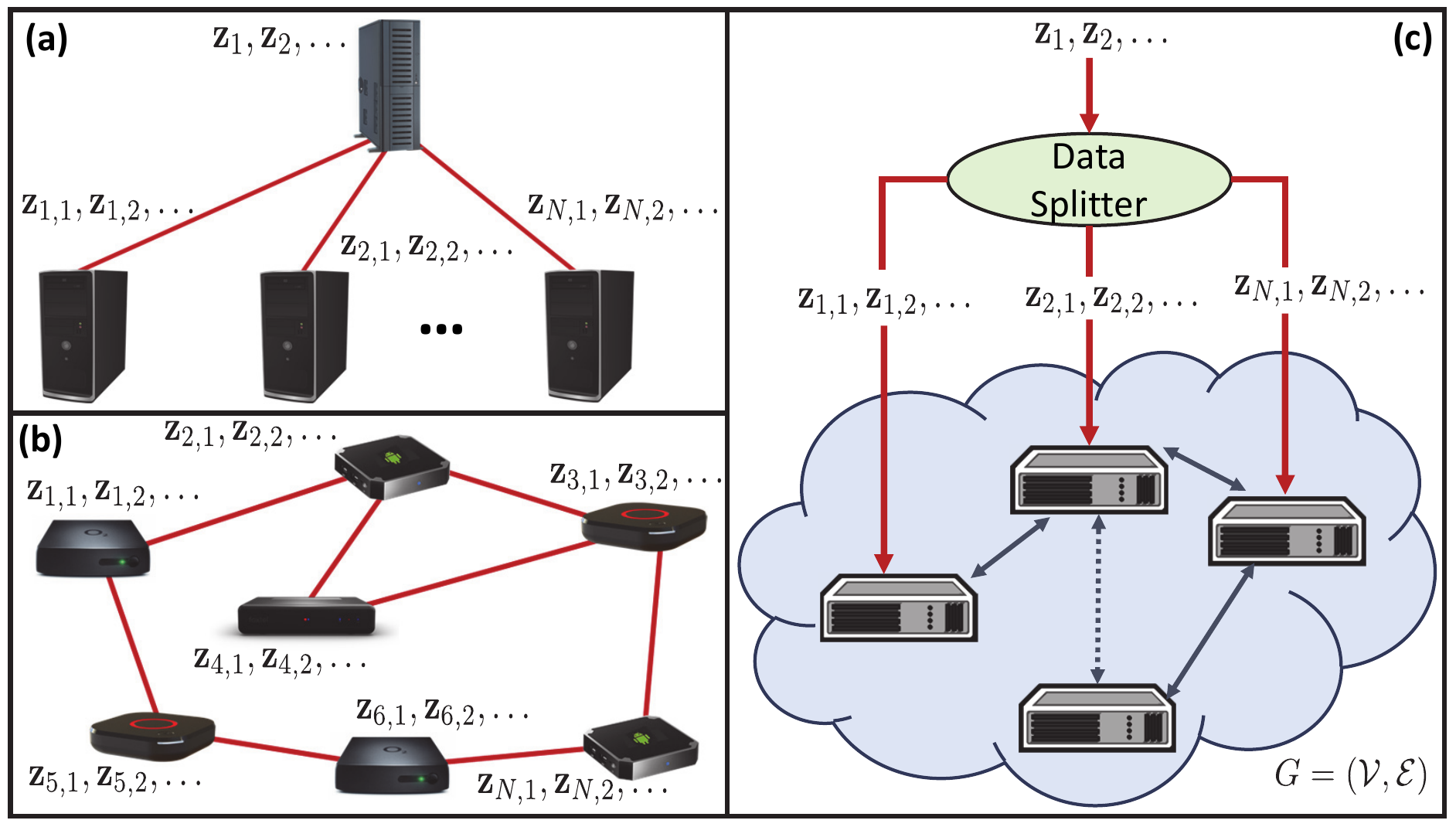}
    \caption{Distributed training of machine learning models from streaming data can arise in several contexts, including (a) the master--worker computing framework and (b) the edge computing and federated learning frameworks. In this paper, we study a unified abstraction (c) of such frameworks, in which a data stream is split into $N$ parallel streams, one for each compute node in a network $G = (\cV, \cE)$ of $N$ nodes.}
    \label{fig:DistributedNetwork}
\end{figure*}

% \begin{figure}[t]
% 	\centering
% 	\subfigure[]%Using a splitter to split data among different nodes in network.
%         {\includegraphics[width=0.4\columnwidth]{Figs/DistributedNetwork_2a}}
% 	\qquad
% 	\subfigure[]%Data arriving at different nodes in a network.]
%         {\includegraphics[width=0.4\columnwidth]{Figs/DistributedNetwork_1a}}
%     \caption{The distributed statistical problem, which involves distributed processing of data over a network of $N$ processors, can
%         arise in two contexts. (a) A data splitter can split a data stream into $N$ parallel streams, one for each processor in the network. In relation to the original data stream, this effectively reduces the data arrival rate for each parallel stream by a factor of $N$. (b) Data can be inherently distributed, as in the Internet-of-Things systems, and can arrive at $N$ different processing nodes as $N$ separate data streams.}
% 	\label{fig:DistributedNetwork}
% \end{figure}

Mathematically, let us discretize the data arrival time as $t'=1,2,\dots$, and let $\bz_{t'}$ be a stream of independent and identically distributed (i.i.d.) data samples arriving at the splitter at a fixed rate of $R_s$ samples per second. The splitter then evenly distributes the data stream across %an interconnected
a %
network of $N$ nodes, which we represent by an undirected connected graph $G := (\cV, \cE)$; here, $\cV := \nN{N}$ denotes the set of all nodes in the network and $\cE \subseteq \cV \times \cV$ denotes the set of edges corresponding to the communication links between these nodes, i.e., $(n,k) \in \cE$ means there is a communication link between nodes $n$ and $k$. We also define $t\in\Z_{+}$ to be the index that denotes the total number of data-splitting operations that have been performed within the system. Without loss of generality, we take each data-splitting round to be the time in which nodes carry out a single iteration of a distributed %optimization
algorithm; i.e., after $t$ data-splitting rounds, the nodes have carried out $t$ iterations of the distributed algorithm under study.

We next set the notation for the distributed data streams within our data-splitting abstraction to facilitate the prevalent practice of training %machine learning models
using mini-batches of data samples. To this end,
%. In order to facilitate analysis of the prevalent practice of mini-batch training of machine learning models within our data-splitting abstraction,
we assume without loss of generality that a total of $B \in \{N, 2N, 3N, \dots\}$ samples arrive at the network during each data-splitting round. That is, a \emph{system-wide} mini-batch of size $B$ is processed by the network during each algorithmic iteration (see, e.g., Fig.~\ref{fig:TimelinesNW}). Hence, each splitting operation results in a mini-batch of size $\tfrac{B}{N} \in \Z_+$ arriving at each node. The data splitting across $N$ nodes in the system therefore gives rise to $N$ i.i.d.~streams of \emph{mini-batched} data, where we denote the $\tfrac{B}{N}$ i.i.d.~data samples within the $t$-th mini-batch at node $n$ as $\left\{\bz_{n,b,t} \stackrel{\text{i.i.d.}}{\sim} \cD\right\}_{b=1,t \in \Z_+}^{B/N}$, with the mapping of these samples to the ones in the original data stream $\bz_{t'}$ given in terms of the relationship $t' = b+(n-1)B/N+(t-1)B$.

\revise{Given this distributed, streaming data model, our goal %in the paper
is study of machine learning algorithms that can efficiently process and incorporate the $B$ newest-arriving network-wide samples into a running approximation of the Bayes optimal solution (cf.~\eqref{eqn:StatisticalLearningCentralized}) \emph{before} the arrival of the next mini-batch of data. In order to highlight the challenges involved in the designs of such algorithms, we can divide the task of processing of a mini-batch of $B$ samples within the network into two phases (cf.~Fig.~\ref{fig:TimelinesNW}): ($i$) the \emph{computation phase}, in which each node performs computations over its \emph{local} mini-batch of $B/N$ data samples, and ($ii$) the subsequent \emph{communications phase}, in which nodes share the outcomes of their local computations with each other for eventual incorporation into the (network-wide, decentralized) machine learning models $\{\bw_{n,t}\}_{n \in \cV}$.}

\begin{figure}[t]
    \centering
    \includegraphics[width=0.55\textwidth]{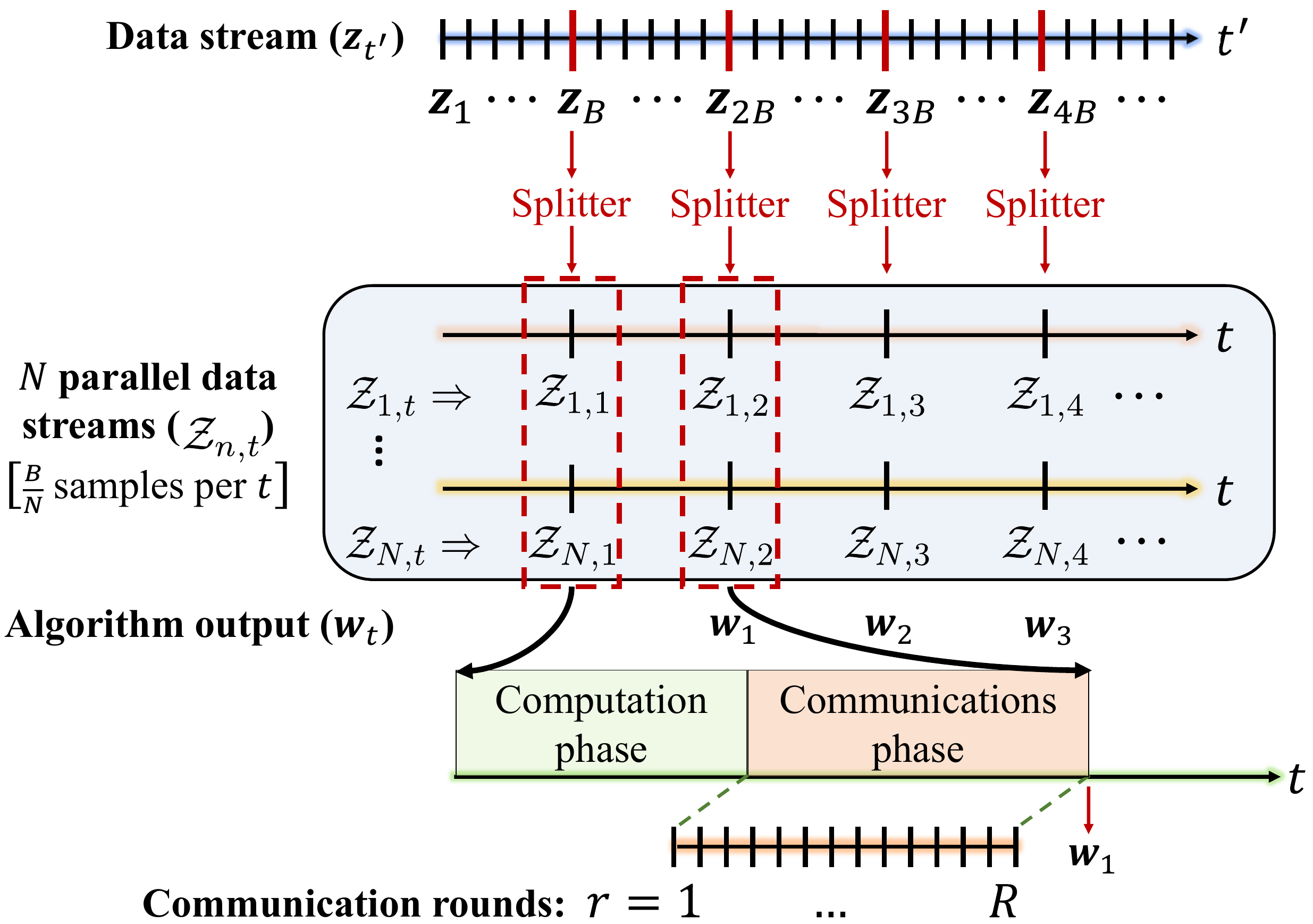}
    \caption{Data arrival, splitting, processing, and communications timelines within the distributed mini-batch framework of the paper. A stream of data, $\bz_{t'}$, arriving at a splitter at the rate of $R_s$ samples per second, is evenly split every $B$ samples across $N$ nodes in the network. This results in \emph{mini-batched} data streams $\cZ_{n,t} := \left\{\bz_{n,b,t}\right\}_{b=1,t \in \Z_+}^{B/N}$ within the network, with $\cZ_{n,t}$ denoting the $t$-th mini-batch of $B/N$ samples at node $n$. The system processes these distributed mini-batches by engaging in local computations followed by $R$ rounds of inter-node communications and produces an output $\bw_t$ before the arrival of the next set of mini-batches.}
    \label{fig:TimelinesNW}
\end{figure}

\revise{Consider now the \emph{compute-limited} regime within our framework, in which the distributed system comprising $N$ compute nodes is incapable of finishing computations on $B$ samples between two consecutive data-splitting instances because of the fast data streaming rate. (Indeed, the time between two data-splitting instances decreases as $R_s$ increases.) One could push the system out of this compute-limited regime by adding more compute nodes to the system. Keeping the system-wide mini-batch size $B$ fixed (and large), this will result in smaller local mini-batch size $B/N$. Alternatively, keeping the local mini-batch size fixed, this will result in larger time between two data-splitting instances. And in either case, the system no longer remains compute limited. However, as one adds more and more compute nodes into the distributed system, it could be pushed into the \emph{communications-limited} regime, in which the size/topology of the network prevents the nodes from completing full exchange of their local computations between two consecutive data-splitting instances. This communications-limited regime---which becomes especially pronounced in systems with slower communications links---can only be mitigated through larger data-splitting intervals, which in turn necessitates a larger $B$ for any fixed data streaming rate $R_s$. But this can again push the system into the compute-limited regime. Therefore, any machine learning algorithm intending to process fast streaming data in an optimal fashion must strike a balance between the compute- and the communications-limited regimes through judicious choices of system parameters such as $B$ and $N$. We now formalize some of this discussion in the following, which should lead to a better understanding of the interplay between the data streaming rate, the computational capabilities of compute nodes, the communications capabilities of the network, the system-wide mini-batch size $B$, and the number of compute nodes $N$ in distributed systems.}

\subsection{Interplay Between System Parameters in Distributed, Streaming Machine Learning}\label{subsec:HighrateStreaming}
We have already defined $R_s$ as the number of data samples $\bz_{t'}$ arriving per second at the splitter. We also assume the $N$ compute nodes in the system to be homogenous in nature and use $R_p$ to denote the processing/compute rate of each of these nodes, defined as the number of data samples per second that can be locally processed per node during the computation phase. Distributed algorithms also involve the use of message passing routines for inter-node communications. We use $R_c$ to denote the rate of messages shared among nodes using such routines, defined as the number of messages (synchronously) communicated between nodes per second during the communications phase. This parameter $R_c$ also subsumes within itself any overhead associated with implementation of the message passing routine such as time spent on additional computations or communications necessitated by the implementation.

Distributed machine learning algorithms typically involve multiple message passing rounds within the communications phase (cf.~Fig.~\ref{fig:TimelinesNW}), which we denote by $R \in \Z_+$. This parameter $R$, which we assume remains fixed for the duration of the training, can be constrained in terms of the system parameters $B$, $N$, $R_s$, $R_p$, and $R_c$ as follows:
\begin{align}\label{eqn:Rounds}
    0 < R \leq \Bigg\lfloor BR_c \Bigg(\frac{1}{R_s}-\frac{1}{NR_p}\Bigg)\Bigg\rfloor.
\end{align}
Our focus in this paper is on algorithms that make use of either ``exact'' or ``inexact'' \emph{distributed averaging} procedures within the communications phase for information sharing. Specifically, let $\{\bv_n\in\R^d\}_{n \in \cV}$ be a set of vectors that is distributed across the $N$ nodes in the network at the start of any communications phase and define $\widehat{\bv}_n$ to be an estimate of the average $\bar{\bv} := \tfrac{1}{N}\sum_{n\in\cV} \bv_n$ of these vectors at node $n$. We then have the following communications-related characterizations of the algorithms being studied in the paper.
\begin{enumerate}
    \item \textbf{Exact averaging algorithms.} After $R$ message passing rounds within the communications phase, these algorithms can exactly estimate the average at each node, i.e., $\forall n \in \cV, \|\widehat{\bv}_n-\bar{\bv}\|_2=0$.
    \item \textbf{Inexact averaging algorithms.} After $R$ message passing rounds within the communications phase, these algorithms can only guarantee $\epsilon$-accurate estimates, i.e., $\forall n \in \cV, \|\widehat{\bv}_n-\bar{\bv}\|_2 \leq \epsilon$ for some parameter $\epsilon > 0$ that typically increases as $R$ decreases and/or $N$ increases.
\end{enumerate}
Exact averaging algorithms often find applications in settings like high-performance computing clusters and enterprise cloud computing systems, where communications is typically fast and reliable. In contrast, inexact averaging algorithms tend to be more prevalent in settings like edge computing systems, multiagent systems, and IoT systems, where the network connectivity can be sparser and the communications tend to be slower and unreliable.

\revise{We have now described all the system parameters needed to formalize the notion of \emph{effective (mini-batch) processing rate}, $R_e$, of the distributed system, which is defined as the number of mini-batches comprising $B$ samples that can be processed by the system per second. (In the non-distributed setting, corresponding to $N=1$, it is straightforward to see that $R_e = R_p/B$.) Under the assumption of a synchronous system in which computation and communications phases are carried out one after the other, the parameter $R_e$ can be defined as follows:
\begin{align}\label{eqn:Re}
R_e := \frac{1}{\text{time spent in computation} + \text{time spent in communication} } = \frac{1}{\tfrac{B}{N R_p}+\tfrac{R}{R_c}}=\Bigg(\frac{B}{N R_p} + \frac{R}{R_c}\Bigg)^{-1}.
\end{align}
This expression formally highlights the tradeoff between the compute-limited and the communications-limited regimes. In the case of fixed $R_p$ and $R_c$, for instance, increasing the effective processing rate requires an increase in $N$. Doing so, however, necessitates an increase in $R$ that---beyond a certain point---can only be accomplished through an increase in $B$ (cf.~\eqref{eqn:Rounds}), which in turn also increases the first term in \eqref{eqn:Re}.}

\revise{The overarching theme of this paper is discussion of algorithmic strategies that can be used to tackle the challenge of \emph{near-optimal} training of machine learning models from fast streaming data, where ``fast'' is defined in the sense that $\frac{R_s}{R_e} \gg B$. This discussion involves allowable selections of system parameters such as the network-wide mini-batch size $B$, the number of nodes $N$, and the number of communications rounds $R$ that facilitate taming of the fast incoming data stream without compromising the fidelity of the final solution. In particular, the recommended strategies end up pushing the ratio $\frac{R_s}{R_e}$ to satisfy either $\frac{R_s}{R_e} \leq B$ or $\frac{R_s}{R_e} = \left(B + \mu\right)$ for an appropriate parameter $\mu \in \Z_+$, where the latter scenario involves discarding of $\mu$ samples per splitting instance $t$ at the data splitter.}

\revise{In order to prime the reader for subsequent discussion, we also provide a simple example in Fig.~\ref{fig:DifferentRates} that illustrates the impact of the choice of (network-wide) mini-batch size $B$ on system performance. We suppose a network of $N=10$ compute nodes, and focus on the exact averaging paradigm described above. We assume a data streaming rate of $R_s=10^{6}$ samples per second, whereas the data processing rate per node is taken to be $R_p=1.25\times 10^{5}$ samples per second. We plot the ratio of the streaming rate and the effective (mini-batch) processing rate $R_e$ as defined in \eqref{eqn:Re}, for communications rates $R_c = 10^3$ and $R_c = 10^4$, as a function of the mini-batch size $B$. As noted earlier, the number of samples effectively processed by the network keeps pace with the number of samples arriving at the system provided $R_s/R_e \leq B$, and we observe that for sufficiently large mini-batch size $B$, the ratio indeed drops below the $R_s/R_e = B$ line plotted in Fig.~\ref{fig:DifferentRates}.}

\revise{Next, we also overlay corresponding plots of the excess risk predicted for {\em Distributed Minibatch} SGD, presented in Section~\ref{subsec:ConvexMinibatch}, after $t'=10^6$ samples have arrived at the system. These plots show that increased mini-batch size helps the excess risk, but only to a point. Eventually, $B$ becomes so large that the reduction in the number of algorithmic iterations carried out by the network hurts the overall performance more than the increase in the effective processing rate helps it. This illustrates that the mini-batch size $B$ must be chosen judiciously, and in the following sections we will discuss theoretical results that shed light on this choice.}

\begin{figure}
\centering
%\subfigure[Impact of the increasing number of processors on the total processing rate of the network. We consider two network topologies in this figure, first one is a star topology which results in a constant communication rate of $R_c=0.1$ while the other is for a tree-structured network which results in a communication rate of $O(R_c/log(N))$.]{\includegraphics[width=0.3\columnwidth]{Figs/Rate_Vary_N.eps}\label{fig:NodesImpact}}
%	\qquad
 	\includegraphics[width=0.5\textwidth]{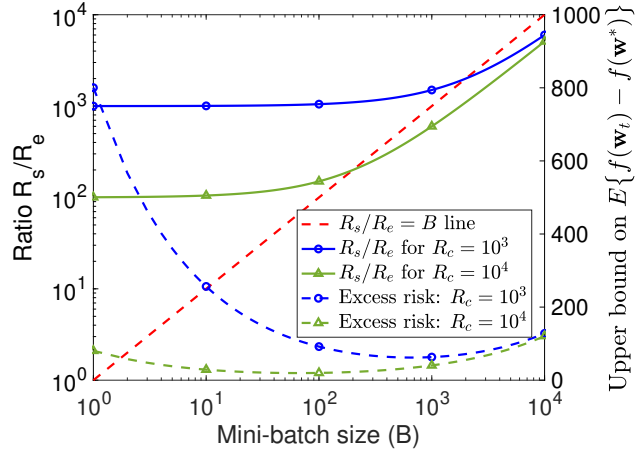}
%  	\subfigure[]{\includegraphics[width=0.4\columnwidth]{Figs/Rate_Vary_B.eps}}
%  	\qquad
%  	\subfigure[]{\includegraphics[width=0.4\columnwidth]{Figs/Consensus_Vary_B.eps}}
 \caption{An illustration of the impact of mini-batching on distributed, streaming processing under the exact averaging paradigm.
% , $N=10$ processing nodes in the network, and it is assumed that we will only observe $T=10^6$ in total. These examples show the impact of mini-batch size, $B$, on \emph{effective processing rate} $R_e$ and \emph{excess risk} $\E\{f(\bw_t) - f(\bw^*)\}$. For $\sigma=100$ and $L=1$, the order-wise excess risk of $O(L/t+\sigma/\sqrt{t})$ which is given in Theorem~\ref{thm:SGD_Mini}, is plotted here. (a) Example illustrating the relationship between high streaming rate of data $R_s$ and the effective processing rate $R_e$ of network. A network of $N=10$ nodes is considered and rates used are $R_s=10^6$ samples per second, $R_p=10^5$ samples per second. For two choices of communication rates $R_c=10^4$ and $R_c=10^5$ samples per second we vary mini-batch size $B$ from 1 to $10^4$. Important observation here is that for slower communication rate we need larger mini-batch size to achieve $R_e\geq R_s$. (b) This plot explains the relationship between $R_s$ and $R_e$ as we vary mini-batch size $B$ for algorithms requiring multiple rounds of communication, $R$, during the communication phase. For these example plots the communication rate is assumed to be $R_c=10^5$ samples per second. We can see that more rounds of communication (larger value of $R$) require larger mini-batch sizes to make $R_e\geq R_s$.
}
 	\label{fig:DifferentRates}
 \end{figure}

\section{\revise{An Overview of the Technical Landscape}}\label{sec:RelatedWork}
This paper ties together research in optimization and distributed processing within the context of machine learning. To elucidate the state of the art and set the stage for the results described in Sections~\ref{sec:DM_Krasulina} and~\ref{sec:Decentralized}, we present an overview of these areas and describe in detail key results that will be used later.

\subsection{Optimization for Machine Learning}\label{subsec:Optimization_ML}
As mentioned in Section~\ref{ssec:intro.framing}, the literature on optimization for machine learning \revise{can be roughly divided into two interrelated frameworks, namely, \emph{statistical risk minimization} (SRM) and \emph{empirical risk minimization} (ERM). Both these frameworks aim to find a solution to the statistical optimization problem~\eqref{eqn:StatisticalLearningCentralized} and, as such, fall under the broad category of \emph{stochastic optimization} (SO) within the optimization literature~\cite{Spall.BookCh12}. In particular, the SRM framework is often referred to as \emph{stochastic approximation} (SA) and the ERM framework is sometimes termed \emph{sample-average approximation} (SAA) in the literature~\cite{KimPasupathyEtAl.BookCh15}. In terms of specifics, the SA/SRM framework} considers directly the statistical learning problem~\eqref{eqn:StatisticalLearningCentralized}, and researchers have developed algorithms that minimize the risk $f(\bw)$ using ``noisy'' (stochastic) samples of its gradient $\nabla f(\bw)$. \revise{In contrast, the risk in the SAA/ERM framework} is approximated by the empirical distribution over a fixed training dataset $\cZ_T := \{\bz_1, \dots, \bz_T\}$ of $T$ data samples. This \emph{empirical} risk, defined as $\hat{f}(\bw) := \frac{1}{T}\sum_{t'=1}^T \ell(\bw, \bz_{t'})$, is then minimized directly within the ERM framework, usually via some form of gradient-based (first-order) deterministic optimization methods. In the following, we describe a few key results from these two frameworks that are the most relevant to our discussion in this paper.

\subsubsection{Stochastic \revise{Approximation} (SA)}\label{sect:so}
The general assumption within the SA framework is that one has access to a stream of noisy gradients $\{\bg_1, \bg_2, \dots\}$ in order to solve \eqref{eqn:StatisticalLearningCentralized}, where the noisy gradient $\bg_t$ at iteration $t$ is defined as
\begin{equation}
    \bg_t := \nabla f(\bw_t) + \bzeta_t,
\end{equation}
with $\bzeta_t$ denoting i.i.d. noise with mean zero and finite variance, i.e., $\E\{\| \bzeta_t \|_2^2\} \leq \sigma^2$. In the parlance of SA, we have access to a first-order ``oracle'' that can be queried for a noisy gradient evaluated at the query point $\bw_t$. In the parlance of machine learning, we have a stream of data samples $\{\bz_1, \bz_2, \dots \}$, each drawn i.i.d. according to the data distribution $\mathcal{D}$, and we solve \eqref{eqn:StatisticalLearningCentralized} using the gradients $\bg_t := \nabla \ell(\bw_t, \bz_t)$, which have gradient noise variance as defined in Definition~\ref{def:gradient.noise.variance}.\footnote{Note that the data arrival index $t'$ and the algorithmic iteration index $t$ are one and the same in a centralized setting; we are using $t$ here in lieu of $t'$ to facilitate comparisons with results in distributed settings.} It is straightforward to verify that these two formulations are equivalent: $\E \{\bg_t\} = \nabla f(\bw_t)$, so we can define $\bzeta_t := \bg_t - \nabla f(\bw_t)$ to be the zero-mean gradient noise in our problem setup.

The prototypical SA algorithm for loss functions whose gradients exist is \emph{stochastic gradient descent} (SGD)~\cite{RobbinsMonro.AMS51}, in which iterations/iterates take the form
\begin{equation}\label{eqn:SGD}
    \bw_{t+1} = \left[ \bw_{t} - \eta_t \bg_t \right]_\mathcal{W},
\end{equation}
where  $[\cdot]_\mathcal{W}$ denotes projection onto the constraint set $\cW$ \revise{and $\eta_t > 0$ denotes an appropriate \emph{stepsize} that is either fixed (\emph{constant} stepsize) or that decays to $0$ with increasing $t$ according to a prescribed strategy (\emph{decaying} stepsize)}.
\begin{remark}
The term `stochastic gradient descent' is overloaded in the literature. Many papers (e.g., \cite{bach2011nonasymptotic,Shamir.ConfICML16}) use the term in the SA sense described here, with a continuous stream of data in which no sample is used more than once. However, other papers (e.g., \cite{bottou2010large,BijralSarwateEtAl.ITAC17}) use the term within the ERM framework to describe algorithms that operate on a fixed dataset, from which mini-batches of data are sampled with replacement and noisy gradients are computed. To disambiguate, some authors (e.g., \cite{FrostigGeEtAl.ConfCOLT15}) use the term {\em single-pass} SGD to indicate the former usage.
\end{remark}

\textbf{Convex Problems.} A common elaboration on SGD for convex loss functions is {\em Polyak--Ruppert averaging} \cite{Polyak.AIT90,Ruppert.TechRep88,polyak1992acceleration,bach2011nonasymptotic}, in which a running average of iterates $\bw_t$ is maintained as $\bw_t^{\textsf{av}} := \frac{1}{t}\sum_{\tau=0}^{t-1} \bw_\tau.$ The convergence rates of SGD for convex SA have been studied under a variety of settings, both with and without iterate averaging. The following result with a modified form of Polyak--Ruppert averaging comes from \cite{lan2012optimal}, in which iterate averaging takes the form
\begin{align}
    \label{eqn:Lan.iterate.ave}
    \bw_t^{\textsf{av}} := \left(\sum_{\tau = 0}^{t-1}\eta_\tau\right)^{-1} \sum_{\tau=0}^{t-1}\eta_\tau\bw_{\tau+1}.
\end{align}
\begin{theorem}\label{thm:SGD_Mini}
    For convex and smooth loss functions $\ell(\bw, \bz)$ \revise{with (gradient) Lipschitz constant $L$}, gradient noise variance $\sigma^2$, and bounded model space with expanse $D_\mathcal{W}$, there exist stepsizes $\eta_t$ such that the approximation error of SGD with iterate averaging in \eqref{eqn:Lan.iterate.ave} satisfies:
    \begin{equation}
        \label{eqn:thm.SGD.rate}
        \E\{f(\bw_t^{\textsf{\emph{av}}})\} - f(\bw^*) = O(1) \left[ \frac{L}{t} + \frac{\sigma}{\sqrt{t}} \right].
    \end{equation}
\end{theorem}
\begin{remark}\label{rem:stepsize}
\revise{In \cite{lan2012optimal}, an optimal {\em constant} stepsize $\eta_t = \eta$ is given in the case where the optimization ends at a finite time horizon $t := T$ known in advance. In this case, the prescribed stepsize is $\eta := \min\left\{1/(2L), \sqrt{D_\mathcal{W}^2/(2T)} \right\},$ and this achieves the bound %stated in Theorem~\ref{thm:SGD_Mini}
in \eqref{eqn:thm.SGD.rate}%
. When the time horizon is unknown, a varying stepsize policy $\eta_t = O(1/\sqrt{t})$ achieves expected excess risk $O(\sigma/\sqrt{t})$, which is optimum for $t$ much larger than $L$. For simplicity, we are working with the optimum stepsize proposed in \cite{lan2012optimal} to retain the analysis for $t$ not necessarily much larger than $L$.}
\end{remark}

\begin{remark}
\revise{It is desirable in some applications to state the SGD results in terms of convergence of the averaged iterate $\bw_t^{\textsf{av}}$ to $\bw^*$. In the case of convex, smooth, and twice continuously differentiable loss functions, \cite{polyak1992acceleration} provides such results for Polyak--Ruppert averaging in the almost sure sense and also proves asymptotic normality of $\bw_t^{\textsf{av}}$, i.e., $\sqrt{t}(\bw_t^{\textsf{av}} - \bw^*)$ converges to a zero-mean Gaussian vector. In the case of strongly convex and smooth loss functions, \cite{bach2011nonasymptotic} derives non-asymptotic convergence results for the Polyak--Ruppert averaged iterate $\bw_t^{\textsf{av}}$ in the mean-square sense. However, since machine learning is often concerned with minimizing the excess risk $\E\{ f(\bw_t^{\textsf{av}}) - f(\bw^*) \}$, we do not indulge further in discussion of convergence of the SGD iterates to the Bayes optimal solution $\bw^*$.}
\end{remark}

A natural question is whether the convergence rate of Theorem~\ref{thm:SGD_Mini} can be improved upon by another algorithm. It has been shown that incorporating {\em Nesterov's acceleration}~\cite{nesterov1983method} into SGD can indeed improve this rate somewhat. Roughly speaking, Nesterov's acceleration introduces a ``momentum'' term into the SGD iterations, allowing the directions of previous gradients to impact the direction taken during the current step and thereby speeding up convergence. The following formulation is an SGD-based simplification of the {\em accelerated stochastic mirror descent} algorithm of \cite{lan2012optimal}. Define the {\em accelerated SGD} updates as follows:
\begin{align}\label{eqn:ASGD}
    \bu_t &= \beta_t^{-1}\bv_t + (1-\beta_t^{-1})\bw_t,\\
    \bv_{t+1} &= \left[ \bu_{t} - \eta_t \bg_t \right]_\cW, \ \text{and}\\
    \bw_{t+1} &= \beta_t^{-1}\bv_{t+1} + (1 - \beta_t^{-1})\bw_{t},
    %\bu_t &= \beta_t^{-1}\bw_t + (1-\beta_t^{-1})\bv_t,\\
    %\bw_{t+1} &= \left[ \bu_{t} - \eta_t \bg_t \right]_\cW, \ \text{and}\\
    %\bv_{t+1} &= \beta_t^{-1}\bw_{t+1} + (1 - \beta_t^{-1})\bv_{t},
\end{align}
where $\bg_t := \nabla f(\bu_t) + \bzeta_t$, and $\beta_t > 0$ and $\eta_t > 0$ are stepsizes. We then have the following result from \cite{lan2012optimal}.
\begin{theorem}\label{thm:ASGD}
    For convex and smooth loss functions $\ell(\bw, \bz)$ \revise{with (gradient) Lipschitz constant $L$}, gradient noise variance $\sigma^2$, and bounded model space with expanse $D_\mathcal{W}$, there are stepsizes $\eta_t$ and $\beta_t$ such that the expected risk of accelerated SGD is bounded by
    \begin{equation}
        \label{eqn:thm.ASGD.rate}
        \E\{f(\bw_t)\} - f(\bw^*) = O(1) \left[ \frac{L}{t^2} + \frac{\sigma}{\sqrt{t}} \right].
    \end{equation}
\end{theorem}
\begin{remark}
\revise{Similar to standard SGD, \cite{lan2012optimal} prescribes stepsizes in the case of known and finite time horizon $T$, with $\beta_t = t/2$ and $\eta = t/2\min \{1/(2L), \sqrt{6}/D_\mathcal{W}/(\sigma(T+1)^{3/2}) \}$. Again a varying stepsize policy achieves excess risk $O(\sigma/\sqrt{T})$ for large $t$, and we suppose the optimum stepsize given in \cite{lan2012optimal} in order to facilitate analysis for $t$ not necessarily much larger than $L$.}
\end{remark}

Both Theorem~\ref{thm:SGD_Mini} and Theorem~\ref{thm:ASGD} explicitly bring out the dependence of the convergence rates on the gradient noise variance $\sigma^2$. In doing so, they hint at the potential performance advantages of (centralized or distributed) mini-batching of data. As the number of samples/iterations $t$ goes to infinity and all else is held constant, the $O(\sigma/\sqrt{t})$ terms dominate the convergence rates in \eqref{eqn:thm.SGD.rate} and \eqref{eqn:thm.ASGD.rate}. Mini-batching can reduce the ``equivalent'' noise variance $\sigma^2$ of the mini-batched data samples and speed up convergence, but only to a point. If mini-batching forces the $O(\sigma/\sqrt{t})$ term smaller than the respective first terms in \eqref{eqn:thm.SGD.rate} and \eqref{eqn:thm.ASGD.rate}, then gradient noise is no longer the bottleneck to performance and mini-batching cannot improve convergence speed any further. Indeed, in the sequel we will choose the mini-batch size to carefully balance the two terms in \eqref{eqn:thm.SGD.rate} and \eqref{eqn:thm.ASGD.rate}, and we will further see that more aggressive mini-batching is advantageous when using accelerated methods.
%
% *****WUB: Replace by the previous paragraph for the camera-ready version*****
%Both Theorem~\ref{thm:SGD_Mini} and Theorem~\ref{thm:ASGD} explicitly bring out the dependence of the convergence rates on gradient noise variance $\sigma^2$. And in doing so, they hint at the potential advantages of (centralized or distributed) mini-batching of data for improved performance. In both instances, as the number of samples/iterations $t$ goes to infinity and all else is held constant, the $O(\sigma/\sqrt{t})$ terms dominate the convergence rates in \eqref{eqn:thm.SGD.rate} and \eqref{eqn:thm.ASGD.rate}. But if the $O(\sigma/\sqrt{t})$ terms can be made to decay faster than their respective first terms then both SGD and accelerated SGD result in improved asymptotic convergence rates in the limit of many data (or gradient) samples. %
%
% *****WUB: Replaced by the previous set of sentences*****
% Whenever $\frac{\sigma}{\sqrt{t}}$ decays faster than $L/t$, accelerated SGD gives an improved asymptotic convergence rate in the limit of many data (or gradient) samples. However, if all else is held constant, the $O(\sigma/\sqrt{t})$ term dominates the error as the number of samples/iterations $t$ goes to infinity.

We conclude our discussion of SA for convex loss functions by noting that the convergence rate of accelerated SGD is provably optimal for smooth, convex SA problems in the {\em minimax} sense: there is no single algorithm that can converge for all such SA problems at a rate faster than $O(L/t^2 + \sigma/\sqrt{t})$. (See \cite{lan2012optimal} for an argument for this.) However, generalized and sometimes improved rates are possible outside of the regime of this setting. In particular, when $\ell(\bw, \bz)$ is smooth and {\em strongly} convex, a convergence rate of $O(\sigma^2/t)$ is possible for $\sigma^2$ bounded away from zero, and it is the minimax rate \cite{kushner2003stochastic,bach2011nonasymptotic}. Results are also available when the loss function is non-smooth, when the solution is sparse or otherwise structured, and when the optimization space has a geometry that can be exploited to speed up convergence \cite{agarwal2009information,xiao2010dual,lan2012optimal}.

\textbf{Nonconvex Problems.} \revise{Nonconvex functions can have three types of critical points, defined as points $\bw$ for which $\nabla f(\bw) = \bzero$: \emph{saddle points}, \emph{local minima}, and \emph{global minima}. This makes optimization of nonconvex (loss) functions using only first-order (gradient) information challenging. While works such as~\cite{ghadimi2013stochastic,ghadimi2016accelerated,ge2015escaping,hazan2016graduated,hazan2017near,reddi2016stochastic,reddi2016bfast} provide convergence rates for nonconvex problems that are similar to their convex programming analogs, the convergence is only guaranteed to a critical point that is not necessarily a global optimum. Nonetheless, global optimization of nonconvex SA problems has been studied in the literature under a variety of assumptions on the geometry of objective functions. A major strand of work in this direction involves modifying the canonical SGD algorithm by injecting slowly decreasing Monte Carlo noise in its iterations. The resulting SA methods have been investigated in works such as~\cite{Kushner.SJAM87,GelfandMitter.SJCAO91,FangGongEtAl.SJCO97,Yin.SJO99,MaryakChin.ITAC08,RaginskyRakhlinEtAl.ConfCOLT17,ErdogduMackeyEtAl.ConfNIPS18} under the monikers of (continuous) \emph{simulated annealing} and \emph{stochastic gradient Langevin dynamics}. (Strictly speaking,~\cite{ErdogduMackeyEtAl.ConfNIPS18} does not fall under the SA framework being discussed in this section.) A recent work~\cite{ShiSuEtAl.arxiv20} also provides global convergence guarantees for SGD for the class of (nonconvex) Morse functions.}

Another major strand of work in global optimization of nonconvex functions involves explicit exploitation of the geometry of \emph{structured} nonconvex problems such as \emph{principal component analysis} (PCA), dictionary learning, phase retrieval, and low-rank matrix completion for global convergence guarantees. In this paper, we focus on one such structured nonconvex SA problem that corresponds to estimating the top eigenvector $\bw^* \in \R^d$ of the covariance matrix $\bSigma \in \R^{d \times d}$ of i.i.d.~samples $\{\bz_1,\bz_2,\dots\} \subset \R^d$. The investigation of this \emph{$1$-PCA} problem in the paper, whose global convergence behavior has been investigated in works such as~\cite{Balsubramani2015,jain2016streaming,allen2016variance,Shamir.ConfICML16}, serves two purposes. First, it helps validate the generality of the main message of this paper that the mismatches between the data streaming rate, compute rate, and communications rate can be accounted for through judicious choices of system parameters such as $R$, $B$, and $N$. Second, it helps crystallize the key characteristics of any global convergence analysis of nonconvex problems that can facilitate the convergence speed-up guarantees for the distributed mini-batch framework.

The loss function for the $1$-PCA problem under the assumption of zero-mean distribution $\cD$ supported on $\R^d$ and having covariance matrix $\bSigma := \E_{\bz \sim \cD}\{\bz\bz^\tT\}$ takes the form
\begin{equation}\label{eqn:PCA_Loss}
    \ell(\bw, \bz) = - \frac{\bw^\tT \bz \bz^\tT \bw}{\| \bw \|_2^2}.
\end{equation}
Note that $\nabla \ell(\bw, \bz) = -\frac{2\bz\bz^\tT\bw}{\|\bw\|_2^2} + \frac{2(\bw^\tT\bz\bz^\tT\bw)\bw}{\|\bw\|_2^4}$ and the optimal solution $\bw^* \in \argmin_{\bw} \left[f(\bw) := \mathbb{E}_{\bz \in \mathcal{D}} \ell(\bw, \bz)\right]$ corresponds to the dominant eigenvector of $\bSigma$. In this paper, we focus on the SA approach termed Krasulina's method~\cite{krasulina1969method} that approximates the optimal solution $\bw^*$ from data stream $\{\bz_1, \bz_2, \dots\}$ using iterations of the form
\begin{align}\label{eqn:Krasulina}
    \revise{\bw_t=\bw_{t-1}+\eta_t\Big(\bz_{t}\bz_{t}^{\tT}\bw_{t-1}-\frac{\bw_{t-1}^{\tT}\bz_{t}\bz_{t}^{\tT}\bw_{t-1}}{\|\bw_{t-1}\|_2^2}\,\bw_{t-1}\Big)}.
\end{align}
Notice that changing $\eta_t$ to $\frac{\eta_t}{\|\bw_{t-1}\|_2^2}$ in \eqref{eqn:Krasulina} gives us the SGD iteration. Despite the empirical success of SA iterations such as \eqref{eqn:Krasulina} in approximating the top eigenvector of $\bSigma$, earlier works only provided asymptotic convergence guarantees for such methods. Recent studies such as \cite{Balsubramani2015,Shamir.ConfICML16,jain2016streaming,allen2016first,tang2019exponentially} have filled this gap by providing non-asymptotic results. The following theorem, which is due to \cite{Balsubramani2015}, provides guarantees for Krasulina's method.
\begin{theorem}\label{thm:Oja}
    Let the i.i.d.~data samples be bounded, i.e., $\forall t, \|\bz_t\|_2 \leq \kappa$, define $\gap := \lambda_1(\bSigma)-\lambda_2(\bSigma) > 0$, fix any $\delta \in (0,1)$, and define $c := \frac{c_0}{2\gap}$ for any $c_0 > 2$. Next, pick any
	\begin{align}\label{eqn:LowerboundL_1}
	   Q\geq \frac{512 e^2 d^2 \kappa^4 \max(1, c^2)}{\delta^4}\ln\frac{4}{\delta}
	\end{align}
	and choose the stepsize sequence as $\eta_t := c/(Q+t)$. Then there exists a sequence $(\Omega_{t}^{'})_{t \in \Z_+}$ of nested subsets of the sample space $\Omega$ such that $\bbP\left(\cap_{t>0}\Omega_t^{'}\right)\geq 1-\delta$ and
	\begin{align}\label{eqn:FinalResult_1}
	\E_t\big\{f(\bw_t)\big\} - f(\bw^*) \leq C_1\Big(\frac{Q + 1}{t + Q + 1}\Big)^{\frac{c_0}{2}} + C_2\Big(\frac{\kappa^2}{t + Q + 1}\Big),
	\end{align}
    where $\E_t$ is the conditional expectation over $\Omega_t^{'}$, and $C_1$ and $C_2$ are constants defined as
	$$C_1 := \frac{\lambda_1(\bSigma)}{2}\Bigg(\frac{4ed}{\delta^2}\Bigg)^{\frac{5}{2\ln2}}e^{2c^2\lambda_1^2(\bSigma)/Q}\quad\textnormal{and}\quad C_2 := \frac{2 c^2 \lambda_1(\bSigma) e^{(c_0+2c^2\lambda_1^2(\bSigma))/Q}}{(c_0-2)}.$$
\end{theorem}

The convergence guarantees in Theorem~\ref{thm:Oja} depend on problem parameters such as $d$, $\gap$, and $\delta$. Recent works~\cite{allen2016first,simchowitz2018tight} have provided lower bounds on the dependence of convergence rates on these parameters for the stochastic PCA problem. Theorem~\ref{thm:Oja} achieves these lower bounds with respect to $\gap$ and $\delta$ up to logarithmic factors. But the dependence on data dimension in Theorem~\ref{thm:Oja} is $d^4$, while the lower bound suggests $\Omega(\log(d))$ dependence. In addition, convergence guarantees for a variant of Krasulina's method termed Oja's algorithm are known to achieve this lower bound dependence on data dimensionality~\cite{allen2016first,jain2016streaming,yang2018history,tang2019exponentially}.

Despite this somewhat suboptimal nature of Theorem~\ref{thm:Oja}, Krasulina's method lends itself to relatively simpler analysis for the distributed (mini-batch) framework being studied in this paper. Specifically, as alluded to in our discussion in Section~\ref{subsec:StatisticalLearning}, implicit averaging out of the sample-covariance noise is the key reason for the potential speed-up in convergence within any distributed processing framework. And while Theorem~\ref{thm:Oja} does not have an explicit dependence on the noise variance $\sigma^2$, a variance-based analysis of Krasulina's method---discussed in detail in Section~\ref{sec:DM_Krasulina} and having similar dependence on $d$, $\gap$, and $\delta$ as Theorem~\ref{thm:Oja}---has been provided in a recent work~\cite{raja2020distributed}. In contrast, results in~\cite{allen2016first,tang2019exponentially} are oblivious to the variance in sample covariance and hence cannot be used to show faster convergence within distributed frameworks. On the other hand, while the results in~\cite{jain2016streaming,yang2018history} do take the noise variance into account, the probability of success in these works cannot be improved beyond $3/4$ in a single-pass SA setting.
%
% *****WUB: Replaced by the previous paragraph*****
% Our motivation for focusing on result in Theorem~\ref{thm:Oja} is the availability of variant of Krasulina's for distributed and mini-batch settings, which is not the case for Oja's algorithm as explained in the following. As already explained in this section, mini-batching reduces variance of stochastic input which results in faster convergence rate, but result in Theorem~\ref{thm:Oja} and results in~\cite{allen2016first,tang2019exponentially} are oblivious to the variance in input and hence cannot be used to show faster convergence using mini-batching. On the other hand, results in \cite{jain2016streaming,yang2018history} take variance into account, but limitation here is that probability of success cannot be improved beyond 3/4 in streaming settings. Recently, variance based result for Krasulina's method has been provided in \cite{raja2020distributed}, we discuss this result in detail in Section~\ref{sec:DM_Krasulina}. Briefly stating, the result in \cite{raja2020distributed} has similar dependence on $\gap$, $d$, and $\delta$ as Theorem~\ref{thm:Oja}, while the lower bound on choice of constant $Q$ decreases which improves convergence rate, and in terms of dependence on iterations $t$ the result changes from $O(m^2/t)$ to $O(\sigma^2/t)$, which spells improvement in convergence rate by virtue of decreasing variance via averaging of iterates in a mini-batch.

\subsubsection{Empirical Risk Minimization (ERM)}
Given the \emph{fixed} training dataset $\cZ_T$ of $T$ i.i.d.~samples drawn from the distribution $\cD$ and the corresponding empirical risk $\hat{f}(\bw)$, the main objective within the ERM framework is to directly minimize $\hat{f}(\bw)$ in order to obtain the ERM solution $\bw^*_{\textsc{erm}} \in \argmin_{\bw \in \cW} \hat{f}(\bw)$. Such problems, sometimes referred to as {\em finite-sum optimization problems}, have traditionally been solved using (deterministic, projected) gradient descent or similar methods. %
%
% *****WUB: Replaced by the previous set of sentences*****
% Letting $\bw^*_{\textsc{erm}}$ denote the ERM solution, the gap between the empirical risk $\hat{f}(\bw^*_{\textsc{erm}})$ and the true risk $f(\bw^*_{\textsc{erm}})$ is called the {\em generalization error}, which is analyzed via cross-validation or formal analysis \cite{vapnik1999statistical}. Such problems are sometimes called {\em finite-sum optimization problems}. In empirical risk minimization, we suppose access to a fixed dataset $D$ of $T$ i.i.d. samples from the distribution $\mathcal{D}$, and that the objective is to minimize directly the finite-sum empirical risk $\hat{f}(\bw)$.
%
% The simplest approach to ERM is to minimize $\hat{f}(\bw)$ directly via (deterministic) gradient descent or similar methods.
%
But the advent of massive datasets has made direct computations of gradients of $\hat{f}(\bw)$ intractable. %which has led to the development of families of stochastic gradient descent\footnote{Here meant in the {\em finite-sum} sense described above.} algorithms applied to ERM.
This has led to the development of several families of SGD-type methods for the ERM problem, where the stochasiticity in these methods refers to noisy gradients of the empirical risk $\hat{f}(\bw)$, as opposed to noisy gradients of the true risk $f(\bw)$ within the (single-pass) SA framework. Specifically, the prototypical SGD algorithm for the ERM problem samples with replacement a single data sample $\bz_k$ (or a small mini-batch of samples) from $\cZ_T$ in each iteration $k$, computes the gradient $\nabla \ell(\bw_k, \bz_k)$, and takes a step in the negative of the computed gradient's direction. The iterates $\bw_k$ of this particular SGD variant are known to converge reasonably fast to the ERM solution $\bw^*_{\textsc{erm}}$ under various assumptions on the geometry of the loss function $\ell(\bw,\bz)$~\cite{bottou2010large,hardt2015train}.
%
% *****WUB: Replaced by the previous set of sentences*****
% In this type of SGD algorithm, at each iteration a single data point $\bz_t$ is sampled with replacement, the gradient $\nabla \ell(\bw_t, \bz_t)$ is computed, and a gradient descent step is taken in the direction of this gradient. This type of SGD has been shown to have reasonably fast convergence on the ERM solution, \cite{bottou2010large}, and SGD has further been shown to have good generalization properties \cite{hardt2015train}.

A variety of adaptive and more elaborate SGD-style algorithms, such as Adagrad, RMSProp, and Adam~\cite{duchi2011adaptive,kingma2014adam}, which introduce adaptive stepsizes, momentum terms, and Nesterov-style acceleration, have been developed in recent years. Empirically, these methods provide faster convergence to at least a stationary point of $\hat{f}(\bw)$, especially when training deep neural networks. (Note that some of these methods have provable convergence issues, even for convex problems~\cite{reddi2019convergence}.) A family of so-called {\em variance-reduction} methods \cite{johnson2013accelerating,allen2016variance,lei2017non,allen2017natasha,allen2018natasha}, such as \emph{stochastic variance reduced gradient} (SVRG), \emph{stochastically controlled
stochastic gradient} (SCSG), and \textsc{Natasha}, have also been developed in the literature for the ERM problem. In these methods, iterates from previous epochs are averaged to produce a low-complexity estimate of the gradient with provably small variance, which speeds up convergence. In terms of theoretical analysis, SGD-style and variance-reduction algorithms are studied in both convex and nonconvex settings. Unlike the SA framework, however, the convergence analysis of these methods for the ERM setting is in terms of the computational effort, measured in terms of the number of gradient evaluations, needed to approach a global optimum or a stationary point of the empirical risk $\hat{f}(\bw)$.

\revise{Since optimization methods for the ERM framework primarily provide bounds on either $[\hat{f}(\bw_k) - \hat{f}(\bw^*_{\textsc{erm}})]$ or $\|\bw_k - \bw^*_{\textsc{erm}}\|_2$, a bound on the excess risk $[f(\bw_k) - f(\bw^*)]$ under the ERM setting necessitates additional analytical steps that typically involve bounding the \emph{generalization error}, defined as $[f(\bw^*_{\textsc{erm}}) - \hat{f}(\bw^*_{\textsc{erm}})]$, of the ERM solution.
%
% *****WUB: Replaced by the previous set of sentences*****
% ERM methods are used to approximate the statistical learning problem, and a natural question is the quality of this approximation. This question is usually analyzed via {\em generalization error}. Let $\bw^*_{\textsc{erm}} := \arg\min_{\bw} \hat{f}(\bw)$ be the ERM solution; then $f(\bw^*_{\textsc{erm}}) - \hat{f}(\bw^*_{\textsc{erm}})$ is the generalization error.
Classic generalization error bounds have been provided in terms of the {\em Vapnik--Chervonenkis dimension} or {\em Rademacher complexity} of the class of functions induced by $\mathcal{W}$~\cite{vapnik1999statistical,bartlett2002rademacher}, or in terms of the uniform or so-called ``leave-one-out'' stability~\cite{kearns1999algorithmic,bousquet2002stability,mukherjee2002statistical,shalev2010learnability} of the solution. Together, the optimization-theoretic bounds and the \emph{learning-theoretic} bounds on quantities such as the generalization error result in excess risk bounds that decay at rates $O(T^{-1/2})$ or $O(T^{-1})$ for various loss functions as long as the number of optimization iterations $k$ is on the order of the number of training samples $T$. %Typical results have the generalization error decaying at rates $O(T^{-1/2})$ or $O(T^{-1})$. We emphasize that these rates match order-wise the convergence rates of SO methods if one supposes that the number of iterations is on the order of the number of training samples.
Thus, the ERM framework can yield excess risk bounds that match the sample complexity of the ones under the SA framework. Nonetheless, we focus primarily on the SA setting in this paper for two reasons. First, we are concerned with the statistical optimization problem~\eqref{eqn:StatisticalLearningCentralized}, and the SA framework measures performance directly with respect to this problem, whereas the ERM/finite-sum setting yields the final results only after a combination of optimization-theoretic and learning-theoretic bounds. %In this paper, we focus primarily on the SO setting, for two reasons. First, we are concerned with the statistical learning problem, and SO measures performance directly with respect to the statistical risk, whereas the ERM/finite-sum setting considers only an approximation.
Second, the SA framework is naturally well-suited to the setting of streaming data, whereas ERM supposes access to the entire dataset.}

\subsection{Distributed Optimization and Machine Learning}\label{sect:distributed.learning}
Distributed optimization is an extremely broad field, with a rich history. In this paper, we only discuss the portion of the literature most relevant to our problem setting. Specifically, we focus on methods for distributing SGD-style algorithms over collections of computing devices and/or processors that communicate over networks defined by graphs and aggregate data by \emph{averaging} information over the network. We further divide these methods into two categories, based on the nature of distributed averaging that is employed within each algorithm: %
%
%Distributed learning via averaging is a popular strategy with a vast literature, which we divide between two main approaches:
{\em exact averaging}, in which processing nodes use a robust \emph{message passing interface} (MPI) communications primitive such as {\tt AllReduce}~\cite{GabrielFaggEtAl.ConfOpenMPI04} to compute exact averages of gradients and/or iterates in the network, and {\em inexact averaging}, in which an approximate approach such as distributed consensus/diffusion~\cite{Tsitsiklis.Athans.ITAC1984,xiao2004fast,ChenSayed.ITSP12} is used to \emph{approximate} averages of gradients and/or iterates in the network. The former category of algorithms requires careful network configuration in order to coordinate {\tt AllReduce}-style averaging, whereas the latter category requires minimal explicit configuration, but the algorithms can suffer from slower convergence due to approximation error in the averaging step.

\subsubsection{Exact Averaging and Distributed Machine Learning}
In the case of algorithms utilizing exact averaging, processing nodes employ an MPI library to compute exact averages in a robust manner. While implementations differ, a generic approach is to compute averages over a spanning tree in the network. Reusing the notation introduced in Section~\ref{subsec:HighrateStreaming}, let $\{\bv_n\in\R^d\}_{n \in \cV}$ be the set of vectors distributed across the network at the start of the averaging subroutine and let $\bar{\bv} := \tfrac{1}{N}\sum_{n\in\cV} \bv_n$ denote their average. %Letting $G = (\cV,\cE)$ be the graph that describes the network, let $\bv_n \in \mathbb{R}^d$ be the $d$-dimensional signals at each processing node, and let $\bv_{av} = 1/N \sum_{n=1}^N \bv_n$ denote their average. Then, averages are constructed in a two-pass manner.
Then, the average $\bar{\bv}$ can be obtained at each node in a two-pass manner. In the first pass, each leaf node $n$ in the spanning tree passes its vector $\bv_n$ to its parent node, which averages together the vectors of its child nodes and passes the average to {\em its} parent node; this process continues recursively until the root node has the average $\bar{\bv}$. In the second pass, the root node disseminates $\bar{\bv}$ to the network by passing it to its child nodes; this continues recursively until all of the leaf nodes posses $\bar{\bv}$. This type of averaging is provably efficient, requiring only $R = O(N)$ exchange of messages within the network.

This generic approach to computing exact averages has been applied to distributed machine learning via a variety of implementations, especially under the distributed computing framework. {\tt TensorFlow} has a package for parameter-server distributed learning on multiple GPUs that uses exact averaging; worker nodes compute gradients, which are forwarded to the parameter server for exact averaging~\cite{AbadiBarhamEtAl.ConfOSDI16}. By contrast, {\tt Horovod}~\cite{sergeev2018horovod} is a distributed-parameter library for deep learning that averages gradients using {\em ring} {\tt AllReduce}; the GPU nodes are connected into a ring topology, which makes for simple and efficient exact averaging.

\subsubsection{Inexact Averaging and Distributed Machine Learning}\label{sssec:inexact.ave}
In the case of algorithms utilizing inexact averaging, processing nodes use local communications, without network-wide coordination, to compute approximate averages of their data. A widespread method for this is {\em averaging consensus}, a mainstay of distributed control, signal processing, and learning \cite{dimakis2010gossip,olshevsky2011convergence}. Again suppose $\{\bv_n\in\R^d\}_{n \in \cV}$ is the set of vectors distributed across the network at the start of the averaging subroutine and $\bar{\bv}$ denotes the exact average of these vectors. %Again suppose that the network graph is $G = (\cV,\cE)$ with individual signals $\bv_n \in \mathbb{R}^d$ with the average $\bv_{av} = 1/N \sum_{n=1}^N \bv_n$.
Next, define a {\em doubly stochastic matrix} $\bA \in \mathbb{R}^{N \times N}$ that is {\em consistent} with the topology of the network $G = (\cV,\cE)$. That is, $\bA$ is a matrix whose entries are non-negative, whose rows and columns sum to one, whose diagonal entries $a_{n,n}$ are non-zero, and whose $(n,m)$-th entry $a_{n,m} \neq 0$ only when $(n,m) \in \cE$. Averaging consensus then proceeds in multiple rounds of the following iteration using local communications:
\begin{equation}
    \bv_n^{r+1} = \sum_{m=1}^N a_{n,m} \bv_m^r.
\end{equation}
Here, $r \in \Z_+$ denotes the iteration index for averaging consensus, $\bv_n^r$ denotes an approximation of $\bar{\bv}$ at node $n$ after $r$ iterations, and $\bv_n^0 := \bv_n$. In words, each processing node takes a convex combination of the estimate of $\bar{\bv}$ at its neighboring nodes. Under mild conditions, averaging consensus converges geometrically on $\bar{\bv}$, with the approximation error scaling as $\| \bv_n^r - \bar{\bv} \|_2 = O\left(|\lambda_2(\bA)|^r\right)$.

{\em Distributed gradient descent} (DGD) is a classic approach to distributed optimization via inexact averaging~\cite{nedic2009distributed}. It uses only a single round of averaging consensus per iteration, i.e., $R = 1$ using the notation of Section~\ref{subsec:HighrateStreaming}, and it is posed in the setting of finite-sum optimization: each node $n$ has a local cost function $\hat{f}_n(\bw)$, and the objective is to minimize the sum $\hat{f}(\bw) := \sum_{n=1}^N \hat{f}_n(\bw)$. While DGD was originally posed in the framework of distributed control, it applies equally well to the distributed ERM setting in which $\hat{f}_n(\bw)$ corresponds to the empirical risk over the training data at node $n$. In terms of specifics, the original DGD formulation supposes a synchronous communications model in which each node $n$ computes a weighted average of its neighbors' iterates at each iteration $t$, after which it takes a gradient step with respect to its local cost function:
\begin{equation}\label{eqn:dgd}
    \bw_{n,t+1} = \bw_{n,t} + \sum_{m=1}^N a_{n,m} \bw_{m,t} - \eta_{t}\nabla \hat{f}_n(\bw_{n,t}) .
\end{equation}
Thus, each node takes a standard gradient descent step preceded by one-round averaging consensus on the iterates.

Several extensions to DGD have been proposed in the literature, including extensions to time-varying and directed graphs~\cite{nedic2014distributed,nedic2016stochastic,saadatniaki2018optimization} and variations with stronger convergence guarantees~\cite{shi2015extra,xi2017dextra,XinKhan.ICSL18}. Other related works have studied distributed (stochastic) optimization via means other than gradient descent, including distributed dual averaging~\cite{duchi2011dual,tsianos2012push} and the \emph{alternating direction method of multipliers} (ADMM)~\cite{wei20131,shi2014linear,mokhtari2016dqm}. The convergence of DGD-style methods has been studied under a variety of settings; two relevant results are that stochastic DGD-style algorithms have error decaying as $O(\log(t)/\sqrt{t})$ for general smooth convex functions and $O(\log(t)/t)$ for smooth strongly convex functions, even if the network is time varying~\cite{nedic2014distributed,nedic2016stochastic}.

We conclude by noting that inexact averaging-based distributed algorithms have also been analyzed/proposed for nonconvex optimization problems. In particular, DGD-style methods for nonconvex finite-sum problems are presented in~\cite{tatarenko2017non,wai2017decentralized,zeng2018nonconvex}, and convergence rates to stationary points and, when possible, local minima are derived. Further particularization of these works to problems with ``nicer'' geometry of saddle points and to structured nonconvex problems such as PCA can be found in works such as~\cite{li2011distributed,bertrand2014distributed,schizas2015distributed,fan2019distributed}.

\subsection{Roadmap for the Remainder of the Paper}
Putting the results presented in this paper in the context of the preceding discussion, the rest of the paper describes recent results in distributed machine learning from fast streaming data over networks that aggregate the distributed information using both exact and inexact averaging. Specifically, we synthesize results from four recent papers~\cite{dekel2012optimal,raja2020distributed,tsianos2016efficient,NoklebyBajwa.ITSIPN19} that focus on the distributed SA setting of Section~\ref{sec:ProblemFormulation}. Among these works, nodes in~\cite{dekel2012optimal,raja2020distributed} exchange messages using a robust MPI primitive such as {\tt AllReduce}, allowing exact averaging of messages for processing. The main distinction between these two works is that~\cite{dekel2012optimal} focuses on distributed convex SA problems, whereas \cite{raja2020distributed} studies the distributed PCA problem under the SA setting. %In \cite{dekel2012optimal}, mini-batched distributed SGD-style methods are proposed for solving convex SA problems, whereas in \cite{raja2020distributed} distributed {\em nonconvex} SA methods are proposed for solving PCA problems.
In contrast, nodes in~\cite{tsianos2016efficient,NoklebyBajwa.ITSIPN19} exchange messages using multiple rounds of averaging consensus and thus, similar to DGD, aggregate information using inexact averaging of messages. Both these works study distributed convex SA problems, with~\cite{tsianos2016efficient} focusing on dual averaging and \cite{NoklebyBajwa.ITSIPN19} investigating gradient descent as solution strategies.

\section{Distributed Stochastic \revise{Approximation} Using Exact Averaging}\label{sec:DM_Krasulina}
We detail two machine learning algorithms in this section for the distributed mini-batch framework of Section~\ref{sec:ProblemFormulation}, with one algorithm for general convex loss functions and the other one for the nonconvex loss function corresponding to the $1$-PCA problem. Both these algorithms operate under the assumption of nodes aggregating distributed information via exact averaging using {\tt AllReduce}-style communications primitives. The main focus in both these algorithms is to strike a balance between streaming, computing, and communications rates, while ensuring that the error in the final estimates is near optimal in terms of the number of samples arriving at the distributed system.
%
% *****WUB: Replaced by the previous paragraph*****
% In this section, we detail two distributed stochastic methods; one for solving general convex optimization problem and other for computing a top eigenvector of a covariance matrix ($1$-PCA), which is a nonconvex optimization problem with a geometric landscape that is favorable for computing global solution. As mentioned above, here we suppose that nodes aggregate distributed data via {\em exact} averaging via a fast and reliable network that permits {\tt MapReduce} and/or {\tt AllReduce} primitives. The methods presented here are designed to reduce the amount of communications overhead while taking full advantage of all the samples in the data streams.

\revise{Both the algorithms take advantage of the fact that (implicit or explicit) mini-batching reduces (gradient / sample covariance) noise variance. Between any two data-splitting instances, nodes in each algorithm compute average gradients/iterates over the newest (network-wide) $B$ data samples and use these \emph{exactly} averaged quantities for a stochastic update. Given ample compute resources and keeping everything else fixed, an increase in network-wide mini-batch size $B$ under such a strategy decreases both the noise variance and demands on the communications resources. In doing so, however, one also reduces the number of algorithmic iterations that take place within the network per second, which has the potential to slow down the convergence rates of the algorithms to the optimal solutions. An important question then is whether (and when) it is possible to utilize network-wide mini-batch averaging to simultaneously balance the compute-limited and communications-limited regimes in high-rate streaming settings (i.e., ensure $\frac{R_s}{R_e} \not\gg B$), reduce the noise variance, and guarantee that (order-wise) the convergence rate is not adversely impacted. We address this question in the following for the case of exact averaging.}
%
% *****WUB: Replaced the following two paragraphs by the previous paragraph*****
% To do so, these methods perform mini-batching, in which nodes compute average gradients over $B$ data samples and use the averaged gradient for a stochastic update. This reduces the communications overhead, in terms of the number of gradients that the network must communicate, by a factor of $B$. However, mini-batching also slows down the rate of optimization iterations, which may slow down convergence to the solution of the learning problem. An important question is whether (and when) it is possible to use mini-batching to reduce the communications overhead enough to process all of the data samples, but not so much that convergence speed is significantly slowed down.
%
% Hence, a main effort of this section is the derivation of conditions on $B$ that ensure order-optimal excess risk, in terms of the rates at which nodes communicate and data samples arrive to the network. In the following we present algorithms and theoretical results first for solving general convex learning problems using distributed mini-batched SGD, and then we present results for solving $1$-PCA, which as mentioned above is nonconvex but ``well-behaved.'' Finally, experiments over real-world data are provided to corroborate the theoretical findings.

\subsection{Distributed Mini-batched Stochastic Convex Approximation}\label{subsec:ConvexMinibatch}
Due to the high impact of mini-batching on the performance of distributed stochastic optimization, distributed methods deploying mini-batching and utilizing exact averaging have been studied extensively in the past few years; see, e.g.,~\cite{dekel2012optimal,byrd2012sample,li2014efficient,shamir2014distributed}. Among these works, the results in~\cite{dekel2012optimal} provide an upper bound on the network-wide mini-batch size $B$ that ensures sample-wise order-optimal convergence in SA settings. In contrast, \cite{li2014efficient,byrd2012sample,shamir2014distributed} focus on the selection of mini-batch size under ERM settings. Since the SA setting is best suited for the streaming framework of this paper, our discussion here focuses exclusively on the \emph{distributed mini-batch} (DMB) algorithm proposed in~\cite{dekel2012optimal} for stochastic convex approximation. The DMB algorithm is listed as Algorithm~\ref{algo:DMB} in the following and discussed further below. %As already explained SO methods are better suited for streaming settings, therefore, in the following we will explain in detail the DMB algorithm proposed in \cite{dekel2012optimal} (given here as Algorithm~\ref{algo:DMB}).

We begin with the data-splitting model of Section~\ref{sec:ProblemFormulation} and initially assume sufficient provisioning of resources so that $R_s \leq BR_e$. The DMB algorithm at iteration $t$ in this setting has a mini-batch $\{\bz_{t'}, t'=(t-1)B+1,\dots,tB\}$ of $B$ data samples at the splitter, which is then distributed as $N$ smaller mini-batches of size $B/N$ each across the network of $N$ compute nodes. Afterwards, the nodes in the network locally (and in parallel) compute an average gradient $\bg_{n,t}$ of the loss function over their local mini-batch of $B/N$ data samples (see Steps~\ref{alg:DMB.label.3}--\ref{alg:DMB.label.4} in Algorithm~\ref{algo:DMB}). Next, nodes engage in distributed exact averaging of their local mini-batched gradients using an {\tt AllReduce}-style communications primitive to obtain the network-wide mini-batched average gradient $\bg_t$ (cf.~Step~\ref{alg:DMB.label.2}, Algorithm~\ref{algo:DMB}), which is then used to update the network-wide estimate $\bw_t$ of the machine learning model (cf.~Step~\ref{alg:DMB.label.5}, Algorithm~\ref{algo:DMB}).

\begin{algorithm}[t]
    \textbf{Require:} Provisioning of compute and communications resources to ensure fast effective processing rate, i.e., either $R_s \leq B R_e$ or $R_s = \left(B + \mu\right)R_e$, as well as guaranteed exact averaging in $R$ rounds of communications\\
	\textbf{Input:} Data stream $\{\bz_{t'} \stackrel{\text{i.i.d.}}{\sim} \cD\}_{t' \in \Z_+}$ that is split into $N$ streams of mini-batched data $\{\bz_{n,b,t}\}_{b=1,t\in\Z_+}^{B/N}$ across the network of $N$ nodes (after possible discarding of $\mu$ samples per split) and stepsize sequence $\{\eta_t \in \R_+\}_{t \in \Z_+}$\\
    %Incoming data streams at $N$ processors, expressed as $\left\{\bz_{n,t} \stackrel{\text{i.i.d.}}{\sim} \cD\right\}_{n=1,t \in \Z_+}^N$, and a stepsize sequence $\left\{\gamma_t \in \R_+\right\}_{t \in \Z_+}$\\
	\textbf{Initialize:} All compute nodes initialize with $\bw_0 = \bzero \in \R^d$
	%\algsetup{indent=1em}
	\begin{algorithmic}[1]
	\For{$t=1,2,\dots$,}
		\State $\forall n \in \{1, \dots, N\}, \ \bg_{n,t} \leftarrow \bzero \in \R^d$ \label{alg:DMB.label.1}
		    \For{$b=1,\dots,B/N$} \Comment{Node $n$ receives the mini-batch $\{\bz_{n,b,t}\}_{b=1}^{B/N}$ and updates $\bg_{n,t}$ locally} \label{alg:DMB.label.3}
		        \State $\forall n \in \{1, \dots, N\}, \ \bg_{n,b,t} \leftarrow \nabla \ell(\bw_t, \bz_{n,b,t})$
		        \State $\forall n \in \{1, \dots, N\}, \ \bg_{n,t} \leftarrow \bg_{n,t} + \frac{1}{B/N}\bg_{n,b,t}$
		    \EndFor \label{alg:DMB.label.4}
		\State Compute $\bg_t \leftarrow \frac{1}{N}\sum_{n=1}^{N}\bg_{n,t}$ in the network using exact averaging \label{alg:DMB.label.2}
		\State Set $\bw_{t+1}\leftarrow\left[ \bw_{t} - \eta_t \bg_t \right]_\mathcal{W}$ across the network \label{alg:DMB.label.5}
		    \If{$R_s = \left(B + \mu\right)R_e$} \Comment{Slight under-provisioning of compute/communications resources}
            %\If{$N < \tfrac{R_s}{R_p} + \tfrac{N R_s}{B R_c}$}
		        \State The system receives $(B + \mu)$ additional data samples during execution of Steps~\ref{alg:DMB.label.1}--\ref{alg:DMB.label.5}, out of which $\mu \in \Z_+$ samples are discarded at the splitter
            \EndIf
	\EndFor
	\end{algorithmic}
	{\bf Return:} An estimate $\bw_t$ of the Bayes optimal solution after receiving $t'=(B+\mu)t$ samples
	\caption{The Distributed Mini-batch (DMB) Algorithm~\cite{dekel2012optimal}}
	\label{algo:DMB}
\end{algorithm}

The DMB algorithm can also deal with \emph{reasonable} under-provisioning of resources without sacrificing too much in terms of the quality of the estimate $\bw_t$. Recall that the distributed processing framework cannot process all incoming samples when $R_s > BR_e$. However, as long as $R_s \not\gg BR_e$, the DMB algorithm simply resorts to dropping $\mu~(\in \Z_+) := (\tfrac{R_s}{R_e} - B)$ samples per splitting instance at the splitter in this resource-constrained setting and then proceeds with Steps~\ref{alg:DMB.label.1}--\ref{alg:DMB.label.5} using the remaining $B$ samples as before.
%
% *****WUB: Replaced the following paragraph by the above two paragraphs*****
% In the splitter model defined in Section~\ref{sec:ProblemFormulation}, at iteration $t$, DMB algorithm receives a minibatch $\{\bz_{t'},\;\forall t'=(t-1)B+1,\dots,tB\}$, of $B$ data samples, from the splitter. This mini-batch is then distributed across $N$ compute nodes in the network and gradient of loss function is computed for each of these data samples in parallel across the $N$ nodes as elucidated in Steps~4--7 of Algorithm~\ref{algo:DMB}. Next, once the computation is complete at all the nodes then they engage in distributed averaging of the local gradients in Step~8 of Algorithm~\ref{algo:DMB}, which can be implemented efficiently using routines like {\tt MapReduce} or {\tt AllReduce}. Finally, the iterate is updated in Step~12 of DMB algorithm.  Recall that, as explained in Section~\ref{subsec:HighrateStreaming} the time to complete this whole operation (Steps~4--8) depends on the effective processing rate, $R_e$, of the network. In resource constrained settings, i.e., $R_s/B>R_e$, we cannot process all the data samples received during iteration $t$ of DMB which results in discarding $\mu$ samples per iteration. Specifically, as given in Steps~9--11, for fixed values of $R_s$, $R_p$, and $R_c$ if number of nodes, $N$, in the network is less than $\frac{R_s}{R_p}+\frac{N R_s}{B R_c}$ then by the time Steps~4--8 are completed we will have received $(B+\mu)$ new samples at the splitter and out of these we can only process $B$ in the next iteration of DMB algorithm.

The main analytical contribution of~\cite{dekel2012optimal} was providing upper bounds on the mini-batch size $B$ and, when necessary, the number of discarded samples $\mu$ that ensure sample-wise order-optimal convergence for the DMB algorithm. We summarize these results of~\cite{dekel2012optimal} in the following theorem.

% \begin{theorem}\label{thm:DMB}
%     Let the loss function $\ell(\bw, \bz)$ be convex and smooth with $L$-Lipschitz gradient and gradient noise variance $\sigma^2$. Then, assuming bounded model space $\cW$, there exist stepsizes $\eta_t$ such that the approximation error of Algorithm~\ref{algo:DMB} after $t$ iterations is bounded as follows:

%     \todo{Then, assuming bounded model space $\cW$ and choosing stepsizes as $\eta_t=frac{1}{L+(\sigma/D_{\cW})\sqrt{t}}$ the approximation error of Algorithm~\ref{algo:DMB} after $t$ iterations is bounded as follows:}
%     \begin{align}\label{eqn:thm.DMB.1}
%         \E\left\{f(\bw_t)-f(\bw^*)\right\} \leq (B+\mu)\left(\frac{2D_\mathcal{W}^2 L} {t'}+\frac{2D_\mathcal{W}\sigma}{\sqrt{t'}}\right),
%     \end{align}
%     where $t'=(B+\mu)t$ is the total number of samples that arrived at the system by iteration $t$. Furthermore, setting the mini-batch size $B=({t'})^{1/3}$ gives the bound
%     \begin{align}
%          \E\left\{f(\bw_t)-f(\bw^*)\right\} \leq \frac{2D_\mathcal{W}\sigma}{\sqrt{t'}}+\frac{2D_\mathcal{W}}{({t'})^{2/3}}(LD_\mathcal{W}+\sigma\sqrt{\mu})+\frac{2D_\mathcal{W}\sigma}{({t'})^{5/6}}+\frac{2D_\mathcal{W}\sigma}{({t'})^{7/6}}+\frac{2\mu D_\mathcal{W}^2 L}{t'}.
%     \end{align}
%     In fact, if $B=(t'){^\rho}$ for any $\rho\in(0, 1/2)$ and $\mu=o(B)$, then the approximation error is bounded as
%     \begin{align}\label{eqn:thm.DMB.2}
%          \E\left\{f(\bw_t)-f(\bw^*)\right\} \leq \frac{2D_\mathcal{W}\sigma}{\sqrt{t'}}+o\Big(\frac{1}{\sqrt{t'}}\Big).
%     \end{align}
% \end{theorem}

\begin{theorem}\label{thm:DMB}
    Let the loss function $\ell(\bw, \bz)$ be convex and smooth with $L$-Lipschitz gradients and gradient noise variance $\sigma^2$. Then, assuming bounded model space $\cW$ and choosing stepsizes as $\eta_t=\frac{1}{L+(\sigma/D_{\cW})\sqrt{t}}$, the approximation error of Algorithm~\ref{algo:DMB} after $t$ iterations is bounded as follows:
    \begin{align}\label{eqn:thm.DMB.1}
        \E\left\{f(\bw_t)\right\} - f(\bw^*) \leq (B+\mu)\left(\frac{2D_\mathcal{W}^2 L} {t'}+\frac{2D_\mathcal{W}\sigma}{\sqrt{t'}}\right).
    \end{align}
    % where $t'=(B+\mu)t$ is the total number of samples that arrived at the system by iteration $t$. Furthermore, setting the mini-batch size $B=({t'})^{1/3}$ gives the bound
    % \begin{align}
    %      \E\left\{f(\bw_t)-f(\bw^*)\right\} \leq \frac{2D_\mathcal{W}\sigma}{\sqrt{t'}}+\frac{2D_\mathcal{W}}{({t'})^{2/3}}(LD_\mathcal{W}+\sigma\sqrt{\mu})+\frac{2D_\mathcal{W}\sigma}{({t'})^{5/6}}+\frac{2D_\mathcal{W}\sigma}{({t'})^{7/6}}+\frac{2\mu D_\mathcal{W}^2 L}{t'}.
    % \end{align}
    Furthermore, if $B=(t'){^\rho}$ for any $\rho\in(0, 1/2)$ and $\mu=o(B)$, then the approximation error is bounded as
    \begin{align}\label{eqn:thm.DMB.2}
         \E\left\{f(\bw_t)\right\} - f(\bw^*) \leq \frac{2D_\mathcal{W}\sigma}{\sqrt{t'}}+o\Big(\frac{1}{\sqrt{t'}}\Big).
    \end{align}
\end{theorem}

It can be seen from Theorem~\ref{thm:DMB} that the DMB algorithm results in near-optimal convergence rate of $O(1/\sqrt{t'})$, which corresponds to speed-up by a factor of $O(B)$, in two cases. First, when $R_s \leq BR_e$ and thus $\mu \equiv 0$, it can be seen from \eqref{eqn:thm.DMB.1} that this speed-up is obtained as long as $B=O(\sqrt{t'})$. Second, even when $R_s > BR_e$ and therefore $\mu = (\tfrac{R_s}{R_e} - B) > 0$, \eqref{eqn:thm.DMB.1} guarantees the convergence speed-up as long as $B=o(\sqrt{t'})$ and $\mu = o(B)$. Stated differently, the speed-up can be obtained provided the streaming rate ($R_s$) does not exceed the effective processing rate per sample ($BR_e$) by too much, i.e., $\frac{R_s}{BR_e} = O(1)$.
%
% *****WUB: Replaced the following paragraph by the previous paragraph*****
% From this result we can conclude that in order to have a near-optimal solution after observing $t'$ samples we need to choose mini-batch size as $B=O(\sqrt{t'})$. Furthermore, this result is also applicable to practical situations when computational and communication resources cannot process all the samples i.e., $R_s/B>R_e$, which results in discarding $\mu>0$ samples per iteration. Result in Theorem~\ref{thm:DMB} shows that as long as the mismatch between data arrival rate and effective processing rate $R_s/B-R_e$ is small enough such that $\mu=o(B)$ we will achieve order-optimal convergence rate.

\subsection{\revise{Numerical Experiments for the DMB Algorithm}}\label{subsec:numerical_DMB}
\revise{We demonstrate the effectiveness of the scaling laws implied by Theorem~\ref{thm:DMB} by using the DMB algorithm to train a binary linear classifier (supervised learning problem) from streaming (labeled) data using \emph{logistic regression}~\cite{bishop.book06}. To this end, we take the labeled data as the tuple $\bz := (\bx,y)$ with $\bx \in \R^d$ and the labels $y \in \{-1,1\}$, define the regression model as $\bw := (\wtbw, w_0) \in \R^d \times \R$, and recall that the convex and smooth loss function for logistic regression can be expressed as $\ell(\bw,\bz) = \ln\left(1 + \exp(-y(\wtbw^\tT\bx + w_0))\right)$. Note that the optimal batch solution for logistic regression corresponds to the maximum likelihood estimate of the ground-truth regression coefficients that generate data $\bz = (\bx,y)$~\cite{bishop.book06}.}

\revise{The experimental results reported in this section correspond to $d=5$ and are averaged over 50 Monte Carlo trials. In order to generate data for each trial, we first generate ground-truth regression parameters via a random draw from the standard normal distribution, $\bw^* = (\wtbw^*, w_0^*) \sim \cN(\bzero, \bI)$. Next, we generate data samples as independent draws from another standard normal distribution, $\bx_{t'} \sim \cN(\bzero, \bI)$, and generate the corresponding labels $y_{t'}$ as independent draws from the Bernoulli distribution induced by the regression coefficients, i.e.,}
\begin{align}
    \revise{\Pr(y_{t'}=1 | \bx_{t'}) = 1/(1+ \exp{(-(\wtbw^*{}^{\tT}\bx_{t'} + w_0^*))}).}
\end{align}

\begin{figure}[t]
	\centering
	\subfigure[Impact of the mini-batch size on the convergence rate of the DMB algorithm for the resourceful regime. Note that the $B=1$ plot is effectively standard SGD.]{
		\includegraphics[width=0.45\columnwidth]{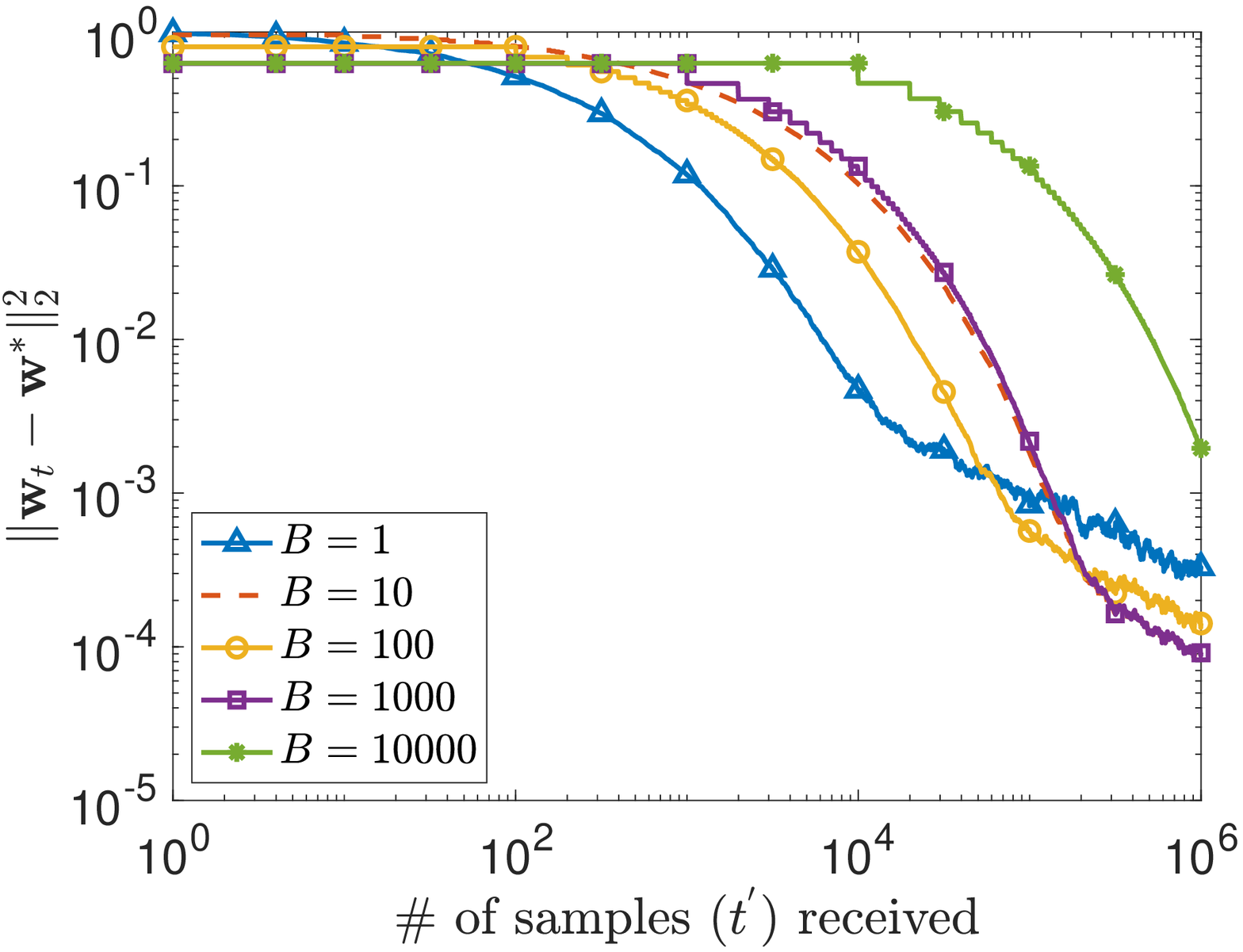}
		\label{fig:Synthetic_regression_NoLatency}}
	\qquad
	\subfigure[Performance of the DMB algorithm in a resource-constrained regime (i.e., $R_s > B R_e$), which causes loss of $\mu$ samples per iteration; here, $(N,B)=(10,500)$.]{
		\includegraphics[width=0.45\columnwidth]{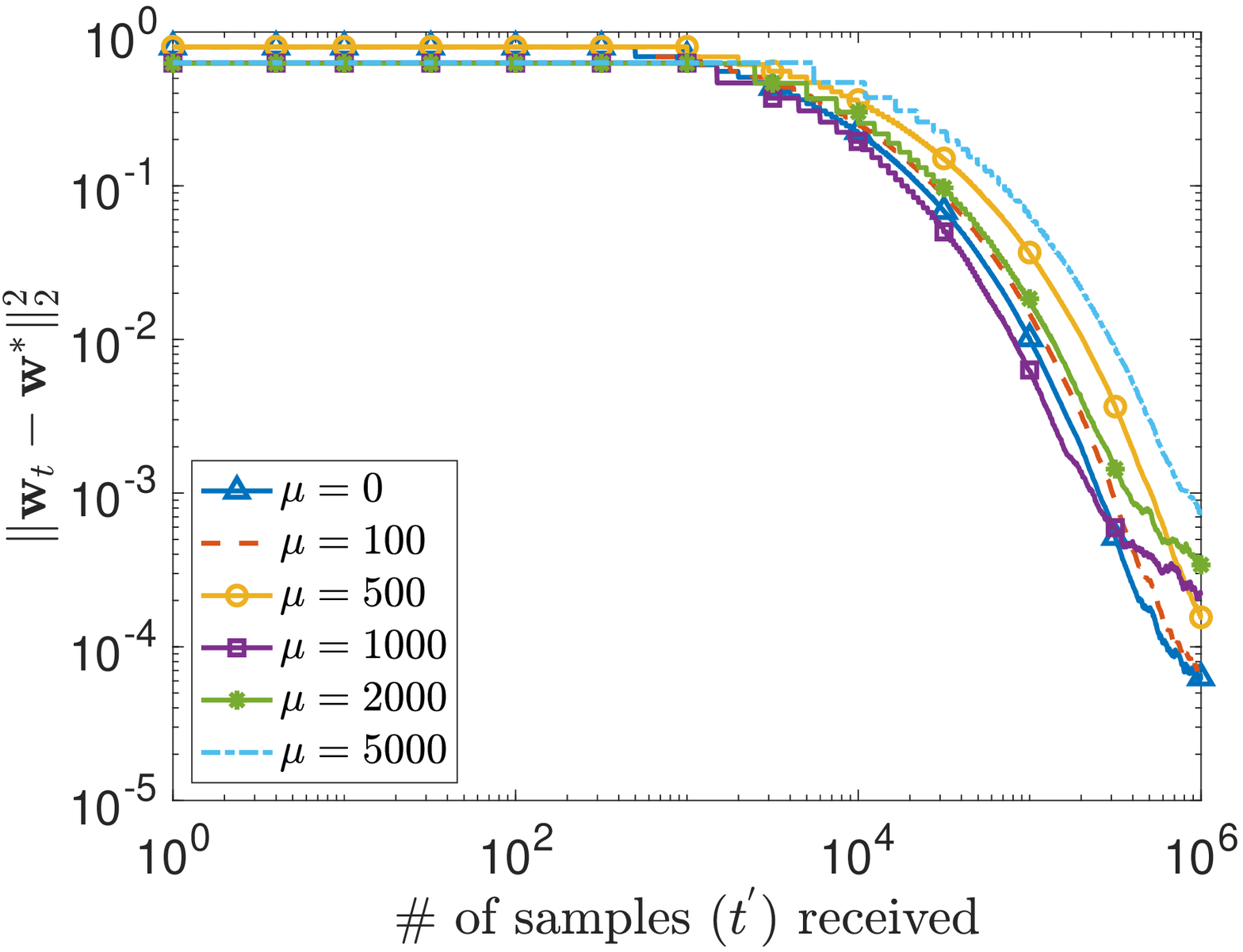}
		\label{fig:Synthetic_regression_Latency}}
	\caption{Convergence behavior of the DMB algorithm for the case of synthetic data under two scenarios: (a) No data loss ($\mu = 0$) and (b) loss of $\mu > 0$ samples per algorithmic iteration.}
\end{figure}

\revise{We report results of two experiments for the distributed, streaming framework of Section~\ref{sec:ProblemFormulation}. The first experiment deals with the resourceful regime, i.e., $R_s \leq BR_e$, and uses mini-batches of size $B \in \{1, 10, 100, 1000\}$. The results, shown in Fig.~\ref{fig:Synthetic_regression_NoLatency}, are obtained for stepsize of the form $c/\sqrt{t}$ (as prescribed by Theorem~\ref{thm:DMB}), where the corresponding value of $c$ chosen for different batch sizes is $c \in \{0.1, 0.1, 0.5, 1, 1\}$. In order to select these values of $c$, we ran the experiment for multiple choices of $c$ and picked the values that achieved the best results. Note that the results in Fig.~\ref{fig:Synthetic_regression_NoLatency} correspond to the optimality gap $\|\bw_t - \bw^{*}\|_2^2$ of the iterates from the ground truth; since the logistic loss is Lipschitz continuous, this trivially upper bounds the \emph{square} of the excess risk $f(\bw_t) - f(\bw^*)$. It can be seen from these results that, as predicted by Theorem~\ref{thm:DMB}, the estimation error $\|\bw_t - \bw^{*}\|_2^2$ after $t=t'/B$ iterations of the DMB algorithm is roughly on the order of $O(1/t')$ for $B \in \{1, 10, 100, 1000\}$, while it is worse by an order of magnitude for $B=10^4>\sqrt{t'}$.} %As predicted by Theorem~\ref{thm:DMB} we can see that after $t=t'/B$ iterations of DMB algorithm $\|\bw_t - \bw^{*}\|_2^2$ is on the order of $O(1/\sqrt{t'})$ for $B=\{1, 10, 100, 1000\}$, while for $B=10^4>\sqrt{t'}$ the error in parameter value is worse by an order of magnitude.

\revise{Next, we demonstrate the performance of the DMB algorithm for resource-constrained settings, i.e., $R_s > B R_e$, which causes the algorithm to discard $\mu = (R_s/R_e - B)$ samples per iteration. The experiment for this setting corresponds to a network of 10 nodes ($N=10$) with network-wide mini-batch of size $B=500$ (i.e., $B/N=50$). We consider different mismatch factors between streaming, processing, and communication rates in this experiment, which result in the number of samples being discarded as $\mu \in \{0, 100, 500, 1000, 2000, 5000\}$. The results are plotted in Fig.~\ref{fig:Synthetic_regression_Latency}, which shows that the error $\|\bw_t - \bw^{*}\|_2^2$ for $\mu=100$ is comparable to that for $\mu=0$ and progressively worsens as $\mu$ increases from $\mu=500$ to $\mu=5000$.}

\subsection{Distributed Mini-batched Streaming PCA}\label{subsec:NonconvexMinibatch}
Mini-batching and variance-reduction techniques have also been utilized for nonconvex stochastic optimization problems in both centralized and distributed settings~\cite{allen2016variance,zhang2016riemannian,reddi2016stochastic,allen2016improved,sato2017riemannian,XinKarEtAl.ISPM20,li2019communication,sun2019improving,assran2018stochastic}. But these and similar works typically either only guarantee convergence to first-order stationary points~\cite{allen2016variance,assran2018stochastic} and/or they are not applicable to the single-pass SA setting~\cite{allen2016improved,reddi2016stochastic,zhang2016riemannian,sato2017riemannian,li2019communication,sun2019improving}. In the following, we focus on the structured nonconvex SA problem of estimating the top eigenvector of a covariance matrix from fast streaming i.i.d.~data samples (i.e., the streaming $1$-PCA problem). Global convergence guarantees for this problem, as noted in Section~\ref{subsec:Optimization_ML}, have been derived in the literature for ``slow'' data streams. Our discussion here revolves around the \emph{distributed mini-batch Krasulina} (\DMK) algorithm that has been recently proposed and analyzed in~\cite{raja2020distributed} for the distributed mini-batch framework of Section~\ref{sec:ProblemFormulation} for fast streaming data. %In the following we describe distributed mini-batch Krasulina (\DMK) algorithm (Algorithm~\ref{algo:DMB_Krasulina}) proposed in \cite{raja2020distributed}, that is shown to converge to the global optimum for top eigenvector computation of a covariance matrix ($1$-PCA).

The \DMK~algorithm (see Algorithm~\ref{algo:DMB_Krasulina}) can be seen as a slight variation on the DMB algorithm for solving the $1$-PCA problem from fast streaming data. In particular, \DMK~is nearly identical to the DMB algorithm except for the fact that the generic gradients $\bg_{n,t}$ and $\bg_t$ in Algorithm~\ref{algo:DMB} are replaced by pseudo-gradient terms $\bxi_{n,t}$ and $\bxi_t$, respectively, in Algorithm~\ref{algo:DMB_Krasulina}. Therefore, the implementation details provided for the DMB algorithm in Section~\ref{subsec:ConvexMinibatch} also apply to \DMK. Nonetheless, the analytical tools utilized by \cite{raja2020distributed} to theoretically characterize the interplay between the solution accuracy of \DMK~and different system parameters differ greatly from those utilized for Theorem~\ref{thm:DMB}.
%
% *****WUB: Replaced the following paragraph by the previous paragraph*****
%\subsubsection{Data Model}\label{subsec:DataModel}
%As already explained in Section~\ref{subsec:Optimization_ML}, minimizing \eqref{eqn:PCA_Loss} from a distributed data stream is equivalent to finding the top principal component in an online manner. Here, we present \DMK~algorithm (Algorithm~\ref{algo:DMB_Krasulina}), which can be seen as an instance of DMB algorithm for solving specific problem of $1$-PCA. We can see \DMK~is very similar to the DMB algorithm except for Step~4, where update is defined explicitly instead of generic gradient term. Therefore, apart from Step~4, the description of DMB algorithm provided in Section~\ref{subsec:ConvexMinibatch} still holds.

\begin{algorithm}[t]
    \textbf{Require:} Same as the DMB algorithm in Algorithm~\ref{algo:DMB}\\
    \textbf{Input:} Same as the DMB algorithm in Algorithm~\ref{algo:DMB}\\
    %\textbf{Input:} Incoming data streams at $N$ processors, expressed as $\left\{\bz_{i,t} \stackrel{\text{i.i.d.}}{\sim} \cD\right\}_{i=1,t \in \Z_+}^N$, and a stepsize sequence $\left\{\gamma_t \in \R_+\right\}_{t \in \Z_+}$\\
    \textbf{Initialize:} All compute nodes initialize with the same $\bw_0\in\R^d$ randomly generated over the unit sphere
%	\algsetup{indent=1em}
	\begin{algorithmic}[1]
	\For{$t=1,2,\dots$,}
		\State $\forall n \in \{1, \dots, N\}, \ \bxi_{n,t} \leftarrow \bzero \in \R^d$ \label{alg:DMK.label.1}
		\For{$b=1,\dots,B/N$} \Comment{Node $n$ receives the mini-batch $\{\bz_{n,b,t}\}_{b=1}^{B/N}$ and updates $\bxi_{n,t}$ locally} \label{alg:DMK.label.3}
    		\State $\forall n \in \{1,\dots,N\}, \ \bxi_{n,t} \gets \bxi_{n,t} + \bz_{n,b,t}\bz_{n,b,t}^{\tT} \bw_{t-1}-\frac{\bw_{t-1}^{\tT}\bz_{n,b,t}\bz_{n,b,t}^{\tT}\bw_{t-1} \bw_{t-1}}{\|\bw_{t-1}\|_2^2}$
		\EndFor \label{alg:DMK.label.4}
		\State Compute $\bxi_t \leftarrow \frac{1}{N}\sum_{n=1}^{N}\bxi_{n,t}$ in the network using exact averaging \label{alg:DMK.label.2}
		\State Update the eigenvector estimate in the network as follows: $\bw_t \leftarrow \bw_{t-1}+\eta_t \bxi_t$ \label{alg:DMK.label.5}
        \If{$R_s = \left(B + \mu\right)R_e$} \Comment{Slight under-provisioning of compute/communications resources}
            \State The system receives $(B + \mu)$ additional data samples during execution of Steps~\ref{alg:DMK.label.1}--\ref{alg:DMK.label.5}, out of which $\mu \in \Z_+$ samples are discarded at the splitter
        \EndIf
		\EndFor
	\end{algorithmic}
	{\bf Return:} An estimate $\bw_t$ of the eigenvector $\bw^*$ associated with $\lambda_1(\bSigma)$ after receiving $t'=(B+\mu)t$ samples
	\caption{Distributed Mini-batch Krasulina (\DMK) Algorithm~\cite{raja2020distributed}}
	\label{algo:DMB_Krasulina}
\end{algorithm}

In order to discuss the convergence behavior of \DMK, we recall the assumptions stated in Section~\ref{subsec:Optimization_ML} for the $1$-PCA problem. Specifically, the i.i.d.~data samples $\bz_{t'}$ have zero mean and are bounded almost surely by some positive constant $\kappa$, i.e., $\E\{\bz_{t'}\}=\bzero$ and $\forall t', \|\bz_{t'}\|_2\leq \kappa$. Notice that Steps~\ref{alg:DMK.label.3}--\ref{alg:DMK.label.2} in Algorithm~\ref{algo:DMB_Krasulina} lead to an implicit computation of an unbiased estimate of the population covariance matrix $\bSigma = \E_{\bz \sim \cD}\{\bz\bz^\tT\}$ from the network-wide mini-batch of $B$ samples, which we denote by $\bA_{t}:=(1/B)\sum_{n=1}^N\sum_{b=1}^{B/N}\bz_{n,b,t} \bz_{n,b,t}^{\tT}$. The results for \DMK~depend on the variance of this unbiased sample covariance, which is defined as follows.
%
% *****WUB: Replaced the following paragraph by the previous paragraph*****
% Despite these striking similarities between DMB and \DMK, the later requires added assumptions on data to work. Specifically, the data distribution $\cD$ needs to be such that $\E\{\bz_{t'}\}=0$ and $\|\bz_{t'}\|_2\leq m$, where $m$ is some positive constant. Further, we associate with each data sample $\bz_{t'}$ a rank-one random matrix $\bA_{t'}:=\bz_{t'} \bz_{t'}^{\tT}$, which is a trivial unbiased estimate of the population covariance matrix $\bSigma$. Similary, for mini-batching in Steps~5--7 of \DMK~we get an unbiased estimator $\bA_{t}:=(1/B)\sum_{n}\sum_{b}\bz_{n,b,t} \bz_{n,b,t}^{\tT}$. We then define the variance of this unbiased estimate as follows.
\begin{definition}[Variance of sample covariance in \DMK]\label{def:SampleVarianceMinibatch}
The variance of the distributed sample covariance matrix $\bA_t$ in \DMK~is defined as follows: $$\sigma_B^2:=\mathbb{E}_{\cD}\left\{\left\|\frac{1}{B}\sum_{n=1}^{N}\sum_{b=1}^{B/N}\bz_{n,b,t}\bz_{n,b,t}^\tT - \bSigma\right\|_F^2\right\}.$$
\end{definition}
Note that $\sigma_B^2$ for $B=1$ corresponds to the single-sample covariance noise variance $\sigma^2$ defined in Section~\ref{sec:ProblemFormulation}. It is also straightforward to show that $\sigma_B^2 \leq \sigma^2/B$. And since all moments of the probability distribution $\cD$ exist by virtue of the norm boundedness of $\bz_{t'}$, the variance $\sigma_B^2$ as defined above exists and is finite. We now provide the main result for \DMK~from \cite{raja2020distributed} that expresses the convergence behavior of \DMK~in terms of the mini-batched noise variance $\sigma_B^2$.
%
% *****WUB: Replaced the following paragraph(s) by the previous sentence*****
% Next, we will define the notion of excess risk that will be used to provide results for \DMK. Letting $\bw^*$ denote the solution, the excess risk has a specific form, which is often termed the {\em potential function} in the literature. For estimate $\bw$ of the top eigenvector, the potential function is defined as
% \begin{align}\label{eqn:Error}
%     \Psi(\bw) := 1-\frac{(\bw^{\tT}\bw^*)^2}{\|\bw\|^2}.
% \end{align}
% It is straightforward to verify that $\Psi(\bw) = f(\bw) - f(\bw^*)$. Indeed, we can see that
% \begin{align}
%     f(\bw) - f(\bw^*)&=\frac{-\bw^T \bSigma \bw}{\|\bw\|_2^2}+\frac{\bw^{{*}^{\tT}}\bSigma \bw^*}{\|\bw^*\|_2^2}\\
%         &= \lambda_1 - \frac{\bw^T \bSigma \bw}{\|\bw\|_2^2}\\
%         &= \lambda_1 - \sum_{i=1}^d\frac{\lambda_i (\bw^T \bw^*_i )^2}{\|\bw\|_2^2}\\
%         &\leq \lambda_1 - \lambda_1\frac{(\bw^{\tT}\bw^*)^2}{\|\bw\|_2^2}\\
%         &\leq \lambda_1 \Psi(\bw).
% \end{align}
%
% \subsubsection{Main Result}\label{subsubsec:MainResultDM_Krasulina}
% The theoretical result for \DMK~is based on understanding the rate at which the \emph{potential function} $\Psi(\bw_t)$ of \DMK~converges to zero as a function of the number of algorithmic iterations $t$. Following is the main result from \cite{raja2020distributed} that provides an optimal choice of $B$ for solving $1$-PCA problem.
\begin{theorem}\label{thm:DMK}
    Let the i.i.d.~data samples be bounded, i.e., $\forall t', \|\bz_{t'}\|_2 \leq \kappa$, define $\gap := \lambda_1(\bSigma)-\lambda_2(\bSigma) > 0$, fix any $\delta \in (0,1)$, and pick $c := \frac{c_0}{2\gap}$ for any $c_0 > 2$. %Under the assumptions that $\|\bz_{t'}\|_2\leq m$ (for any $m>0$), $\E\{\bz_{t'}\}=0$, $\bSigma:=\E\{\bz_{t'}\bz_{t'}^{\tT}\}$, and $\gap:=\lambda_1(\bSigma)-\lambda_2(\bSigma)>0$, fix any $\delta \in (0,1)$ and pick $c:=\frac{c_0}{2\gap}$ for any $c_0 > 2$.
    Next, suppose $R_s \leq BR_e$ (i.e., no discarded data) and define
	\begin{align}\label{eqn:LowerboundL}
	   Q_{1}:= \frac{64 ed\kappa^4 \max(1, c^2)}{\delta^2}\ln\frac{4}{\delta},\quad Q_{2}:=\frac{512 e^2 d^2 \sigma_B^2\max(1, c^2)}{\delta^4}\ln\frac{4}{\delta},
	\end{align}
	pick any $Q\geq Q_1+Q_2$, and choose the stepsize sequence as $\eta_t := c/(Q+t)$. Then, we have for \DMK~that there exists a sequence $(\Omega_{t}^{'})_{t \in \Z_+}$ of nested subsets of the sample space $\Omega$ such that $\bbP\left(\cap_{t>0}\Omega_t^{'}\right)\geq 1-\delta$ and
	\begin{align}\label{eqn:FinalResult}
	\E_t\left\{f(\bw_t)\right\} - f(\bw^*) \leq C_1\Big(\frac{Q + 1}{t + Q + 1}\Big)^{\frac{c_0}{2}} + C_2\Big(\frac{\sigma_B^2}{t + Q + 1}\Big),
	\end{align}
    where $\E_t$ is the conditional expectation over $\Omega_t^{'}$, and $C_1$ and $C_2$ are constants defined as
	$$C_1 := \frac{\lambda_1(\bSigma)}{2}\Bigg(\frac{4ed}{\delta^2}\Bigg)^{\frac{5}{2\ln2}}e^{2c^2\lambda_1^2(\bSigma)/Q}\quad\textnormal{and}\quad C_2 := \frac{2 c^2 \lambda_1(\bSigma)e^{(c_0+2c^2\lambda_1^2(\bSigma))/Q}}{(c_0-2)}.$$
\end{theorem}
Theorem~\ref{thm:DMK} is similar in flavor to Theorem~\ref{thm:DMB} in the sense that their respective excess risk bounds in \eqref{eqn:thm.DMB.1} and \eqref{eqn:FinalResult} have (asymptotically) dominant error terms that involve the noise variance (gradient noise for the convex problem and covariance noise for the $1$-PCA problem). In particular, since $\sigma_B^2 \leq \sigma^2/B$, the excess risk $f(\bw_t) - f(\bw^*)$ in \DMK~can be driven down faster by increasing the mini-batch size $B$ up to a certain limit, as noted in the next result. But the two theorems also have some key differences, which can be attributed to \DMK's focus on global convergence for the nonconvex $1$-PCA problem. The first difference is that the excess risk in \eqref{eqn:FinalResult} is being bounded in expectation over a subset of the sample space, whereas the expectation in Theorem~\ref{thm:DMB} is over the whole sample space $\Omega$. The second difference is that the result in Theorem~\ref{thm:DMB} is independent of the ambient dimension $d$, whereas the result for the $1$-PCA problem has $d^4$ dependence.
%
% *****WUB: Commented out as redundant*****
% We would like to comment that this result in not tight in terms of dependence on dimension as Allen-Zhu and Li \cite{allen2016first} have proved a lower bound of $\log(d)$ which is achieved by \cite{allen2016first,jain2016streaming}. Issue here is that work in \cite{allen2016first} does not consider distributed settings and such extension does not seem trivial, while the weakness in Jain et al.'s result is that probability of success cannot be improved beyond $3/4$ for streaming settings. Another parameter of interest in PCA problem is the eigen gap $(\lambda_1(\bSigma) - \lambda_2(\bSigma))$, and Theorem~\ref{thm:DMK} achieves optimal dependence of $1/(\lambda_1(\bSigma)-\lambda_2(\bSigma))^2$.

We now provide a corollary of Theorem~\ref{thm:DMK} that highlights the speed-up gains associated with \DMK~as long as the mini-batch size $B$ does not exceed a certain limit.
\begin{corollary}\label{cor:BatchNoLatency}
Let the parameters and constants be as specified in Theorem~\ref{thm:DMK}. Next, pick parameters $(Q_1', Q_2')$ such that $Q_1 ' \geq Q_1$ and $Q_2' \geq Q_2/\sigma_B^2$, and denote the total number of samples processed by \DMK~as $t' := tB$. Then, as long as assumptions from Theorem~\ref{thm:DMK} hold and the network-wide mini-batch size satisfies $B \leq (t')^{1-\tfrac{2}{c_0}}$, there exists a sequence $(\Omega_{t}^{'})_{t \in \Z_+}$ of nested subsets of the sample space $\Omega$ such that $\bbP\left(\cap_{t>0}\Omega_t^{'}\right)\geq 1-\delta$ and
    \begin{align}
	   \E_{t}\left\{f(\bw_t)\right\} - f(\bw^*) \leq c_0 C_1 \frac{{Q_1'}^{c_0/2}}{t'} + c_0 C_1 \Bigg(\frac{\sigma^2 Q_2'}{t'}\Bigg)^{c_0/2} + \frac{C_2 \sigma^2}{t'}. %,
	\end{align}
	%where $\E_t$ is the conditional expectation over $\Omega_t^{'}$, and $C_1$ and $C_2$ are constants as defined in Theorem~\ref{thm:DMK}.
\end{corollary}
\begin{proof}
	Substituting $t = t'/B$ in \eqref{eqn:FinalResult} and using simple upper bounds yield
	\begin{align*}
	   \E_{t}\left\{f(\bw_t) - f(\bw^*)\right\} \leq C_1\Big(\frac{Q + 1}{Q + t}\Big)^{\frac{c_0}{2}} + C_2\Big(\frac{ \sigma_B^2}{t}\Big) \leq 2 C_1\Big(\frac{Q}{t}\Big)^{\frac{c_0}{2}} + C_2\Big(\frac{\sigma_B^2}{t}\Big).
	\end{align*}
	Next, substituting $Q=Q_1'+\sigma_B^2 Q_2'$ in this expression gives us
	\begin{align}\label{eqn:E_TB}
	   \E_{t}\left\{\Psi_{t}\right\}\leq c_0 C_1\Big(\frac{Q_1'}{t}\Big)^{\frac{c_0}{2}} + c_0 C_1\Big(\frac{\sigma_B^2 Q_2'}{t}\Big)^{\frac{c_0}{2}} + C_2\Big(\frac{\sigma_B^2}{t}\Big).
	\end{align}
    Since $\sigma_B^2 \leq \sigma^2/B$ and $t = t'/B$, \eqref{eqn:E_TB} reduces to the following expression:
    \begin{align*}
	   \E_{t}\left\{f(\bw_t) - f(\bw^*)\right\} \leq c_0 C_1 \Bigg(\frac{B Q_1'}{t'}\Bigg)^{c_0/2}+c_0 C_1 \Bigg(\frac{ \sigma^2 Q_2'}{t'}\Bigg)^{c_0/2}+\frac{C_2 \sigma^2}{t'}.
	\end{align*}
    The proof now follows from the assumption that $B \leq (t')^{1-\tfrac{2}{c_0}}$.
\end{proof}

In words, Corollary~\ref{cor:BatchNoLatency} states that \DMK~achieves the optimal excess risk of $O(1/t')$ for the $1$-PCA problem, which corresponds to a speed-up gain by a factor of $B$, as long as $B=O((t')^{1-\tfrac{2}{c_0}})$ and network resources are provisioned to ensure $R_s \leq B R_e$.

We can also leverage this result to demonstrate how the distributed mini-batch framework of this paper helps us tradeoff computation resources for communication resources. Suppose we are in a compute-rich distributed environment, in which exact averaging requires $R$ rounds of communications, and it is desired to achieve order-optimal risk of $O(1/t')$ for \DMK. This requires that $R_c$ be fast enough to ensure completion of the communications phase within the time between the end of the computation phase and the arrival of next mini-batch of data; using the definitions from Section~\ref{sec:ProblemFormulation}, this means:
\begin{align}\label{eqn:B_bound}
    \frac{R}{R_c} \leq \frac{B}{R_s} - \frac{B}{N R_p} \ \Longrightarrow \ R_c \geq \frac{NRR_sR_p}{B\left(N R_p - R_s\right)}.
\end{align}
We can see from this lower bound that increasing the mini-batch size $B$ up to a certain point, while keeping everything else fixed, relaxes the requirement on the communications rate within the network without affecting the quality of the final solution.
%
% *****WUB: Replaced by the previous paragraph*****
% This result tells us that if we choose $B=O((t')^{1-\tfrac{2}{c_0}})$, optimal excess risk is achieved using \DMK. Now we demonstrate how this result is valuable for network with slow communication links. Using definitions of $R_s$ and $R_p$ from Section~\ref{sec:ProblemFormulation}, in order to finish communication phase before arrival of next mini-batch we need $R_c$ to fast enough to fit all the communications within the amount of time left after computation
% \begin{align*}
%     \textnormal{Communication time}&=\frac{1}{R_c}=(\textnormal{Arrival time for }B\text{ samples})+(\text{Processing time for }B\text{ samples})\\
% &=\frac{B}{R_s}+\frac{B}{N R_p}.
% \end{align*}
% This gives us the following lower bound on $R_c$
% \begin{align}\label{eqn:B_bound}
%     R_c=\Omega\Big(\frac{1}{B}\frac{N R_p - R_s}{NR_s R_p}\Big).
% \end{align}
% From this lower bound we can see that increasing mini-batch size $B$ relaxes the requirement on communication rate to process all the samples. Hence, for networks with slow communication links, choosing a large value of $B$ decreases the minimum required $R_c$ resulting in a broader range of operatibility for \DMK.

We conclude this section by extending Theorem~\ref{thm:DMK} to the under-provisioned setting in which $R_s > B R_e$, possibly due to slower communications links. Similar to our discussion for the DMB algorithm, we express $R_s$ as $R_s = (B+\mu)R_e$ for some $\mu \in \Z_+$ that corresponds to the number of samples that must be discarded at the splitter per iteration due to the mismatch between $R_s$ and $B R_e$. %Finally, we extend result in Theorem~\ref{thm:DMK} for resource constrained settings such that it is not possible to choose a value of $B$ that satisfies both \eqref{eqn:B_bound} and $B=O((t')^{1 - \tfrac{2}{c_0}})$ simultaneously, which will result in $R_s/B>R_e$ leading a situation requiring to discard $\mu=B(R_s/R_e - 1)$ samples per iteration of \DMK.
The following result captures the impact of this data loss on the convergence behavior of \DMK.
\begin{corollary}\label{cor:latency}
Let the parameters and constants be as specified in Corollary~\ref{cor:latency}, and define the final number of algorithmic iterations for \DMK~as $t^\mu :=t'/(B+\mu)$. Then, as long as the assumptions in Theorem~\ref{thm:DMK} hold and the network-wide mini-batch size satisfies $B\leq (t')^{1-\tfrac{2}{c_0}}$, there exists a sequence $(\Omega_{t}^{'})_{t \in \Z_+}$ of nested subsets of the sample space $\Omega$ such that $\bbP\left(\cap_{t>0}\Omega_t^{'}\right)\geq 1-\delta$ and
\begin{align}\label{eqn:BoundLatency}
	\E_{t^\mu}\left\{f(\bw_{t^\mu})\right\} - f(\bw^*) \leq c_0 C_1 \Bigg(\frac{(B+\mu)Q_1'}{t'}\Bigg)^{c_0/2}+c_0 C_1 \Bigg(\frac{(B+\mu) \sigma^2 Q_2'}{Bt'}\Bigg)^{c_0/2}+\frac{C_2 \sigma^2(B+\mu)}{Bt'}. %,
	\end{align}
	%where $\E_t^\mu$ is the conditional expectation over $\Omega_t^{'}$, and $C_1$ and $C_2$ are constants as defined in Theorem~\ref{thm:DMK}.
\end{corollary}

It can be seen from this result that as long as the number of discarded samples per iteration in \DMK~satisfies $\mu = O(B)$, we will have sample-wise order-optimal convergence rate of $O(1/t')$ in the network. This result concerning the impact of discarded samples in under-provisioned distributed systems is similar to the one reported in Theorem~\ref{thm:DMB} for the DMB algorithm.
%
% *****WUB: Replaced by the previous paragraph*****
% From this result, it is easy to see that as long as $\mu=B(R_s/R_e - 1)=O(B)$ we will have order optimal convergence rate. In particular, the expression for $\mu$ tells us that the ratio $R_s/R_e$ needs to stay small relative to $B$ in order to limit discarded samples to $\mu=O(B)$. In the following, these theoretical findings are corroborated with the help of experiments conducted over real-world and synthetic data.

\subsection{Numerical Experiments for DM-Krasulina}\label{subsec:ExperimentsKrasulinas}
In this section, we provide results of numerical experiments on both synthetic and real-world data to demonstrate the impact of mini-batch size on the performance of \DMK's method.

\subsubsection{Synthetic data}\label{subsubsec:Synthetic_DMK}
\begin{figure}[t]
	\centering
	\subfigure[Impact of the mini-batch size on the convergence rate of \DMK~for the resourceful regime. Note that the $B=1$ plot is effectively Krasulina's method.]{
		\includegraphics[width=0.45\columnwidth]{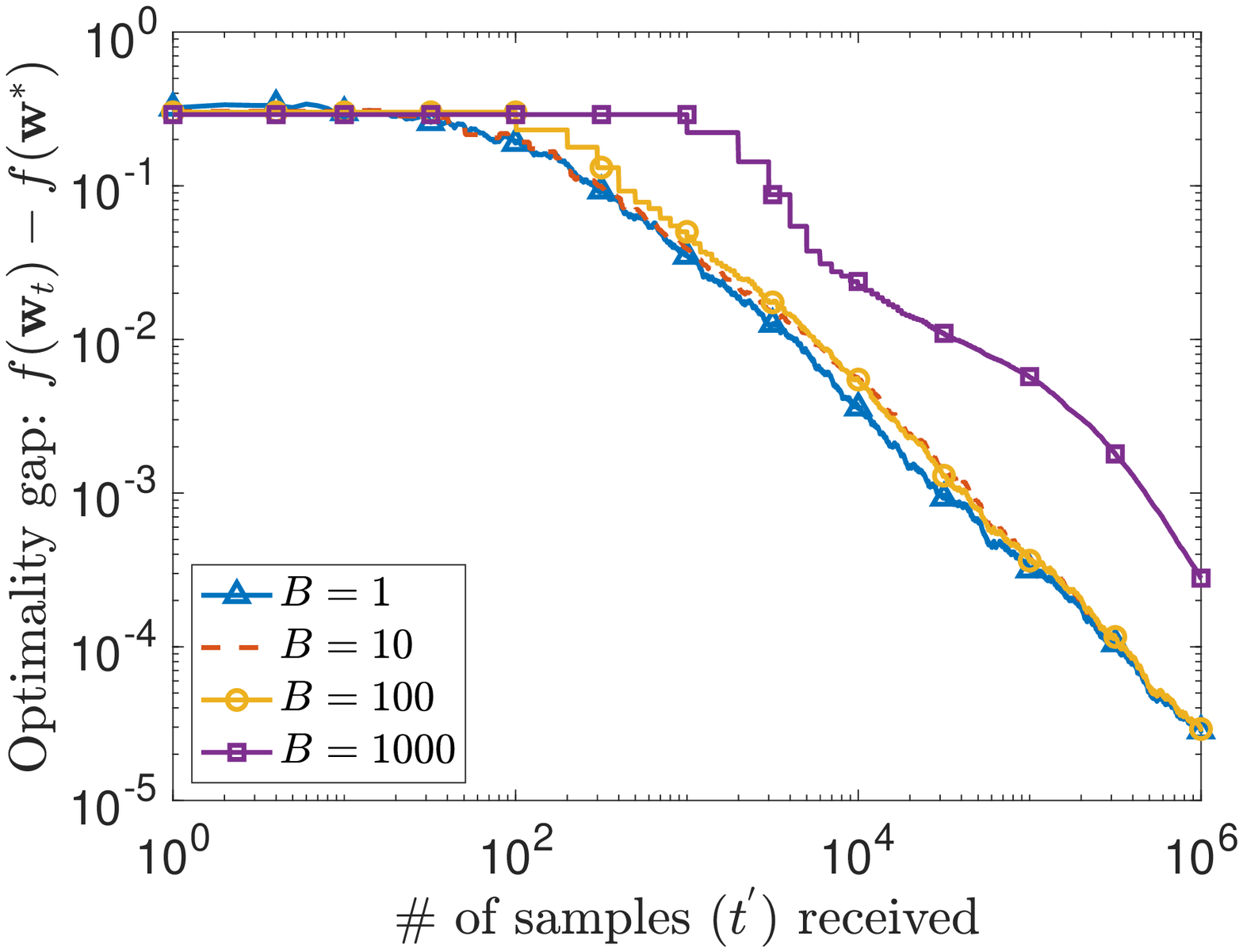}
		\label{fig:SyntheticPCA_NoLatency}}
	\qquad
	\subfigure[Performance of \DMK~in a resource-constrained regime (i.e., $R_s> B R_e$), which causes loss of $\mu$ samples per iteration; here, $(N,B)=(10,100)$.]{
		\includegraphics[width=0.45\columnwidth]{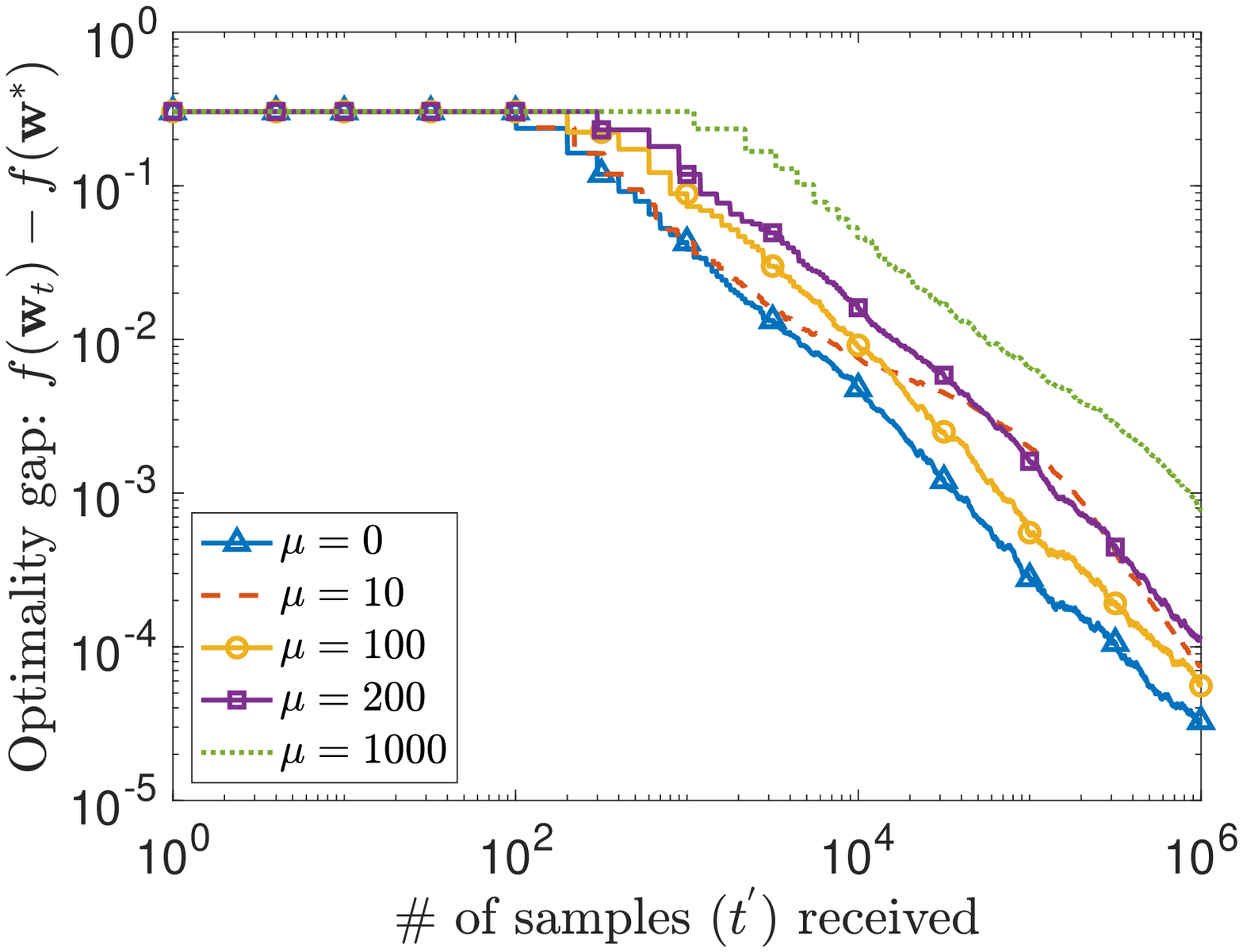}
		\label{fig:SyntheticPCA_Latency}}
	\caption{Convergence behavior of \DMK~for the case of synthetic data under two scenarios: (a) No data loss ($\mu = 0$) and (b) loss of $\mu > 0$ samples per algorithmic iteration.}
\end{figure}

For a covariance matrix $\bSigma\in\R^{10\times 10}$ with $\lambda_1(\bSigma)=1$ and eigengap $\lambda_1(\bSigma) - \lambda_2(\bSigma)=0.1$, we generate $t'=10^6$ samples from a normal distribution $\cN(\bzero, \bSigma)$. The first experiment in this case deals with the resourceful regime, i.e., $R_S \leq BR_e$, with mini-batches of sizes $B \in \{1, 10, 100, 1000\}$. Results of these experiments are shown in Fig.~\ref{fig:SyntheticPCA_NoLatency}, which correspond to stepsize $\eta_t=c/t$ and parameter $c=10$ that was selected after multiple trial-and-error runs. As predicted by Corollary~\ref{cor:BatchNoLatency}, we see the excess risk after $t=t'/B$ iterations of \DMK~is on the order of $O(1/t')$ for $B \in \{1, 10, 100\}$, while it is not optimal anymore for $B=1000$.

Next, we demonstrate the performance of \DMK~for resource constrained settings, i.e., $R_s > B R_e$, which causes the system to discard $\mu := (R_s/R_e - B)$ samples per iteration. Using the same data generation setup as before, we run \DMK~for a network of 10 nodes ($N=10$) with network-wide mini-batch of size $B=100$ (i.e., $B/N=10$). We consider different mismatch factors between streaming, processing, and communication rates in this experiment, which result in the number of samples being discarded as $\mu \in \{0, 10, 100, 200, 1000\}$. The results are plotted in Fig.~\ref{fig:SyntheticPCA_Latency}, which show that the values of excess risk for $\mu \in \{10, 100, 200\}$ are comparable to that for $\mu=0$, but the error for $\mu=1000$ is an order of magnitude worse than the nominal error.

\subsubsection{Real-world Data}\label{subsubsec:Real_DMK}
\revise{We next demonstrate the performance of DM-Krasulina on CIFAR-10 dataset~\cite{krizhevsky2009learning}, which consists of roughly $5\times 10^4$ training samples with $d=3072$.} Our first set of experiments for this dataset uses the stepsize $\eta_t=c/t$ with $c \in \{8\times 10^4, 8\times 10^4, 9\times 10^4, 10^5, 10^5\}$ for network-wide mini-batch sizes $B \in \{1, 10, 100, 1000, 5000\}$ in the resourceful regime ($\mu = 0$). The results, which are averaged over 50 random initializations and random shuffling of data, are given in Fig.~\ref{fig:MNIST_NoLatency}. It can be seen from this figure that the final error relatively stays the same as $B$ increases from $1$ to $1000$, but it starts getting affected significantly as the network-wide mini-batch size is further increased to $B=5000$. Our second set of experiments for the CIFAR-10 dataset corresponds to the resource-constrained regime with $(N,B)=(10,100)$ and stepsize parameter $c = 8\times 10^4$ for the number of discarded samples $\mu \in \{0, 10, 100, 200, 500\}$. The results, averaged over 200 trials and given in Fig.~\ref{fig:MNIST_Latency}, show that the system can tolerate loss of some data samples per iteration without significant increase in the final error; the increase in error, however, becomes noticeable as $\mu$ approaches $B$. Both these observations are in line with the insights of the theoretical results.

%For Higgs dataset we used $c=0.07$ and mini-batch of sizes, $B=\{1, 10^2, 10^3, 10^4, 2\times 10^4\}$. We can see the same trend here with suboptimal rate for mini-batches of size $B=10^4$ and $B=2\times 10^4$, while all the other batch-size choices give us $O(1/t)$ rate. Furthermore, with latency in computing network wide average we observe that for batch size $B=10^3$ we achieve optimal error rate as far as we are discarding up to $10^3$ samples per iteration round.
\begin{figure}
	\centering
		\subfigure[CIFAR-10 Data ($\mu=0$): Impact of network-wide mini-batch size $B$ on the convergence behavior of \DMK~for the resourceful regime.]{
	\includegraphics[width=0.45\columnwidth]{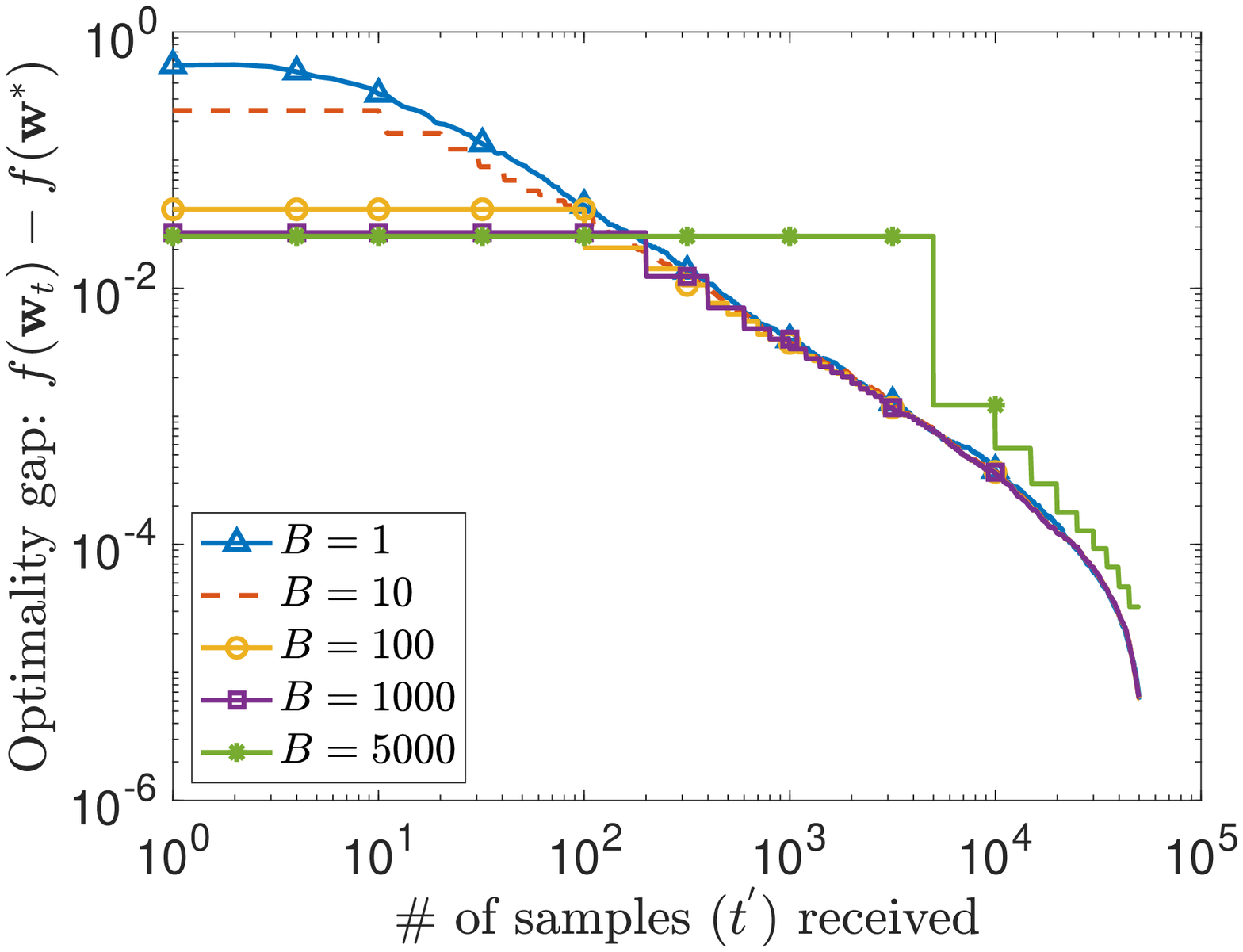}
        \label{fig:MNIST_NoLatency}}
		\qquad
		\centering
	    	\subfigure[CIFAR-10 Data ($N=10$; $B=100$): Convergence behavior of \DMK~in a resource-constrained regime, which causes loss of $\mu$ samples per iteration.]{
	    	\includegraphics[width=0.45\columnwidth]{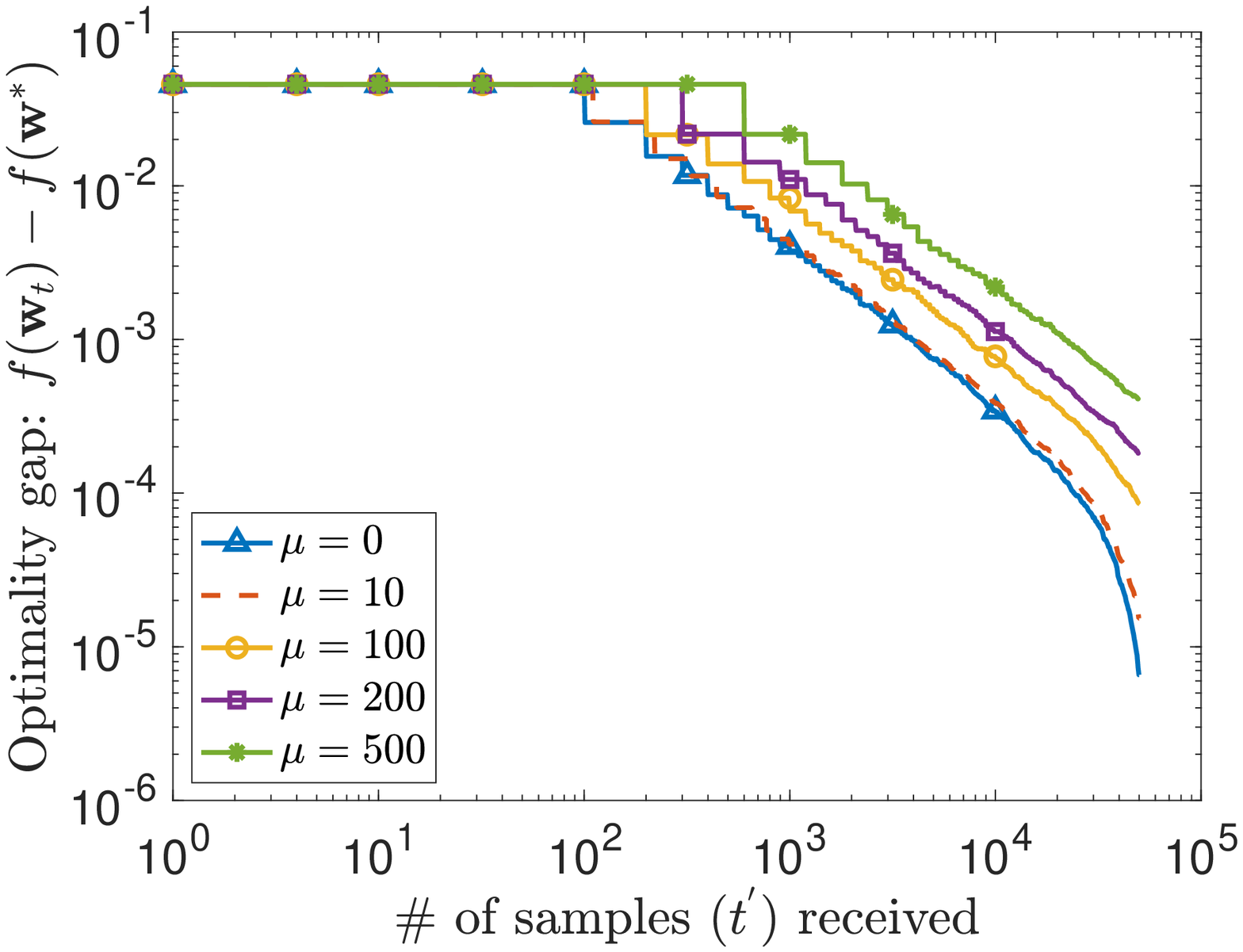}
            \label{fig:MNIST_Latency}}
		\caption{Performance of \DMK~for the CIFAR-10 dataset under two scenarios: (a) No data loss ($\mu = 0$) and (b) loss of $\mu > 0$ samples per algorithmic iteration.}
\end{figure}

\section{Distributed Stochastic \revise{Approximation} Using Inexact Averaging}\label{sec:Decentralized}
In this section, we discuss recent results for distributed learning from fast streaming data when the averages computed among nodes in the network are {\em inexact}. As mentioned earlier, inexact averaging occurs in networks where the communications topology either changes with time or it is unknown in advance to facilitate construction of an MPI infrastructure for {\tt AllReduce}-style computations. Similar to Section~\ref{sec:DM_Krasulina}, we require here that nodes compute distributed averages of their stochastic gradients in order to reduce their variance and speed up convergence. Unlike Algorithms~\ref{algo:DMB} and \ref{algo:DMB_Krasulina}, however, these averages are computed via $R$ rounds of averaging consensus, as described in Section~\ref{sect:distributed.learning}. This introduces a fundamental trade-off: the more consensus rounds $R$ used per algorithmic iteration, the smaller the effective gradient noise and averaging error, but the longer it takes to complete each iteration.

Similar to the DMB algorithm described in Section~\ref{subsec:ConvexMinibatch}, we mitigate this trade-off by careful mini-batching of gradient samples. Each node $n$ first locally averages the gradients for its local mini-batch of $B/N$ data samples, after which nodes \emph{approximately} average the mini-batched gradients via consensus iterations. Such mini-batching again speeds up the effective processing rate, allowing the network to process more samples. In this section we detail the precise conditions under which this speed-up is enough to achieve near-optimum convergence rates.

We focus exclusively on convex loss functions $\ell(\bw,\bz)$ in this section and present two algorithms for tackling fast streaming data under inexact averaging: \emph{distributed stochastic gradient descent} (D-SGD) and \emph{accelerated distributed stochastic gradient descent} (AD-SGD). Both these algorithms are distributed variants of the SGD and accelerated SGD methods described in Section \ref{sect:so}: after computing local mini-batch gradients and performing distributed averaging consensus, nodes in these algorithms take (accelerated) SGD steps with respect to the averaged gradients.

\subsection{Algorithms for Distributed Stochastic Convex Approximation}
We now present algorithmic details for D-SGD and AD-SGD, both of which are distilled and unified versions of algorithms published in \cite{tsianos2016efficient,NoklebyBajwa.ITSIPN19}.\footnote{In particular, \cite{tsianos2016efficient} presents a distributed learning strategy based on {\em dual averaging}, a method for stochastic convex optimization that has convergence rates similar to those of SGD. In contrast, \cite{NoklebyBajwa.ITSIPN19} presents a strategy based on {\em mirror descent}, a generalization of SGD-style methods. For clarity of exposition, we present results in here under the SGD framework.} The operating assumption here is that there is sufficient provisioning of resources within the system to ensure $R_s \leq BR_e$. In the next section, we present the convergence rates of D-SGD and AD-SGD under this assumption, and explicitly describe the regimes of system parameters in which these methods have optimum convergence rates. An important feature of these forthcoming results is the impact of acceleration in these regimes. Accelerated methods provide additional ``headroom,'' allowing for larger mini-batches that can process faster data streams while still yielding optimum convergence rates.

The algorithmic descriptions in the following make use of a symmetric, doubly stochastic matrix $\bA \in \R^{N \times N}$ that is consistent with the network graph $G$, as discussed in Section~\ref{sssec:inexact.ave}. We suppose that the second-largest eigenvalue magnitude obeys $|\lambda_2(\bA)| < 1$, where the inequality must be strict. This rather mild assumption is guaranteed, {\em inter alia}, by choosing $\bA$ to have elements strictly greater than zero for all elements corresponding to an edge of a connected graph $G$, i.e., each node includes its entire neighborhood (including itself) in the local convex combination it computes for each consensus round.

\subsubsection{Description of D-SGD}\label{sect:dsgd.description}
The D-SGD algorithm generalizes SGD with Polyak--Ruppert averaging to the setting of distributed, streaming data. We mathematically detail the steps of D-SGD in Algorithm~\ref{alg:standard}, and summarize them here. At the beginning of every iteration $t$ of D-SGD, each node $n$ receives a mini-batch $\{\bz_{n,b,t}\}_{b=1}^{B/N}$ of $B/N$ i.i.d.~data samples. Each node $n$ afterwards computes $\bg_{n,t}$, the average gradient over its local mini-batch, and then engages in $R \in \Z_+$ rounds of averaging consensus, where the parameter $R$ satisfies the constraints in \eqref{eqn:Rounds}. This results in an approximate average of all $N$ mini-batches at each node $n$, which we denote by $\bh_{n,t,R}$. Each node finally takes an SGD step using $\bh_{n,t,R}$ and engages in Polyak--Ruppert averaging to obtain the estimate $\bw_{n,t}^\mathsf{av}$.

\begin{algorithm}[t]
    \textbf{Require:} Provisioning of compute and communications resources to ensure $R_s \leq B R_e$\\
    \textbf{Input:} Data stream $\{\bz_{t'} \stackrel{\text{i.i.d.}}{\sim} \cD\}_{t' \in \Z_+}$ that is split into $N$ streams of mini-batched data $\{\bz_{n,b,t}\}_{b=1,t\in\Z_+}^{B/N}$ across $N$ nodes, doubly stochastic matrix $\bA$, number of consensus rounds $R$, and stepsize sequence $\left\{\eta_t \in \R_+\right\}_{t \in \Z_+}$\\
	\textbf{Initialize:} All compute nodes initialize with $\bw_{n,0} = \bzero \in \R^d$
	%\algsetup{indent=1em}
	\begin{algorithmic}[1]
		\For{$t=1,2,\dots$,}
		    \State $\forall n \in \{1, \dots, N\}, \ \bg_{n,t} \leftarrow \bzero \in \R^d$
		    %\State Initialize $\bg_n \leftarrow 0\;\forall n \in \{1, \dots, N\}$
		    %\State \textbf{(In Parallel)} Processor $n$ receives a mini-batch $\{\bz_{n,b,t}\}_{b=1}^{B/N}$ and updates $\bg_{n}$ locally as follows:
		    \For{$b=1,\dots,B/N$} \Comment{Node $n$ receives the mini-batch $\{\bz_{n,b,t}\}_{b=1}^{B/N}$ and updates $\bg_{n,t}$ locally}
		        \State $\forall n \in \{1, \dots, N\}, \ \bg_{n,b,t} \leftarrow \nabla \ell(\bw_{n,t}, \bz_{n,b,t})$
		        \State $\forall n \in \{1, \dots, N\}, \ \bg_{n,t} \leftarrow \bg_{n,t} + \frac{1}{B/N}\bg_{n,b,t}$
		        %\State Compute $\bg \leftarrow \nabla_{\bw_{n,t}}\ell(\bw_{n,t}, \bz_{n,b,t})$
		        %\State $\bg_{n} \leftarrow \bg_{n} + \frac{1}{B/N}\bg$
            \EndFor
		    \State $\forall n \in \{1, \dots, N\}, \ \mathbf{h}_{n,t,0} \gets \bg_{n,t}$ \Comment{Get mini-batched gradients}
		    \For{$r=1,\dots,R$ and $n=1,\dots,N$} \Comment{$R$ rounds of averaging consensus}
         	    \State $\mathbf{h}_{n,t,r} \gets \sum_{m = 1}^N a_{n,m}\mathbf{h}_{m,t,r-1}$
            \EndFor
		    \For{$n=1,\dots,N$}
        	    \State $\bw_{n,t+1} \gets [\bw_{n,t} - \eta_t \mathbf{h}_{n,t,R}]_\mathcal{W}$ \Comment{Projected SGD step}
                \State $\mathbf{w}_{n,t+1}^\textsf{av} \gets \left(\sum_{\tau = 0}^{t}\eta_\tau\right)^{-1} \sum_{\tau=0}^{t}\eta_\tau\bw_{n,\tau+1}$ \Comment{Averaging of iterates}
            \EndFor
		\EndFor
	\end{algorithmic}
	{\bf Return:} Decentralized estimates $\{\bw_{n,t}^\textsf{av}\}_{n\in \cV}$ of the Bayes optimal solution after receiving $t'=Bt$ samples
    \caption{Distributed Stochastic Gradient Descent (D-SGD)
    \label{alg:standard}}
\end{algorithm}

\subsubsection{Description of AD-SGD}\label{sect:adsgd.description}
The AD-SGD algorithm generalizes {\em accelerated} SGD, as presented in Section \ref{sect:so}, to the distributed setting. The generalization is similar to D-SGD's extension of the SGD procedure: Nodes collect mini-batches $\{\bz_{n,b,t}\}_{b=1}^{B/N}$ from their data streams, compute average gradients $\bg_{n,t}$, and get approximate gradient averages $\bh_{n,t,R}$ via $R$ rounds of averaging consensus. However, instead of taking a standard SGD step, compute nodes take an accelerated SGD step using $\bh_{n,t,R}$. This involves each node maintaining iterates $\bu_{n,t}$, $\bv_{n,t}$, and $\bw_{n,t}$ as in accelerated SGD, which are averaged and updated according to two sequences of stepsizes $\beta_t \in [1,\infty)$ and $\eta_t \in \R_+$. We mathematically detail the steps of AD-SGD in Algorithm~\ref{alg:accelerated}.

%The setting for AD-SGD is the same as in Section \ref{sect:mirror.descent}. As in Section \ref{sect:dsamd.description}, we suppose a mixing matrix $\mathbf{A} \in \mathbb{R}^{m \times m}$ that is symmetric, doubly stochastic, consistent with $G$, and has nonzero spectral gap. The main distinction between accelerated and standard mirror descent is the way one averages iterates. Rather than simply average the sequence of iterates, one maintains several distinct sequences of iterates, carefully averaging them along the way. This involves two sequences of step sizes $\beta_t \in [1,\infty)$ and $\gamma_t \in \mathbb{R}$, which are not held constant. Again we suppose that there is a predetermined number of algorithmic rounds $t$ and samples processed $t' = Bt$. We detail the steps of AD-SGD in Algorithm \ref{alg:accelerated}.

\begin{algorithm}[t]
    \textbf{Require:} Provisioning of compute and communications resources to ensure $R_s \leq B R_e$\\
    \textbf{Input:} Incoming mini-batched data streams at $N$ compute nodes, expressed as $\{\bz_{n,b,t}\}_{b=1,t\in\Z_+}^{B/N}$, doubly stochastic matrix $\bA$, number of consensus rounds $R$, and stepsize sequences $\left\{\eta_t,\beta_t \in \R_+\right\}_{t \in \Z_+}$\\
	\textbf{Initialize:} All compute nodes initialize with $\bu_{n,0}, \mathbf{v}_{n,0}, \mathbf{w}_{n,0} = \bzero \in \R^d$
	%\algsetup{indent=1em}
	\begin{algorithmic}[1]
		\For{$t=1,2,\dots$,}
        	\State $\forall n \in \{1, \dots, N\}, \ \mathbf{u}_{n,t} \gets \beta_t^{-1}\mathbf{v}_{n,t} + (1-\beta^{-1}_t)\mathbf{w}_{n,t}$
        	\State $\forall n \in \{1, \dots, N\}, \ \bg_{n,t} \leftarrow \bzero \in \R^d$
		    %\State Initialize $\bg_n \leftarrow 0\;\forall n \in \{1, \dots, N\}$
		    %\State \textbf{(In Parallel)} Processor $n$ receives a mini-batch $\{\bz_{n,b,t}\}_{b=1}^{B/N}$ and updates $\bg_{n}$ locally as follows:
		    \For{$b=1,\dots,B/N$} \Comment{Node $n$ receives the mini-batch $\{\bz_{n,b,t}\}_{b=1}^{B/N}$ and updates $\bg_{n,t}$ locally}
		        \State $\forall n \in \{1, \dots, N\}, \ \bg_{n,b,t} \leftarrow \nabla \ell(\bu_{n,t}, \bz_{n,b,t})$
		        \State $\forall n \in \{1, \dots, N\}, \ \bg_{n,t} \leftarrow \bg_{n,t} + \frac{1}{B/N}\bg_{n,b,t}$
		        %\State Compute $\bg \leftarrow \nabla \ell(\bu_{n,t}, \bz_{n,b,t})$
		        %\State $\bg_n\leftarrow \bg_n + \frac{1}{B/N}\bg$
		    \EndFor
		    \State $\forall n \in \{1, \dots, N\}, \ \mathbf{h}_{n,t,0} \gets \bg_{n,t}$ \Comment{Get mini-batched gradients}
            %\State $\mathbf{h}_{n,t,0} \gets \bg_n \;\forall n \in \{1, \dots, N\}$ \Comment{Get mini-batched gradients}
            \For{$r=1,\dots,R$ and $n=1,\dots,N$} \Comment{$R$ rounds of averaging consensus}
         	    \State $\mathbf{h}_{n,t,r} \gets \sum_{m = 1}^N a_{n,m}\mathbf{h}_{m,t,r-1}$
            \EndFor
		    %\For{$r=1:R$, $n=1:N$}
         	%    \State $\mathbf{h}_{n,t,r} \gets \sum_{m \in \mathcal{N}_i} a_{nm}\mathbf{h}_{m,t,r-1}$ \Comment{Consensus rounds}
            %\EndFor
		    \For{$n=1,\dots,N$} \Comment{Project A-SGD step}
        	    \State $\mathbf{v}_{n,t+1} \gets [\mathbf{u}_{n,t} - \eta_t \mathbf{h}_{n,t,r}]_\mathcal{W}$
                \State $\mathbf{w}_{n,t+1} \gets \beta_t^{-1}\mathbf{v}_{n,t+1} + (1-\beta_t^{-1})\mathbf{w}_{n,t}$
            \EndFor
        \EndFor
	\end{algorithmic}
	{\bf Return:} Decentralized estimates $\{\bw_{n,t}\}_{n\in \cV}$ of the Bayes optimal solution after receiving $t'=Bt$ samples
\caption{Accelerated Distributed Stochastic Gradient Descent (AD-SGD)
    \label{alg:accelerated}}
\end{algorithm}

% \begin{algorithm}[t]
%   \caption{Accelerated distributed stochastic gradient descent (AD-SGD)
%     \label{alg:accelerated1}}
%     \begin{algorithmic}[1]
%     \Require Doubly-stochastic matrix $\mathbf{A}$, step size sequences $\gamma_s$, $\beta_s$, number of consensus rounds $r$, batch size $b$, and stream of mini-batched gradients $\theta_i(s)$.
%     \For{$i=1:m$}
%     	\State $\mathbf{x}_i(1),\mathbf{x}^\mathrm{md}_i(1),\mathbf{x}^\mathrm{ag}_i(1) \gets 0$ \Comment{Initialize search points}
%     \EndFor
%     \For {$s=1:S$}
%         \For {$i=1:m$}
%         	\State $\mathbf{w}_i^\mathrm{md}(s) \gets \beta_s^{-1}\mathbf{w}_i(s) + (1-\beta^{-1}_s)\mathbf{w}_i^\mathrm{ag}(s)$
%             \State $\mathbf{h}_i^0(s) \gets \theta_i(s)$ \Comment{Get mini-batched gradients}
%         \EndFor

%         \For{$q=1:r$, $i=1:m$}
%         	\State $\mathbf{h}_i^q(s) \gets \sum_{j \in \mathcal{N}_i} a_{ij}\mathbf{h}_j^{q-1}(s)$ \Comment{Consensus rounds}
%         \EndFor
%         \For{$i=1:m$}
%         	\State $\mathbf{w}_i(s+1) \gets P_{\mathbf{w}_i(s)}(\gamma_s \mathbf{h}_i^r(s))$ \Comment{Prox mapping}
%             \State $\mathbf{w}^\mathrm{ag}_i(s+1) \gets \beta_s^{-1}\mathbf{w}_i(s+1) + (1-\beta_s^{-1})\mathbf{w}_i^\mathrm{ag}(s)$
%         \EndFor
%     \EndFor
%   \end{algorithmic}
%   \Return $\mathbf{x}_i^{\mathrm{ag}}(S+1), i=1,\dots,m.$
% \end{algorithm}

\subsection{Convergence Results and Scaling Laws}
Here, we present results on the convergence speeds of D-SGD and AD-SGD as well as their gap to the ideal case in which data streams can be centrally processed by a single powerful machine, or, equivalently, compute nodes can process instantaneously and communicate at infinite rates. We begin by presenting results for D-SGD.

\begin{theorem}\label{thm:mirror.descent.convergence.rate}
    Let the loss function $\ell(\bw, \bz)$ be convex and smooth with $L$-Lipschitz gradients and gradient noise variance $\sigma^2$, and suppose a bounded model space with expanse $D_\mathcal{W}$. %Let $\ell(\bw, \bz)$ be convex and have an  $L$-Lipschitz gradient in $\bw$ for each $\bz \in \cZ$, suppose that the stochastic gradient $\nabla_{\bw}\ell(\bw,\bz)$ has $\sigma^2$-bounded variance for all $\bw\in\cW$, and suppose a  bounded model space with expanse $D_\mathcal{W}$.
    Further suppose that nodes engage in $R$ rounds of consensus averaging in each iteration of D-SGD. Then, there exist stepsizes $\eta_t$ such that the expected excess risk at each node $n$ after $t$ iterations is bounded by
    \begin{equation}\label{eqn:DSAMD.convergence.rate}
    	\E\big\{f(\mathbf{w}_n^{\textsf{\emph{av}}}(t))\big\} - f(\bw^*) \leq \frac{2L}{t} + \sqrt{\frac{4\Delta_t^2}{t}} +  \sqrt{\frac{1}{2}}\frac{\Xi_t D_\mathcal{W}}{L},
    \end{equation}
where
\begin{equation}
	\Xi_t := \left(\frac{\sigma}{\sqrt{B/N}}\right)(1+ N^2 |\lambda_2(\bA)|^R)((1+N^2 |\lambda_2(\bA)|^R)^{t}-1)
\end{equation}
and
\begin{equation}
    \Delta_t^2 := 4\sigma^2/B + 2 \left(\frac{\sigma}{\sqrt{B/N}}\right)^2(1+N^4 |\lambda_2(\bA)|^{2R}) ((1+N^2 |\lambda_2(\bA)|^R)^{t}-1)^2
    + 4 |\lambda_2(\bA)|^{2R}\sigma^2 N^3/B
\end{equation}
quantify the moments of the effective gradient noise.
\end{theorem}
\begin{proof}[Proof Sketch]
    The convergence analysis follows that of standard results of SGD-style methods, with careful analysis of the equivalent gradient noise resulting from distributed consensus averaging. In particular, $\Delta_t^2$ bounds the equivalent noise variance, and it has two components: $4\sigma^2/B$, which is the equivalent variance if distributed gradients are averaged exactly, and additional terms that express variance increases due to inexact averaging. These latter terms go to zero geometrically as $R \to \infty$. The term $\Xi_t$ bounds the gradient {\em bias}, which is entirely due to averaging errors and which also goes to zero as $R \to \infty$. Inserting these bounds on the bias and variance into the analysis of SGD (see, e.g., \cite{lan2012optimal}) gives the result. See \cite{NoklebyBajwa.ITSIPN19} for a detailed proof.
\end{proof}

\begin{remark}
\revise{The result in Theorem~\ref{thm:mirror.descent.convergence.rate} is stated in terms of the ideal stepsizes discussed in Remark~\ref{rem:stepsize}, which suppose a finite and known time horizon $t$. Although the time horizon may not be known in advance, the conditions for optimality discussed in the following depend on the time horizon; as such, we retain these ideal stepsizes to make explicit this dependence of the results on the time horizon / final iteration count $t$.}
\end{remark}

The convergence rate in Theorem \ref{thm:mirror.descent.convergence.rate} has the same form as that of standard SGD given in Theorem~\ref{thm:SGD_Mini}. The main difference is in the bias and variance of the stochastic gradients, which depend on the processing and communications rates relative to the data streaming rate, and therefore on how many rounds of averaging consensus nodes can carry out per iteration of D-SGD.

The critical question for D-SGD is how fast communication needs to be for order-optimum convergence speed, i.e., the convergence speed that one would obtain if nodes had noiseless access to other nodes' gradient estimates in each iteration. Recall that the system has received $t' = tB$ data samples after $t$ data-splitting instances. Centralized SGD---with access to all $tB$ data samples in sequence---achieves the convergence rate $O\left(\frac{L}{Bt} + \frac{\sigma}{\sqrt{Bt}} \right)$, where the final term dominates the error as a function of $t$ if $\sigma^2  > 0$. In the following corollary we state conditions under which the convergence rate of D-SGD matches this term.
\begin{corollary}\label{cor:mirror.descent.consensus.rounds}
	The optimality gap for D-SGD at each node $n$ satisfies
     \begin{equation}
        \E\big\{f(\mathbf{w}_n^\textsf{\emph{av}}(t))\big\} - f(\bw^*) = O\left(\frac{\sigma}{\sqrt{t'}} \right),
    \end{equation}
    provided the network-wide mini-batch size $B$, the communications rate $R_c$, and the number of nodes $N$ satisfy
    \begin{align*}
    	B/N &= \Omega\left(1 + \frac{\log(t')}{\rho\log(1/|\lambda_2(\bA)|)}\right), \quad B/N = O\left(\frac{\sigma  \sqrt{t'}}{N}\right),\\
        R_c &= \Omega\left(\frac{R_s \log(t')}{\sigma \sqrt{t'}\log(1/|\lambda_2(\bA)|)} + \frac{R_s}{R_p N} \right), \quad t' = \Omega\left(\frac{N^2}{\sigma^2}\right),
    \end{align*}
    where
    \begin{equation*}
        \rho := N \frac{R_c}{R_s} - \frac{1}{R_p}
    \end{equation*}
    is the ratio of the ``effective'' communications rate per sample, which discounts the time spent in computation, and the rate at which streaming data arrive at the system.
\end{corollary}

This corollary describes the dependence of the convergence rate of D-SGD on the processing, communications and streaming rates, network topology, and network-wide mini-batch size. We now point out a few connections between this result and the one for the DMB algorithm described in Section \ref{subsec:ConvexMinibatch}. In the case of the DMB algorithm, recall that a maximum local mini-batch size of $B/N = O(\sqrt{t'}/N)$ is prescribed to ensure that the $O(1/t')$ term in the SGD convergence bound does not dominate. Corollary~\ref{cor:mirror.descent.consensus.rounds} further prescribes a {\em minimum} local mini-batch size $B/N$ needed to ensure that nodes have time to carry out sufficient consensus. This condition fails to obtain when the communications rate is too slow to accommodate the required mini-batch size. Indeed, for all else constant, the optimum local mini-batch size is merely $\Omega(\log(t'))$, and the condition on $R_c$ essentially ensures $B/N = O(t^{1/2})$.

Further, Corollary~\ref{cor:mirror.descent.consensus.rounds} dictates the relationship between the size of the network and the total number of data samples obtained at each node. Leaving the other terms constant, Corollary~\ref{cor:mirror.descent.consensus.rounds} requires $t' = \Omega(N^2)$, which implies that the total number of data samples processed {\em per node} should scale faster than the number of nodes in the network. This is a relatively mild condition for big data applications; indeed, many applications involve data streams that are large relative to the size of the network.

Along similar lines, ignoring other constants and log terms, Corollary~\ref{cor:mirror.descent.consensus.rounds} indicates that a communications rate of $R_c = \Omega\big(R_s/\sqrt{t'} + R_s/(R_p N) \big)$ is sufficient for order optimality. Thus, if the total number of data samples {\em per node} grows faster than the number of nodes, the required communications rate approaches the ratio between the sampling and processing rates of the network, i.e., \emph{communications need only be fast enough that processing is not the bottleneck}. This implies that for fixed networks or network families in which the spectral gap $1-|\lambda_2(\bA)|$ is bounded away from zero, even slow communications is sufficient for near-optimum learning.

Next, we bound the expected gap to optimality of the AD-SGD iterates.
\begin{theorem}\label{thm:accelerated.optimality.gap}
    Let the loss function $\ell(\bw, \bz)$ be convex and smooth with $L$-Lipschitz gradients and gradient noise variance $\sigma^2$, and suppose a bounded model space with expanse $D_\mathcal{W}$. %Let $\ell(\bw, \bz)$ be convex and have an  $L$-Lipschitz gradient in $\bw$ for each $\bz \in \cZ$, suppose that the stochastic gradient $\nabla_{\bw}\ell(\bw,\bz)$ has $\sigma^2$-bounded variance for all $\bw\in\cW$, and suppose a  bounded model space with expanse $D_\mathcal{W}$.
    Further suppose nodes engage in $R$ rounds of consensus averaging in each iteration of AD-SGD, and use stepsizes $\eta_t = (t+1)/2\eta$, for $0 < \eta < 1/(2L)$ and $\beta_t = (t+1)/2$. Then, the expected excess risk at each node $n$ after $t$ iterations of AD-SGD is bounded by
    \begin{equation*}
    	\E\{f(\mathbf{w}_n(t))\} - f(\bw^*) \leq \frac{8 L}{t^2} + 4 \sqrt{\frac{4\Delta_t^2}{t}} + \sqrt{32}\Xi_t,
    \end{equation*}
    where
    \begin{equation*}
    	\Delta_t^2 = 2(\sigma/\sqrt{B/N})^2((1+ 2\eta_t N^2 L|\lambda_2(\bA)|^R)^t-1)^2 + \frac{4 \sigma^2}{B/N}(|\lambda_2(\bA)|^{2R}N^2 + 1/N)
     \end{equation*}
     and
     \begin{equation*}
        \Xi_t = (\sigma/\sqrt{B/N})(1+B^2|\lambda_2(\bA)|^R)((1+2\eta_t N^2 L |\lambda_2(\bA)|^R)^t-1).
    \end{equation*}
\end{theorem}
\begin{proof}[Proof sketch]
    The result again follows from a careful analysis of the bias and variance of the equivalent gradient noise. As before, as $R \to \infty$, the variance term $\Delta_t^2$ has order $O(\sigma^2/B)$ and the bias term $\Xi_t$ vanishes. Putting these quantities into the analysis of accelerated SGD gives the result; we refer the reader to \cite{NoklebyBajwa.ITSIPN19} for details.
\end{proof}

\begin{remark}
\revise{Theorem~\ref{thm:accelerated.optimality.gap} is also stated in terms of the ideal stepsizes given in \cite{lan2012optimal}, which suppose a finite and known time horizon; we again retain this form of the result in order to keep explicit the dependence on time horizon.}
\end{remark}

% As with D-SGD, we study the conditions under which AD-SGD achieves order-optimum convergence speed. The centralized version of accelerated mirror descent, after processing the $mT$ data samples that the network sees after $S$ mini-batch rounds, achieves the convergence rate
% \begin{equation*}
% 	O(1)\left[ \frac{L}{(mT)^2} + \frac{\sigma}{\sqrt{mT}}\right].
% \end{equation*}
% This is the optimum convergence rate under any circumstance. In the following corollary, we derive the conditions under which the convergence rate matches the second term, which usually dominates the error when $\sigma^2 > 0$.
\begin{corollary}\label{cor:accelerated.descent.consensus.rounds}
	The excess risk for AD-SGD at each node $n$ satisfies
    \begin{equation*}
    	\E\big\{f(\mathbf{w}_n(t)\big\} - f(\bw^*) = O\left(\frac{\sigma}{\sqrt{t'}} \right),
    \end{equation*}
    provided
    \begin{align*}
    B/N &= \Omega\left(1 + \frac{\log(t')}{\rho\log(1/|\lambda_2(\bA)|)}\right), \quad B/N = O\left(\frac{\sigma^{1/2}(t')^{3/4}}{N}\right), \\
        R_c &= \Omega\left(\frac{R_s \log(t')}{\sigma (t')^{3/4}\log(1/|\lambda_2(\bA)|)} + \frac{R_s}{R_p N} \right), \quad t' = \Omega\left(\frac{N^{4/3}}{\sigma^2}\right),
    \end{align*}
    where again $\rho := N \frac{R_c}{R_s} - \frac{1}{R_p}$ is the ratio of the effective communications rate per sample and the streaming rate.
\end{corollary}

Notice that the crucial difference between the two schemes is that AD-SGD has a convergence rate of $1/t^2$ in the absence of noise. This faster term, which is often negligible in centralized SGD, means that AD-SGD tolerates more aggressive mini-batching without an impact on the order of the convergence rate. As a result, the condition on $R_c$ is relaxed by $1/4$ in the exponent of $t'$. This is because the condition $B/N = O(t^{1/2})$, which holds for standard stochastic methods, is relaxed to $B/N = O(t^{3/4})$ for accelerated SGD. Thus, the use of accelerated methods increases the domain in which order-optimum rate-limited learning is guaranteed.

% Similar to Corollary \ref{cor:mirror.descent.consensus.rounds}, Corollary \ref{cor:accelerated.descent.consensus.rounds} prescribes a relationship between $m$ and $T$, but the relationship for AD-SGD is $T~=~\Omega(m^{1/3})$, holding all but $m,T$ constant. This again is due to the relaxed mini-batch condition $b = O(T^{3/4})$ for accelerated SGD. Furthermore, ignoring the $\log$ term, Corollary \ref{cor:accelerated.descent.consensus.rounds} indicates that a communications ratio $\rho = \Omega\left(\frac{m^{1/4}}{T^{3/4}}\right)$ is needed for well-connected graphs such as expander graphs. In this case, as long as $T$ grows faster than the cube root of $m$, order-optimum convergence rates can be obtained even for small communications ratio. Thus, the use of accelerated methods increases the domain in which order optimum rate-limited learning is guaranteed.

\subsection{Numerical Experiments for D-SGD and AD-SGD}
% To demonstrate the scaling laws predicted by Corollaries \ref{cor:mirror.descent.consensus.rounds} and \ref{cor:accelerated.descent.consensus.rounds} and to investigate the empirical performance of D-SAMD and AS-SAMD, we consider supervised learning via binary logistic regression. Specifically, following the same streaming data model as in Section~\ref{sec:ProblemFormulation} we assume each node observes a stream of pairs $\bz_{n,t} = (\bx_{n,t},y_{n,t})$ of
% data points $\bx_{n,t} \in \mathbb{R}^d$ and their labels $y_{n,t} \in \{0,1\}$, from which it learns a classifier with the log-likelihood function
% \begin{equation*}
% 	\ell(\mathbf{\bw},w_0,\bz) = y (\mathbf{x}^{\tT}\mathbf{w} + w_0) - \log(1+\exp(\mathbf{x}^T\mathbf{w} + w_0))
% \end{equation*}
% where $\mathbf{w} \in \mathbb{R}^d$ and $w_0 \in \mathbb{R}$ are regression coefficients.

% The learning task is to learn the optimum regression coefficients $\mathbf{w},w_0$. The convex objective function is the negative of the log-likelihood function, averaged over the unknown distribution of the data, i.e.
% \begin{equation*}
% 	f(\mathbf{w}) = -\E \big\{\ell(\mathbf{w},w_0,\mathbf{x},y)\big\}.
% \end{equation*}
% Minimizing $f$ is equivalent to performing maximum likelihood estimation of the regression coefficients \cite{bishop.book06}.

In order to demonstrate the scaling laws predicted by Corollaries~\ref{cor:mirror.descent.consensus.rounds} and~\ref{cor:accelerated.descent.consensus.rounds} and to investigate the empirical performance of D-SGD and AD-SGD, we once again resort to binary linear classification using logistic regression. Similar to Section~\ref{subsec:numerical_DMB}, we examine performance on synthetic data in order to have a ``ground truth`` data distribution against which to compare performance. For the sake of richness, we generate data in here using a slightly different probabilistic model than in Section~\ref{subsec:numerical_DMB}. Specifically, we suppose the data follow conditional Gaussian distributions: For $y_{t'} \in \{-1,+1\}$, we let $\bx_{t'} \sim \mathcal{N}(\mu_{y_{t'}},\sigma_x^2\bI)$, where $\mu_{y_{t'}} \in \{\mu_{-1}, \mu_1\}$ is one of two mean vectors, and $\sigma_x^2 > 0$ is the noise variance (not to be confused with gradient noise variance $\sigma^2$.) For the experiments, we pick $d=20$, choose $\sigma_x^2=2$, and draw the elements $\mu_{-1}$ and $\mu_1$ randomly from the standard normal distribution.

We compare the performance of D-SGD and AD-SGD against several other schemes. As a best-case scenario, we consider {\em centralized} counterparts of D-SGD and AD-SGD, meaning that all $B$ data samples and their associated gradients at each data-splitting instance are available at a single machine, which carries out SGD and {\em accelerated} SGD. Both these algorithms naturally have the best average performance. As a baseline, we consider {\em local} (accelerated) SGD, in which nodes simply perform SGD on their own data streams without collaboration. This scheme benefits from an insensitivity to the \emph{mismatch ratio} $\rho$ (defined in Corollary~\ref{cor:accelerated.descent.consensus.rounds}), i.e., it does not require any mini-batching, and therefore it represents a minimum standard for performance.

Finally, we consider a communications-constrained adaptation of DGD~\cite{Nedic.Ozdaglar.ITAC2009} (see Section~\ref{sect:distributed.learning} for a description of DGD). Note that DGD implicitly supposes $\rho=1$; to handle the $\rho < 1$ case, we consider two adaptations: {\em naive} DGD, in which data samples that arrive between computation+communications rounds are simply discarded, and {\em mini-batched} DGD, in which nodes compute {\em local} mini-batches of size $B/N=1/\rho$, take gradient updates using the local mini-batch, and carry out a consensus round. While it is not designed for the communications-limited scenario, DGD has good performance in general, so it represents a natural alternative against which to compare the performance of D-SGD and AD-SGD.

For network topology, we use  {\em expander graphs}, which are families of graphs that have spectral gap $1-|\lambda_2(\bA)|$ bounded away from zero as $N \to \infty$. In particular, we use 6-regular graphs, i.e., regular graphs in which each node has six neighbors, drawn uniformly from the ensemble of such graphs. Because $|\lambda_2(\bA)|$ is strictly bounded above by 1 for expander graphs, one can more easily examine whether performance of D-SGD and AD-SGD agrees with the ideal scaling laws discussed in Corollaries~\ref{cor:mirror.descent.consensus.rounds} and~\ref{cor:accelerated.descent.consensus.rounds}. At the same time, because D-SGD and AD-SGD make use of imperfect averaging, expander graphs also allow us to examine non-asymptotic behavior of the two schemes. Per Corollaries~\ref{cor:mirror.descent.consensus.rounds} and~\ref{cor:accelerated.descent.consensus.rounds}, we choose $B/N = \frac{1}{10}\frac{\log(t')}{\rho \log(1/|\lambda_2(\bA)|)}$. While such scaling is guaranteed to be sufficient for optimum asymptotic performance, we chose the multiplicative constant $1/10$ via trial-and-error to give good non-asymptotic performance.

In Fig.~\ref{fig:expander} we plot the performance of different methods averaged over 600 Monte Carlo trials. We take $\rho = 1/2$, and consider the regimes $t' = N^2$ (Fig.~\ref{fig:expander}(a)) and $t' = N^{3/2}$ (Fig.~\ref{fig:expander}(b)). For D-SGD, the stepsizes are taken to be $\eta_t = 2.5/\sqrt{t}$. For AD-SGD, we take $\beta_t = t/2$ as prescribed in \cite{lan2012optimal}, as well as $\eta_t = 8/(t+1)^{3/2}$ when $t' = N^{3/2}$ and $\eta_t = 14/(t+1)^{3/2}$ when $t' = N^2$. We arrived at the constants in front of $\eta_t$ via trial-and-error.
% when $t' = N^{3/2}$, $\gamma_t = 14$ for AD-SGD when $t' = N^2$, and $\gamma = 8$ for AD-SGD when $t' = N^{3/2}$.
We see that AD-SGD and D-SGD outperform local methods, while their performance is roughly in line with asymptotic theoretical predictions. The performance of DGD, on the other hand, depends on the regime: For $t' = N^2$, it appears to have order-optimum performance, whereas for $t' = N^{3/2}$ it has suboptimum performance on par with local methods. The reason for the dependency of DGD on regime is not immediately clear and suggests the need for further study into DGD-style methods in the case of rate-limited networks.

\begin{figure}[t]
  \centering
  \subfigure[$t' = N^2$]
    {\includegraphics[width=0.45\textwidth]{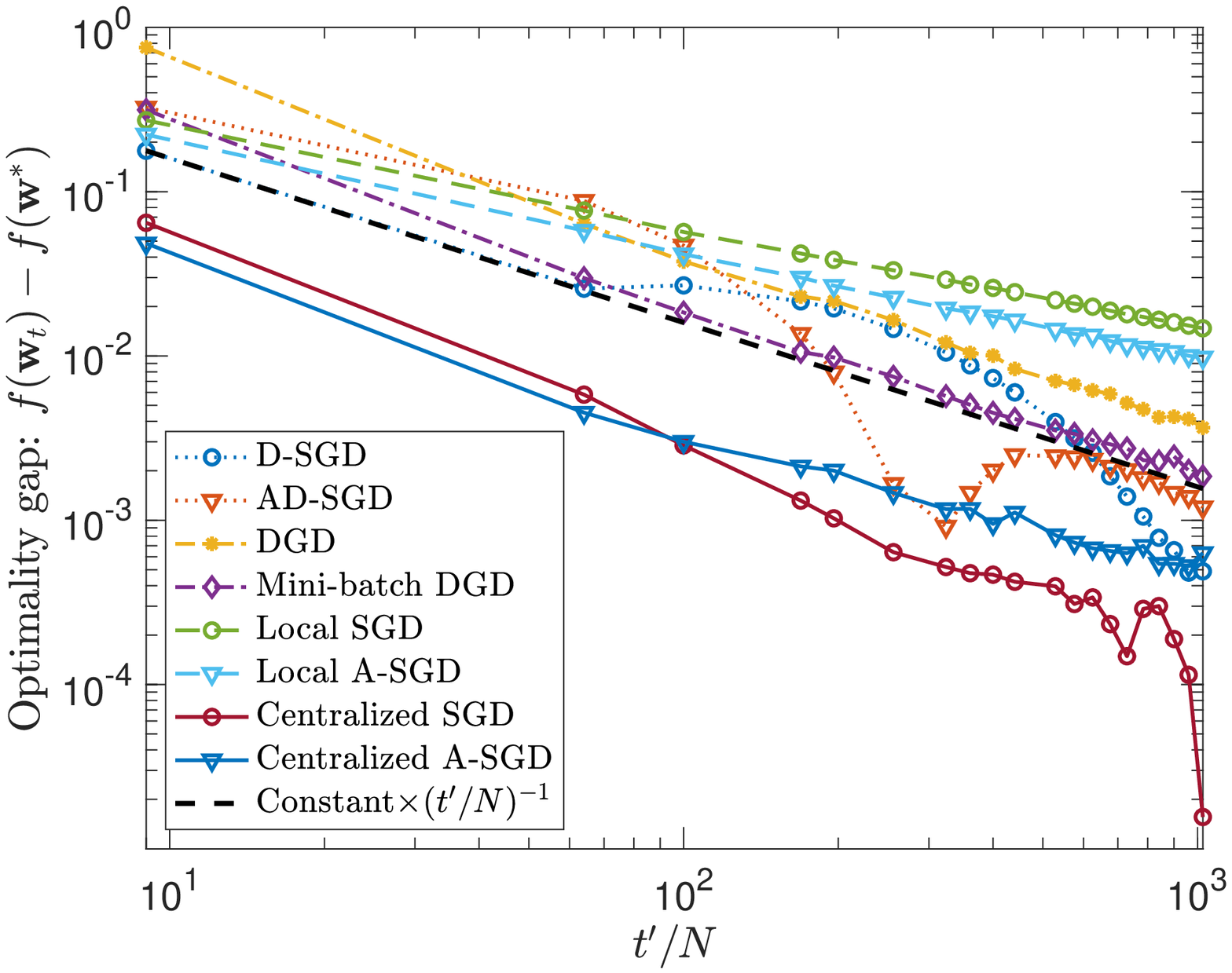}}
    \hfill
  \subfigure[$t' = N^{3/2}$]
    {\includegraphics[width=0.45\textwidth]{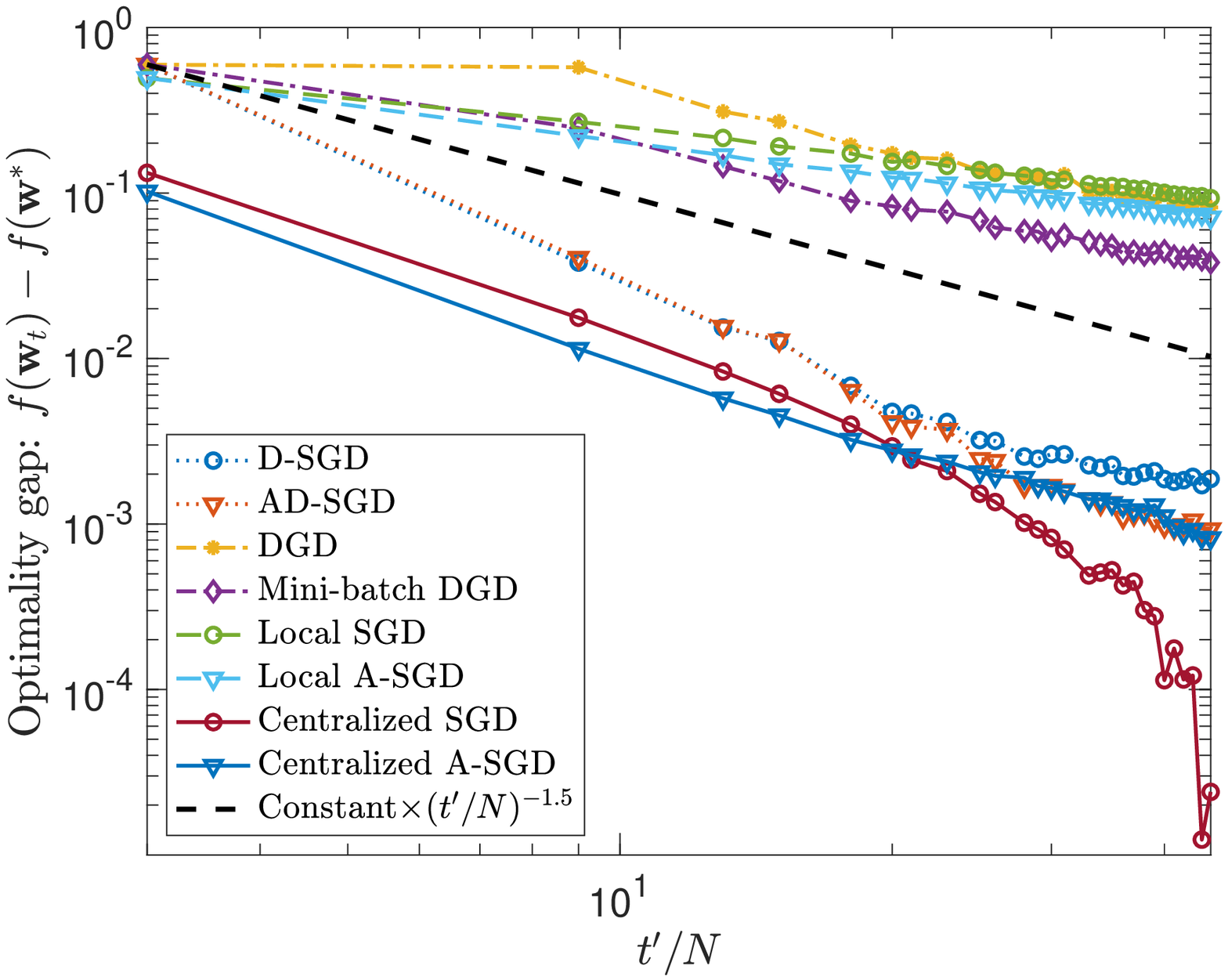}}
   \caption{Performance of different distributed and centralized first-order methods on $6$-regular expander graphs, for $\rho = 1/2$, measured in terms of the excess risk for binary logistic regression.}
	\label{fig:expander}
\end{figure}

% \subsubsection{Erd\H{o}s-Renyi Graphs}
% Finally, we consider {\em Erd\H{o}s-Renyi} graphs, in which a random fraction (in this case $0.1$) of possible edges are chosen. These graphs are not expanders, and their spectral gaps are not bounded. Therefore, order-optimum performance is not easy to guarantee, since the conditions on the rate and the size of the network depend on $\lambda_2(\bA)$, which is not guaranteed to be well behaved. We again take $\rho = 1/2$, consider the regimes $t' = N^2$ and $t' =  N^{3/2}$, and again we choose $B/N = \frac{1}{10}\frac{\log(t')}{\rho \log(1/\lambda_2(\bA))}$. The step sizes are chosen to be the same as for expander graphs in both regimes.
%
% Once again, we observe a clear distinction in performance between local and distributed methods; in particular, all distributed methods (including DGD) appear to show near-optimum performance in both regimes. However, as expected the performance is somewhat more volatile than in the case of expander graphs, especially for the case of $t' = N^2$, and it is possible that the trends seen in these plots will change as $t'$ and $N$ increase.
%
% \begin{figure}[htb]
%   \centering
%   \subfigure[$t' = N^2$]
%     {\includegraphics[width=0.45\textwidth]{Figs/ER_T_m.eps}}
%     \hfill
%   \subfigure[$t' = N^{3/2}$]
%     {\includegraphics[width=0.45\textwidth]{Figs/ER_T_Sq_m.eps}}
%   \caption{Performance on Erd\H{o}s-Renyi graphs, for $\rho = 1/2$.}
% 	\label{fig:erdos}
% \end{figure}

\section{Conclusion and Future Directions}\label{sec:Conclusion}
The development, characterization, and implementation of efficient learning from fast (and distributed) streams of data is a challenge for researchers and practitioners for the coming decade. In this paper, we have laid out results that suggest that such learning is possible---even when the processing rates of individual compute nodes and/or network communications are slow relative to the streaming rate of data. In particular, we have framed this problem as a distributed stochastic approximation problem, in which streams of independent and identically distributed (i.i.d.) data samples arrive at compute nodes that are connected by communications networks and that exchange messages at a fixed rate---which may be slower than the rate at which samples arrive.

Within this framework, we have discussed distributed first-order stochastic optimization methods that efficiently solve the learning problem, both in systems with robust communications networks that can implement exact {\tt AllReduce}-style aggregation of data, and in systems with decentralized communications networks that implement approximate aggregation of data via averaging consensus. We have also presented performance guarantees for these methods for general convex problems and for the ``well-behaved'' nonconvex problem of principal component analysis (PCA), for both of which global optimization is possible.

A critical component of these methods is explicit \emph{local mini-batching}, in which nodes average together the gradients (or gradient-like quantities) of multiple samples. Nodes then need only communicate the gradient of the entire local mini-batch, which substantially reduces the communications burden of distributed learning. A consistent through-line of these results is that both the (implicit) network-wide and the (explicit) local mini-batch sizes must be chosen carefully; small mini-batches do not slow down algorithmic iterations sufficiently to counter the fast streaming rate and/or reduce the communications load, and large mini-batches slow them down so far that convergence is slowed down. We give both a precise characterization of the necessary mini-batch size and the network constraints---in terms of the size, topology, and communications rates---under which it is possible to obtain convergence rates that are as fast (order-wise) as the ideal case in which there are ample compute nodes and communications between the nodes is perfect and instantaneous. Perhaps surprisingly, these results show that even relatively slow communication links are often sufficient for optimal distributed learning.

We conclude this paper with discussion of a subset of research questions that these results leave open.

{\bf Nonconvex losses.} The results presented here are for convex losses or for the PCA problem, a special case in which all local optima are global, and first-order methods can converge on a global optimum. But many important machine learning problems, including deep learning, are highly nonconvex. Recently, several works have proposed and analyzed methods for distributed nonconvex optimization~\cite{tatarenko2017non,wai2017decentralized,zeng2018nonconvex}. However, to the best of our knowledge, the question of general nonconvex learning from fast, distributed data streams with global guarantees remains open. %As mentioned above, SGD-style methods converge on stationary points for nonconvex functions, and as mentioned above several works have proposed and analyzed methods for distributed nonconvex optimization \cite{tatarenko2017non,wai2017decentralized,zeng2018nonconvex}. However, to the best of our knowledge the question of nonconvex learning from fast, distributed data streams remains open.

One of the major challenges in extending the results presented here is pinpointing the impact of inexact averaging on the potential mini-batching gains. When nodes compute inexact averages of their gradients, their iterates may diverge. This divergence may be catastrophic in the case of nonconvex learning, as nodes' iterates may end up in different basins of attraction. In this scenario, gradients will be taken with respect to increasingly distant operating points, and averaging them according to the recipe of this paper may not result in good search directions.

{\bf Message quantization.} This work supposes nodes exchange real-valued messages, whereas messages are quantized in digital communications networks. Further, in addition to local mini-batching, nodes can speed up communications by explicitly quantizing their messages, thereby reducing the network throughput required to exchange messages. But quantization introduces additional noise into the system, which requires further analysis and algorithmic control. A variety of methods have been proposed for tackling quantization-aware learning, both in centralized and distributed settings~\cite{gupta2015deep,zhou2016dorefa,wen2017terngrad}, including {\em signSGD} \cite{bernstein2018signsgd}, in which nodes exchange one-bit quantized gradients. An open question is when and whether such schemes are optimal for distributed learning from fast streaming data.


\end{document}